%% file: main.tex
\documentclass{article}
\usepackage{kr}

\usepackage{times}
\usepackage{soul}
\usepackage{url}
\usepackage[hidelinks]{hyperref}
\usepackage[utf8]{inputenc}
\usepackage[small]{caption}
\usepackage{graphicx}
\usepackage{amsmath}
\usepackage{amsthm}
\usepackage{booktabs}
\usepackage{algorithm}
\usepackage{algorithmic}
\usepackage{enumerate}
\usepackage[shortlabels,inline]{enumitem}
\usepackage{todonotes}

\newtheorem{example}{Example}
\newtheorem{theorem}{Theorem}

\newtheorem{lemma}{Lemma}
\newtheorem{claim}{Claim}

\newtheorem{definition}{Definition}

\usepackage[bibliography=common]{apxproof}
\newtheoremrep{example}{Example}
\newtheoremrep{theorem}{Theorem}
\newtheoremrep{proposition}{Proposition}
\newtheoremrep{lemma}{Lemma}
\newtheoremrep{claim}{Claim}
\newtheoremrep{corollary}{Corollary}
\newtheoremrep{definition}{Definition}
\newtheoremrep{remark}{Remark}
\newtheoremrep{observation}{Observation}
\usepackage{amsfonts}
\usepackage{xspace}
\usepackage{pifont}
\usepackage{color}
\usepackage{amsmath, upgreek, stmaryrd}
\usepackage{amssymb}
\usepackage{xargs}
\usepackage{xcolor}
\usepackage{graphicx}
\usepackage{mathtools}
\usepackage[capitalise]{cleveref}

\usepackage{mathdots}
\usepackage{tikz}
\usetikzlibrary{tikzmark}
\usetikzlibrary{calc}
\usetikzlibrary{positioning,shapes.geometric}
\usetikzlibrary{patterns}

\newcommand{\Amc}{\ensuremath{\mathcal{A}}}
\newcommand{\Cmc}{\ensuremath{\mathcal{C}}}
\newcommand{\Emc}{\ensuremath{\mathcal{E}}}

\newcommand{\Imc}{\ensuremath{\mathcal{I}}}
\newcommand{\Jmc}{\ensuremath{\mathcal{J}}}
\newcommand{\Kmc}{\ensuremath{\mathcal{K}}}
\newcommand{\Lmc}{\ensuremath{\mathcal{L}}}
\newcommand{\Mmc}{\ensuremath{\mathcal{M}}}

\newcommand{\Tmc}{\ensuremath{\mathcal{T}}}
\newcommand{\Vmc}{\ensuremath{\mathcal{V}}}

\newcommand{\A}{\Amc} 
\newcommand{\T}{\Tmc} 
\newcommand{\K}{\Kmc} 
\newcommand{\I}{\Imc}
\newcommand{\J}{\Jmc}

\newcommand{\ISA}{\sqsubseteq}
                 
\newcommand{\cnames}{\mathsf{N_C}}
\newcommand{\inames}{\mathsf{N_I}}
\newcommand{\rnames}{\mathsf{N_R}}
\newcommand{\rnamespm}{\mathsf{N_R^\pm}}
\newcommand{\tbox}{\Tmc}
\newcommand{\abox}{\Amc}
\newcommand{\kb}{\Kmc}

\newcommand{\dllite}{\textup{DL-Lite}}
\newcommand{\dllitebool}{\ensuremath{\dllite_\textup{bool}}}
\newcommand{\dllitehorn}{\ensuremath{\dllite_\textup{horn}}}

\newcommand{\dllitecore}{\ensuremath{\dllite_\textup{core}}}
\newcommand{\ALC}{\mathcal{ALC}}
\newcommand{\ALCIO}{\mathcal{ALCIO}}
\newcommand{\EL}{\mathcal{EL}}
\newcommand{\ELbot}{\mathcal{EL_\bot}}

\newcommand{\ELIO}{\mathcal{ELIO}}

\newcommand{\ExpSpace}{\ensuremath{\mathsf{ExpSpace}}}

\newcommand{\NExp}{\ensuremath{\mathsf{NExp}}}
\newcommand{\NP}{\ensuremath{\mathsf{NP}}}

\newcommand{\NL}{\ensuremath{\mathsf{NL}}}
\newcommand{\PTime}{\ensuremath{\mathsf{P}}}

\newcommand{\sizeof}[1]{\ensuremath{{\left|#1\right|}}}

\newcommand{\cstyle}[1]{\ensuremath{\mathsf{#1}}}
\newcommand{\rstyle}[1]{\ensuremath{\mathsf{#1}}}
\newcommand{\istyle}[1]{\ensuremath{\mathsf{#1}}}

\renewcommand{\emptyset}{\varnothing}

\newcommand{\successor}{\ensuremath{\sigma}}

\newcommand{\cn}[1]{\ensuremath{\mathsf{#1}}\xspace}
\newcommand{\mms}{\textsc{MinModSat}\xspace}
\newcommand{\ccol}{\textsc{Cert3Col}\xspace}

\newcommand{\coccol}{\textsc{coCert3Col}\xspace}

\newcommand{\depgraph}{\mathit{DG}}
\newcommand{\type}{\mathsf{tp}}

\newcommand{\leaft}{\cn{L}_{6n+2}^{(3)}}

\newcommand{\leafone}{\mathrm{Leaf_1}}
\newcommand{\leaftwo}{\mathrm{Leaf_2}}
\newcommand{\leafthree}{\mathrm{Leaf_3}}
\newcommand{\leafred}{\mathrm{Leaf_R}}
\newcommand{\leafblue}{\mathrm{Leaf_B}}
\newcommand{\leafgreen}{\mathrm{Leaf_G}}
\newcommand{\leafi}{\mathrm{Leaf}_i}
\newcommand{\leafc}{\mathrm{Leaf}_C}

\usetikzlibrary{shapes}
\usetikzlibrary{patterns.meta}
\usetikzlibrary{fadings,decorations.pathmorphing}
\usetikzlibrary{shapes.geometric}
\usetikzlibrary{fadings}

\title{Minimal Model Reasoning in Description Logics: Don't Try This at Home!}

\author{%
	Federica Di Stefano$^1$\and
	Quentin Manière$^{2,3}$\and
	Magdalena Ortiz$^1$\and
	Mantas \v{S}imkus$^1$
	\affiliations
	$^1$Institute of Logic and Compuation, TU Wien\\
	$^2$Department of Computer Science, Leipzig University, Germany \\
	$^3$Center for Scalable Data Analytics and Artificial Intelligence (ScaDS.AI), Dresden/Leipzig, Germany
	\emails  
	quentin.maniere@uni-leipzig.de,
	\{federica.stefano\textbar magdalena.ortiz\textbar mantas.simkus\}@tuwien.ac.at
}

\begin{document}

	\maketitle

	\begin{abstract}
		Reasoning with minimal models has always been at the core of many knowledge representation techniques, but we still have only a limited understanding of this problem in \emph{Description Logics (DLs)}.
		Minimization of \emph{some} selected predicates---letting the remaining predicates \emph{vary} or be \emph{fixed}, as proposed in \emph{circumscription}---has been explored and exhibits high complexity.
		The case of `pure' minimal models, where the extension of \emph{all} predicates must be minimal, has remained largely uncharted. 
		We address this problem in popular DLs and obtain surprisingly negative results: concept satisfiability in minimal models is undecidable already for $\EL$.
          This undecidability also extends to a very restricted fragment of \emph{tuple-generating dependencies}. 
 		To regain decidability, we impose acyclicity conditions on the TBox that bring the worst-case complexity below double exponential time and allow us to establish a connection with the recently studied pointwise circumscription; we also derive results in data complexity.
		We conclude with a brief excursion to the DL-Lite family, where a positive result was known for $\dllitecore$, but our investigation establishes $\ExpSpace$-hardness already for its extension $\dllitehorn$.
	\end{abstract}	
	
	\section{Introduction}

    Reasoning with \emph{minimal models} has always been at the core of many \emph{Knowledge Representation (KR)} languages.
It is most prominent in  formalisms for non-monotonic reasoning, 
from default logic \cite{reiter1980defaultlogic} and circumscription \cite{mccarthy1980circumscription}
to answer set programming  \cite{gel88} and it plays a crucial role in classical KR problems like abduction and diagnosis \cite{reiter1987diagnosis}. Finding minimal models and reasoning about them has been a recurring topic in the KR literature for many years; see~\cite{amai/Ben-EliyahuD96,DBLP:journals/ai/Ben-Eliyahu-ZoharyP97,DBLP:conf/kr/LacknerP12,DBLP:journals/ai/AngiulliBFP14,10.1093/logcom/exu053}. 

When reasoning from a knowledge base, \emph{minimal models} provide a natural and intuitive counterpart to traditional open-world semantics and classical entailment, which can easily exclude some expected consequences (e.g., a query may be not entailed due to a counter-example model that includes unexpected and unjustified facts).
   In contrast, considering only those models in which all facts are strictly necessary and justified may lead to more intuitive reasoning.
   We illustrate this in a very simple example.

 \begin{example}
 Under the standard semantics, the inclusion $\mathsf{ScandCountry} \ISA \mathsf{NatoMember}$ is not entailed by the following six assertions, since there may be unknown Scandinavian countries that are not in NATO:
$\mathsf{ScandCountry}(no)$, $\mathsf{ScandCountry}(se)$, $\mathsf{ScandCountry}(dk)$, $\mathsf{NatoMember}(no)$, $\mathsf{NatoMember}(se)$, $\mathsf{NatoMember}(dk)$.
However, the entailment does hold under the minimal model semantics; equivalently, the concept $\mathsf{ScandCountry} \sqcap \lnot \mathsf{NatoMember}$ is not satisfiable in the minimal models.
    \end{example}

Despite the strong motivation, there are still big gaps in our understanding of minimal model reasoning in \emph{Description Logics (DLs)}. 
Predicate minimization has been explored in the context of \emph{circumscription in DLs}, but most existing results spell out the high complexity that results from combining minimized predicates with varying or fixed predicates; see, e.g.,\,\cite{BonattiLW09,LutzMN23}. Specifically, when varying predicates are allowed  (e.g.\,satisfiability of a concept for general circumscription in $\ALC$), reasoning becomes quickly undecidable. 
But the case of purely \emph{minimal models}, where nothing can be removed from the extension of any predicate while preserving modelhood, remained largely unexplored. It was however established recently that for the DL $\ELIO$---a relatively expressive DL with \textsc{ExpTime}-complete concept satisfiability problem for classical semantics---basic minimal model reasoning becomes undecidable~\cite{DistefanoS24}. A positive result was established for $\dllitecore$: here minimal model reasoning exhibits the same worst-case complexity as in the classical case~\cite{BonattiD0S23}. It is thus natural to explore whether similar positive or negative results can be obtained for other lightweight DLs like $\EL$ or other DL-Lite variants.

In this paper we investigate these questions, and provide the following contributions:
\begin{itemize}
\item We show that concept satisfiability in a minimal model is undecidable for the DL $\EL$. 
The decidability status of minimal model reasoning 
has been open for several years, and the negative outcome is somewhat surprising. 
It
contrasts with
 the complexity of the
  classical semantics for $\EL$, which supports tractable reasoning
  for basic reasoning tasks. Our undecidability proof does not use the $\top$-concept, and it thus carries over to \emph{guarded tuple generating dependencies (TGDs)} of very restricted shapes.
\item To regain decidability, we impose two types of \emph{acyclicity
    conditions} on the TBox, which are defined in terms of a
  dependency graph on the predicates of a knowledge base.  If we
  restrict our attention to \emph{strongly acyclic} TBoxes, we can
  import results from \emph{pointwise circumscription}
  of~\cite{DistefanoOS23} to show that $\ELIO$ not only becomes
  decidable, but is feasible in non-deterministic exponential time.
  We also explore \emph{weak acyclicity}, a common notion in the
  setting of TGDs in database theory
  \cite{FAGIN200589,10.1145/3034786.3034794,jair/GrauHKKMMW13}. Weakly acyclic $\EL$ and
  $\ELIO$ remain decidable; we get tight $\NExp^\NP$ bounds on
  combined complexity, and $\Sigma_2^P$ on data complexity.

    \item 
We conclude the paper with a minor excursion into DL-Lite, but even there we find a challenging panorama: satisfiability in minimal models is already $\ExpSpace$-hard for $\dllitehorn$.
\end{itemize}

\noindent
The present document, with full proofs in the appendix, is the long version of a paper published at KR 2025.

	\section{Preliminaries}

    We briefly recall the syntax and semantics of DLs studied in this paper and refer to \cite{bookdls} for more details. 
    We consider countably infinite pairwise disjoint sets  
$\cnames$, 
$\rnames$ and  $\inames$ of \emph{concept}, \emph{role} and \emph{individual names}, respectively, and use $\rnamespm$ to denote the set $\rnames \cup \{r^- \mid r \in \rnames \}$. \emph{Concepts}
in $\ALCIO$ follow the syntax $C := A \mid \{a\} \mid \top \mid  \lnot C \mid C \sqcap C \mid \exists r.C$, where 
$A \in \cnames$, 
$r\in \rnamespm$ and  $a \in \inames$. 
By removing \emph{negated concepts} $\lnot C$ from this grammar, we obtain concepts in $\ELIO$; by 
 removing \emph{nominals} $\{a\}$ and requiring $r\in \rnames$ we obtain $\ALC$. The intersection of 
 the $\ELIO$ and $\ALC$ concept languages is called $\EL$. 
 In $\ELIO_\bot$ and $\ELbot$ we extend $\ELIO$ and $\EL$ with negation, but only in concepts of the form $\lnot \top$, which we equivalently write $\bot$.  
In $\ALCIO$ we use $C \sqcup D$ as a shortcut for  
$\lnot (\lnot C \sqcap \lnot D)$ and 
 $\forall r.C$ as a shortcut for 
 $\lnot(\exists r.\lnot C)$. 

 Let $\Lmc$ be a DL. 
A \emph{TBox} $\Tmc$ (in $\Lmc$) is a finite set of \emph{concept inclusions} $C \sqsubseteq D$ where $C$ and $D$ are concepts in $\Lmc$. 
An \emph{ABox} $\Amc$ is a finite set of \emph{assertions} of the forms $A(a)$ and $r(a,b)$ with $A \in \cnames$, 
$r\in \rnamespm$ and  $a,b \in \inames$. 
A pair $\Kmc=(\Tmc,\Amc)$ of a TBox and an ABox is a \emph{knowledge base (KB)}. 

The semantics of DLs is defined using interpretations $\Imc = (\Delta^\Imc,\cdot^\Imc)$, where $\Delta^\Imc$ is a non-empty domain and the interpretation function $\cdot^\Imc$ maps 
each $A \in \cnames$ to a set $A^\Imc \subseteq \Delta^\Imc$,  
each $r \in \rnames$ to a set of pairs $r^\Imc \subseteq \Delta^\Imc \times \Delta^\Imc$, and  
each $a \in \inames$ to an element $a^\Imc \in \Delta^\Imc$. The interpretation function extends to all concepts as usual, and we call $\Imc$ a \emph{model} of a concept $C$ if $C^\Imc \neq \emptyset$.
For $\alpha$ a concept inclusion, assertion, TBox, ABox or KBs, \emph{modelhood} $\Imc \models \alpha$ is standard. 
We say that $\Imc$ makes the \emph{unique name assumption (UNA)} if $a^\Imc \neq b^\Imc$ for every $a,b \in \inames$ with $a \neq b$. 
When considering $\EL$ and $\ELIO$,
we make the UNA unless stated otherwise. 
In DLs containing $\ELbot$ this assumption is irrelevant since the UNA can be simulated in the usual way.  

 \begin{definition}\label{minimal:model}
    For interpretations $\Imc$ and $\Jmc$, we let $\Imc\subseteq \Jmc$ if 
        \begin{enumerate}[(i)]
            \item $\Delta^{\Imc}=\Delta^{\Jmc}$ and $a^{\Imc}=a^{\Jmc}$ for all $a\in \inames$;
            \item $p^{\Imc}\subseteq p^{\Jmc}$ for all predicates $p\in \cnames \cup \rnames$.
        \end{enumerate}

        We write $\Imc\subsetneq \Jmc$ if $\Imc\subseteq \Jmc$ and $p^{\Imc}\subsetneq p^{\Jmc}$
        for some $p\in \cnames \cup \rnames$. 
          We call $\Imc$ a \emph{minimal model} of a KB $\Kmc$, if (a)
            $\Imc\models \Kmc$,  and (b) there exists no $\Jmc\subsetneq \I$ such that $\Jmc\models \Kmc$.

        \end{definition}

Interpretations with different domains are not comparable according to this definition, which coincides with the preference relation induced by a circumscription pattern where all predicates are minimized \cite{BonattiLW09}.

The reasoning task that we focus on is \emph{concept satisfiability in a minimal model} (\mms for short) defined as follows: Given 
 an $\Lmc$ KB $\Kmc$ and an $\Lmc$ concept $C$, decide whether there exists a minimal model $\I$ of $\Kmc$ with $C^\Imc \neq \emptyset$.

    \begin{example}
      Take a TBox $\T$ stating that (movie) \emph{fans} must
      \emph{like} some movie, while \emph{critics} always
      \emph{dislike} something:
      \[ \mathsf{Fan} \sqsubseteq \exists \mathsf{likes}
        . \mathsf{Movie} \qquad \mathsf{Critic} \sqsubseteq \exists
        \mathsf{dislikes} .\top \] Consider also ABoxes as follows:
      \[\A_1= \{\mathsf{Fan}(ann)\}\qquad \A_2=
        \{\mathsf{Fan}(ann),\mathsf{Critic}(bob)\} \] We are 
      interested in the satisfiability of the concept
      $C = \mathsf{Movie}\sqcap \exists \mathsf{dislikes}^{-}.\top$,
      i.e.\ the existence of a movie that is disliked by someone.
      Observe that $C$ is not satisfiable in a minimal model of
      $\K_1=(\T,\A_1)$, because $\K_1$ has no justification
      of an object (person) that dislikes something. However, $C$ is
       satisfiable in a minimal model of $\K_2=(\T,\A_2)$ (in
      this model $ann$ likes a movie that $bob$ dislikes).
    \end{example}

{Since traditional reductions between basic reasoning tasks do not directly apply to minimal model reasoning, we do not study them here  and we focus on concept satisfiability only.}

	\section{Undecidability of \mms}

        Before we present our main results, we first provide as a `warm-up' a proof of $\Sigma^P_2$-hardness in data complexity of \mms in $\mathcal{ALC}$. The proof is not presented for the complexity result, which is subsumed by tighter bounds in the following sections, but to provide a gentle introduction to the \emph{flooding} technique that will be used heavily in the later reductions. This technique, known as  \emph{saturation} in disjunctive logic programming~\cite{DBLP:journals/amai/EiterG95}, simulates  the universal quantification required for minimization, i.e., testing that \emph{all} substructures are \emph{not models}. Intuitively, a ``flooded''  interpretation contains objects that satisfy a given disjunctive concept in more than one way. At the core of this are \emph{cyclic dependencies} between some concept names  $A_1,A_2$ that may appear together in some disjunction $A_1\sqcup A_2$  on the right-hand-side of a concept inclusion.
        Intuitively, verifying that  $e\in (A_1\sqcap A_2)^\I$ holds in a minimal model $\I$ may require a \emph{case analysis}: we may need to check that $e\in A_1^\I$ implies $e\in A_2^\I$, and that $e\in A_2^\I$ implies $e\in A_1^\I$.          
        Such case-based verification  can be used for testing for crucial properties (errors in a coloring, in a grid construction, etc.), and a flooded minimal model implies that every possible way of avoiding the flooding failed, thus implicitly quantifying over the domain of the structure. 

\begin{figure}
  \centering
            \includegraphics[width=1\linewidth]{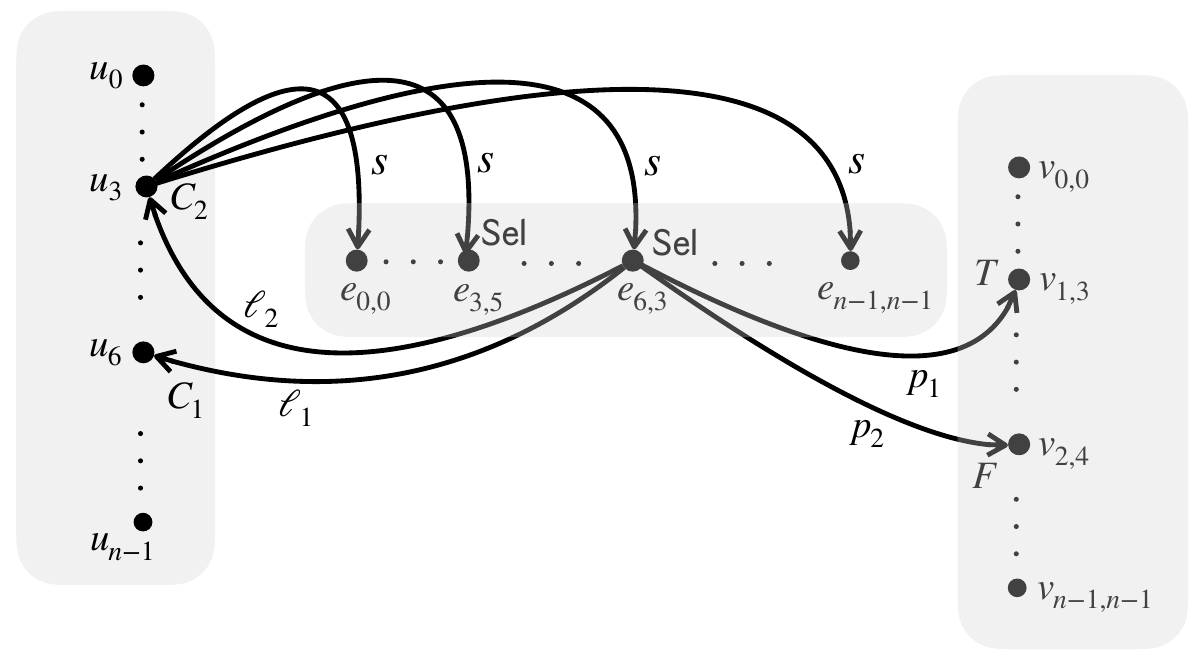}
            \caption{The structure of an ABox $\A$ in
              Example~\ref{example:flooding-technique}. Roles
              $p_1,p_2$ indicate that the first and second 
              literals in the labeling of the edge $(6,3)$ are
              $v_{1,3}$ and $v_{2,4}$, respectively. The figure also
              indicates an interpretation  where $v_{1,3}$ and $ v_{2,4}$ are
              respectively true and false, while the vertices $6$ and
              $3$ are colored with colors $C_1$ and $C_2$. 
              The presence of $\mathsf{Sel}$ indicates that the disjunctions of literals at edges $\{3,5 \}$ and $\{ 6,3 \}$ evaluate to true in the given interpretation.}
  \label{fig:example}
\end{figure}
\begin{example} 
          \label{example:flooding-technique} 
We show a reduction from (the complement of) \ccol,
a $\Pi^P_2$-hard problem \cite{logcom/stewart91}, to \mms. As in~\cite{BonattiLW09}, we define an instance of \ccol (of size $n$) as an undirected graph $G$ with vertices $\{0, 1,\ldots,n-1\}$
        where each edge is labeled with a disjunction of two literals over the Boolean variables $\{V_{i,j} \mid i, j < n\}$. $G$ is a positive instance of \ccol if for every truth assignment $t$, the subgraph $G_t$ of $G$ that contains those edges whole labels evaluate to \emph{true} under $t$ is 3-colorable. 

          We represent $G$ as an ABox $\A$. 
          We then provide a fixed TBox $\T$ and concept $D$ such that $D$ is satisfiable in a minimal model of $(\T, \A)$ iff $G$ is \emph{not} a positive instance of 
          \ccol.
          The stucture of $\A$ is illustrated in Figure~\ref{fig:example}.
          \\[3pt]\noindent
          \textbf{Creating the individuals.} To represent the vertices of $G$,
          we add the assertions $N(u_0),\ldots,N(u_{n-1})$ to
          $\A$. Similarly, to represent the propositional variables,
          we include the assertion $V(v_{i,j})$ for all
          $0 \leq i,j < n$. To represent the edges of $G$, we add to
          $\A$ the assertion $E(e_{i,j})$, for all $0 \leq i,j <
          n$.
		  \\[3pt]\noindent
          \textbf{Connecting the individuals.} We use two 
          roles $\ell_1,\ell_2$ to connect each $e_{i,j}$ to the
          individuals that represent the vertices $i$ and
          $j$: we add $\ell_1(e_{i,j},u_i)$ and
          $\ell_2(e_{i,j},u_j)$ for all $0 \leq i,j < n$. Similarly,
          using the roles $p_1,p_2$, we connect each $e_{i,j}$ to the two individuals that represent the literals in the label between $i$ and $j$: we
          add $p_1(e_{i,j},v_{k_1,k_2})$ and
          $p_2(e_{i,j},v_{k_3,k_4})$, where $v_{k_1,k_2} $ and
          $v_{k_3,k_4}$ are the first and second variables in the label of the edge between $i$ and $j$, respectively. In addition, if $v_{k_1,k_2}$ occurs positively (resp., negatively) in the label, we add $\mathsf{pos}_1(e_{i,j})$ (resp., $\mathsf{neg}_1(e_{i,j})$)
          to $\A$. The second literal is treated similarly, adding 
          $\mathsf{pos}_2(e_{i,j})$ or $\mathsf{neg}_2(e_{i,j})$
          depending on whether $v_{k_3,k_4}$ is positive or negative in the edge represented by $e_{i,j}$. 
          Finally, we use a role $s$ to connect each individual $u_k$ to all the individuals $e_{i,j}$ representing an edge: we
          add $s(u_k,e_{i,j})$ to $\A$, for all $0 \leq k,i,j < n$.
          This completes
          the construction of $\A$.
		  \\[3pt]\noindent
          \textbf{Building the TBox.} 
          $\T$ consists of the following inclusions; note that these do not depend on the instance $G$. 
          We  introduce two disjunctions that enforce in each interpretation a `guess' of a truth value for each individual representing a propositional
          variable, and a `guess' of a color assignment for each 
          individual corresponding to a vertex. 
          \[V\sqsubseteq T \sqcup F \qquad N \sqsubseteq C_1\sqcup
            C_2\sqcup C_3\]
            With the following pair of inclusions we select
          the edges of $G$ that are labeled with a disjunction that evaluates to \emph{true}:
          \[ (\mathsf{pos}_i\sqcap \exists p_i.T) \sqcup
            (\mathsf{neg}_i\sqcap \exists p_i.F) \sqsubseteq
            \mathsf{Sel} \mbox{\quad for~} i\in\{1,2\} \] 
        Now comes the interesting part: if a selected edge is wrongly colored (i.e.\,the same color is found at both ends), then all vertices are ``flooded'' with all
          colors.
 \[\exists s.( \mathsf{Sel}  \sqcap \exists \ell_1. C_i \sqcap \exists  \ell_2. C_i  )\sqsubseteq  C_1\sqcap C_2\sqcap C_3  \mbox{~~for~}  i\in\{1,2,3\}  \]
   Finally, we take $C_1\sqcap C_2$ as our goal
   concept, which is meant to detect `flooding' of the structure. 
   It can be easily verified that 
   $(C_1\sqcap C_2)$ is satisfiable in a minimal model of $(\T, \A)$ iff $G$ is \emph{not} a positive instance of \ccol.
\end{example}
 
This reduction will be adapted below to show 
$\Sigma^{P}_{2}$-hardness in data complexity for \emph{weakly acyclic} $\mathcal{EL}$ KBs.
However, we want to stress that there are no
existential concepts on the right-hand-side of inclusions  in Example~\ref{example:flooding-technique},  and 
that all these inclusions can be written in \emph{Disjunctive Datalog (DD)}
using rules where all relations  have
arity $\leq 2$.
Thus we 
(slightly) strengthen the $\Sigma^{P}_{2}$-hardness proof for data complexity of DD in~\cite{DBLP:journals/tods/EiterGM97}, which uses a relation of arity 5. 

\smallskip
	
We  now show our first and most surprising major result: minimal model reasoning is undecidable already in $\EL$. 
	\begin{theorem}
		\label{theorem:el-is-undecidable-without-top}
		$\mms$ in $\EL$ is undecidable.
		This holds even if the $\top$-concept is disallowed.
	\end{theorem}

		\newcommand{\start}{\cstyle{Start}}
		\newcommand{\finish}{\cstyle{Finish}}
		\newcommand{\continue}{\cstyle{Continue}}
		\newcommand{\hcentral}{\cstyle{HCentral}}
		\newcommand{\vcentral}{\cstyle{VCentral}}
		\newcommand{\inner}{\cstyle{Inner}}
		\newcommand{\croot}{\cstyle{Root}}
		\newcommand{\node}{\cstyle{Node}}
		\newcommand{\hrole}{\rstyle{h}}
		\newcommand{\vrole}{\rstyle{v}}
		\newcommand{\leaf}{\rstyle{Leaf}}
		\newcommand{\srole}{\rstyle{s}}
		\newcommand{\rrole}{\rstyle{r}}
		\newcommand{\chosen}{\cstyle{Ch}}
		\newcommand{\tile}{\cstyle{Tile}}
		\newcommand{\istile}{\cstyle{isT}}
		\newcommand{\istiled}{\cstyle{isTiled}}
		\newcommand{\south}{\cstyle{S}}
		\newcommand{\east}{\cstyle{E}}
		\newcommand{\north}{\cstyle{N}}
		\newcommand{\west}{\cstyle{W}}
		\newcommand{\positions}{\mathsf{Pos}}
		\newcommand{\ispos}{\cstyle{is}}
		\newcommand{\newtop}{\cstyle{Any}}
		\newcommand{\border}{\cstyle{Border}}
		\newcommand{\hend}{\cstyle{HEnd}}
		\newcommand{\vend}{\cstyle{VEnd}}
		\newcommand{\hpos}{\rstyle{hpos}}
		\newcommand{\vpos}{\rstyle{vpos}}
		\newcommand{\real}{\cstyle{Real}}
		\newcommand{\error}{\cstyle{Err}}
		\newcommand{\finite}{\cstyle{Finite}}
		\newcommand{\subgoal}{\cstyle{Subgoal}}
		\newcommand{\rpos}{\rstyle{pos}}
		\newcommand{\ctop}{\rstyle{Any}}
		\newcommand{\rh}{\rstyle{h}}
		\newcommand{\rv}{\rstyle{v}}
		\newcommand{\rx}{\rstyle{s}}
		\newcommand{\cflood}{\rstyle{Flood}}
		\newcommand{\hastile}{\rstyle{tile}}
		
		\newcommand{\cN}{\cstyle{N}}
		\newcommand{\cS}{\cstyle{S}}
		\newcommand{\cE}{\cstyle{E}}
		\newcommand{\cW}{\cstyle{W}}
		\newcommand{\cC}{\cstyle{C}}
		\newcommand{\cNE}{\cstyle{NE}}
		\newcommand{\cNW}{\cstyle{NW}}
		\newcommand{\cSE}{\cstyle{SE}}
		\newcommand{\cSW}{\cstyle{SW}}
		\newcommand{\rspy}{\rstyle{spy}}
		\newcommand{\cgoh}{\cstyle{H}}
		\newcommand{\cgov}{\cstyle{V}}
		\newcommand{\cx}{\cstyle{isX}}
		\newcommand{\cX}{\cstyle{X}}
		\newcommand{\cpx}{\cstyle{PX}}
		\newcommand{\raux}{\rstyle{aux}}
		\newcommand{\blank}{b}
		\newcommand{\tiles}{T}
		\newcommand{\goal}{\cstyle{Goal}}
		\newcommand{\valid}{\cstyle{GoodT}}
		\newcommand{\hvalid}{\cstyle{HGoodT}}
		\newcommand{\vvalid}{\cstyle{VGoodT}}
		\newcommand{\bvalid}{\cstyle{BGoodT}}
		\newcommand{\hconstraints}{\ensuremath{H}}
		\newcommand{\vconstraints}{\ensuremath{V}}
		\newcommand{\goodpos}{\cstyle{GoodP}}
		\newcommand{\hgoodpos}{\cstyle{HGoodP}}
		\newcommand{\vgoodpos}{\cstyle{VGoodP}}

\newcommand{\recttile}{\textsc{RectTile}\xspace}

		We reduce from \recttile, the rectangular tileabilty problem, known to be undecidable \cite{DBLP:journals/ejc/Yang14}:
given a set $\tiles$ of Wang tiles ({a.k.a.}\ dominos) and a special color $\blank$, decide whether $\tiles$ tiles some finite rectangle with $\blank$ on its sides.
		
		Consider an instance $(\tiles, \blank)$ of \recttile.
		We construct an $\EL$ KB $\kb = (\tbox, \abox)$ such that $(\tiles, \blank) \in \recttile$ iff the concept $\goal$ is minimally satisfiable w.r.t.\ $\kb$.
		
		Our main challenge is to guarantee that every minimal model satisfying $\goal$ features a rectangular grid representing a tiling of some rectangle.
		Elements in the grid, further referred to as \emph{nodes}, are identified by the concept $\node$.
		We distinguish nine types of positions in the grid, identified using abbreviations of \emph{north}, \emph{south}, \emph{east} and \emph{west}: being one of the four corners ($\cNE$, $\cNW$, $\cSE$, $\cSW$), lying on one of the four borders ($\cS$, $\cN$, $\cE$, $\cW$), or lying in the \emph{central} part of the grid ($\cC$).
		The following displays an example of how nodes (denoted by their positions) are intended to be arranged:
		\[
		\begin{array}{cccccc}
			\cNW & \cN & \cN & \cN & \cN & \cNE
			\\
			\cW & \cC & \cC & \cC & \cC & \cE
			\\
			\cW & \cC & \cC & \cC & \cC & \cE
			\\
			\cSW & \cS & \cS & \cS & \cS & \cSE
		\end{array}
		\]
		We set $\positions := \{ \cC, \cS, \cN, \cE, \cW, \cSE, \cSW, \cNE, \cNW \}$ the set of those nine concept names.
		We say that $p$ admits $p'$ as a valid horizontal successor, denoted $p \leadsto_h p'$ if, in the above example, $p$ appears on the same row and more to the left than some $p'$ (\emph{e.g.} $\cW \leadsto_h \cC$ and $\cW \leadsto_h \cE$, but $\cW \not\leadsto_h \cSE$).
		Similarly, we say that $p$ admits $p'$ as a valid vertical successor, denoted $p \leadsto_v p'$ if, in the above example, $p$ occurs on the same column and below some $p'$.
		Each of these concepts is made available on a dedicated individual in the ABox by adding an assertion $p(a_p)$ for each $p \in \positions$.
		
		A node $e$ having position $p$ is represented by a connection from $e$ to $a_p$ by the role $\rpos$.
		Each node justifies such a connection using the axiom $\node \sqsubseteq \exists \rpos.\ctop$, where $\ctop$ is an auxiliary concept name that is used in place of $\top$.
		Note that there is no guarantee that a node connects to one of the individuals $a_p$.
		However, such a connection is required to justify further nodes in the grid.
		Indeed, only the south-west corner is provided as part of the ABox, with assertions $\node(a)$, $\newtop(a)$ and $\rpos(a, a_\cSW)$.
		Further nodes are generated by existing nodes with a valid position and as horizontal or vertical neighbors, which is represented by respective roles $\rh$ and $\rv$.
		To this end, we add the following axioms:
		\begin{align*}
			\node \sqcap \exists \rpos.p & \sqsubseteq \exists \rh.\node
			& & \textrm{ for } p \in \positions \setminus \{ \cSE, \cE, \cNE \}
			\\
			\node \sqcap \exists \rpos.p & \sqsubseteq \exists \rv.\node
			& & \textrm{ for } p \in \positions \setminus \{ \cNW, \cN, \cNE \}.
		\end{align*}
		With only these axioms, it should be clear that, in a minimal model, every instance of the $\node$ concept in $\Imc$ is generated following a path of roles $\hrole$ and $\vrole$ from $a$ and has at most one $\rh$-successor and at most one $\rv$-successor that is a node.
		Given an interpretation $\Imc$, we say that a sequence $d_0, r_1, d_1 \dots, r_n, d_n$ is an $\hrole$-$\vrole$-path of nodes from $a$ to $d_n$ if:
		\begin{itemize}
			\item $d_0 = a^\Imc$ and for every $0 \leq i \leq n$, we have $d_i \in \node^\Imc$;
			\item for every $1 \leq i \leq n$, $r_i \in \{ \hrole, \vrole \}$ and $(d_{i-1}, d_i) \in {r}_i^\Imc$. 
		\end{itemize}
        {The following claim holds, also in the presence of the remaining inclusions that we add further in this construction:}
		\begin{claimrep}
			\label{claim:h-v-paths}
			Let $\Imc$ be a minimal model of $\kb$.
			If $d \in \node^\Imc$, then there exists an $\hrole$-$\vrole$-path of nodes from $a$ to $d$.
			Furthermore, there exists at most one element $e \in \node^\Imc$ such that $(d, e) \in \rh^\Imc$.
			The same holds for role $\rv$.
		\end{claimrep}
		
		\begin{proof}[Proof of Claim~\ref{claim:h-v-paths}]
			Let $\Imc$ be a minimal model of $\kb$ and $d \in \node^\Imc$.
			
			For the first part of the claim, assume by contradiction that there is no $\hrole$-$\vrole$-path from $a$ to $d$.
			Consider the interpretation $\Jmc$ obtained from $\Imc$ by removing from $\node^\Imc$ all such $e$ that do not have a $\hrole$-$\vrole$-path from $a$ to $e$.
			It is readily verified that $\Jmc$ is a model of $\kb$.
			Furthermore, from the assumption that $d$ is such an element, we have $\Jmc \subsetneq \Imc$, which contradicts the minimality of $\Imc$.
			
			For the second part of the claim, assume by contradiction that there exists more than one element $e \in \node^\Imc$ such that $(d, e) \in \rh^\Imc$.
			We distinguish three cases: if $d \in (\exists \hrole.(\node \sqcap \cpx))^\Imc$, then we can select $e_0 \in (\cpx \sqcap \node)^\Imc$ such that $(d, e_0) \in \hrole^\Imc$.
			Else, if $d \notin (\exists \hrole.(\node \sqcap \cpx))^\Imc$ but $d \in (\exists \hrole.\node)^\Imc$, then we select $e_0$ as $e_0 \in \node^\Imc$ such that $(d, e_0) \in \hrole^\Imc$.
			In both cases, note that, by assumption, there must exist some $e_1 \neq e_0$ such that $(d, e_1) \in \hrole^\Imc$ and $e_1 \in \node^\Imc$.
			We then define an interpretation $\Jmc$ obtained from $\Imc$ by dropping all $(d, e)$ from $\hrole^\Imc$ except for $(d, e_0)$.
			It is readily verified that $\Jmc$ is a model of $\kb$, notably using the careful choice of the selected $e_0$.
			From the existence of $e_1$, we also obtain $\Jmc \subsetneq \Imc$, contradicting the minimality of $\Imc$.
			
			The case of $\vrole$ is similar.
		\end{proof}
		
		Using the same mechanism as to `choose' positions, we ask each node to choose a tile with the following axiom $\node \sqsubseteq \exists \hastile.\newtop$ and assertions $t(a_t)$ for each $t \in \tiles$.
		
		We now require the assignments of positions and tiles to be consistent in a minimal model to satisfy respective subgoal-concepts $\subgoal_1$ and $\subgoal_2$.
		More formally, regarding positions, we aim for the following claim:
		\begin{claimrep}
			\label{claim:positions}
			Let $\Imc$ be a minimal model of $\kb$ s.t.\ $a^\Imc \in \subgoal_1^\Imc$.
			Then, for every $d \in \node^\Imc$, there exists a unique $p \in \positions$, that we denote $pos(d)$, such that $(d, a_p) \in \rpos^\Imc$.
			Furthermore, for every $d, e \in \node^\Imc$, the following properties hold:
			\begin{enumerate}
				\item if $(d, e) \in \hrole^\Imc$, then $pos(d) \leadsto_h pos(e)$;
				\item if $(d, e) \in \vrole^\Imc$, then $pos(d) \leadsto_v pos(e)$;
				\item for every $\hrole$-$\vrole$-path $d_0, r_1, d_1 \dots, r_n, d_n$ of nodes from $a$, there exists (a potentially longer) one $d_0, r_1, d_1 \dots, r_n, d_n, \dots, r_{n + k}, d_{n + k}$, with $k \geq 0$, and such that $pos(d_{n+k}) =\cNE$;
				\item if $pos(d) = \cNE$, then $d \notin (\exists \hrole)^\Imc$ and $d \notin (\exists \vrole)^\Imc$.
			\end{enumerate}
		\end{claimrep}
		
		\begin{proof}[Proof of Claim~\ref{claim:positions}]
			Let $\Imc$ be a minimal model of $\kb$ such that $a \in \subgoal_1^\Imc$.
			That every node $d \in \node^\Imc$ has at most one successor by $\rpos$ is straightforward as otherwise one could simply drop the $\rpos$-edge to any extra-successor while retaining modelhood.
			That this successor is among the $a_p$ for $p \in \positions$ comes from $a^\Imc \in \subgoal_1^\Imc$.
			If the said successor was not among the $a_p$'s then we can obtain a model $\Jmc$ of $\kb$ by simplifying the interpretation of concepts $\hgoodpos$, $\vgoodpos$ and $\goodpos$ along the $\hrole$-$\vrole$-path from $a$ to $d$ (note that this path must exists from Claim~\ref{claim:h-v-paths}), and removing $a^\Imc$ from $\subgoal_1^\Imc$, thus obtaining $\Jmc \subsetneq \Imc$, which contradicts the minimality of $\Imc$.
			
			For Point~1, assume by contradiction that there exists $(d, e) \in \hrole^\Imc$ such that for all positions $p, p' \in \positions$ with $p \leadsto_h p'$, either $d \notin (\exists \rpos.p)^\Imc$ or $e \notin (\exists \rpos.p')^\Imc$.
			Note that $(d, e) \in \hrole^\Imc$ guarantees that $d \in \exists \rpos.q$ for $q \in \positions \setminus \{ \cSE, \cE, \cNE \}$, as otherwise we could remove $(d, e)$ from $\hrole^\Imc$ and obtain a model $\Jmc \subsetneq \Imc$ contradicting the minimality of $\Imc$.
			From this remark and the assumption, it follows that neither the first nor the second rules from Figure~\ref{figure:extra-rules-positions} triggers.
			It is clear that $d \in \node^\Imc$ as otherwise we could remove $(d, e)$ from $\hrole^\Imc$.
			Now, from Claim~1, there exists a $\hrole$-$\vrole$-path $a = d_0, r_1, d_1, \cdots, r_n, d_n = d$ from $a$ to $d$.
			We construct an interpretation $\Jmc$ by removing $a$ from $\subgoal_1^\Imc$ and $d_k$ from:
			\begin{itemize}
				\item $\goodpos^\Imc$ (always); and from
				\item $\hgoodpos^\Imc$ if $(d_k, d_{k+1}) \in \hrole^\Imc$; and from
				\item $\vgoodpos^\Imc$ if $(d_k, d_{k+1}) \in \vrole^\Imc$.
			\end{itemize}
			Using the second part of Claim~1, it is readily verified that $\Jmc$ is a model of $\kb$.
			Furthermore, since $a \notin \subgoal_1^\Jmc$, we have $\Jmc \subsetneq \Imc$, contradicting the minimality of $\Imc$.
			
			Point~2 is proved similarly.
			
			For Point~3, assume by contradiction that we could find an infinite $\hrole$-$\vrole$-path $a = d_0, r_1, d_1, r_2, d_2, \dots$ starting from $a$ and that never encounters a node whose position is $\cNE$.
			We construct an interpretation $\Jmc$ by removing $a$ from $\subgoal_1^\Imc$ and $d_k$ from:
			\begin{itemize}
				\item $\goodpos^\Imc$ (always); and from
				\item $\hgoodpos^\Imc$ if $(d_k, d_{k+1}) \in \hrole^\Imc$; and from
				\item $\vgoodpos^\Imc$ if $(d_k, d_{k+1}) \in \vrole^\Imc$.
			\end{itemize}
			Using the second part of Claim~1, it is readily verified that $\Jmc$ is a model of $\kb$.
			Furthermore, since $a \notin \subgoal_1^\Jmc$, we have $\Jmc \subsetneq \Imc$, contradicting the minimality of $\Imc$.
			
			Point~4 simply follows from the observation that nodes with position $\cNE$ do not require $\hrole$ nor $\vrole$ successors.
		\end{proof}

		\begin{figure*}
			\begin{align*}
				\exists \rpos.p & \sqsubseteq \hgoodpos
				& & \textrm{ for } p \in \{ \cSE, \cE, \cNE \}
				& & &
				\exists \rpos.p \sqcap \exists \hrole.(\goodpos \sqcap \exists \rpos.p') & \sqsubseteq \hgoodpos
				& & \textrm{ for } p, p' \in \positions \textrm{ s.t.\ } p \leadsto_h p'
				\\
				\exists \rpos.p & \sqsubseteq \vgoodpos
				& & \textrm{ for } p \in \{ \cNW, \cN, \cNE \}
				& & &
				\exists \rpos.p \sqcap \exists \vrole.(\goodpos \sqcap \exists \rpos.p') & \sqsubseteq \vgoodpos
				& & \textrm{ for } p, p' \in \positions \textrm{ s.t.\ } p \leadsto_v p'
			\end{align*}
			\vspace*{-2.8ex}
			\[
			\hgoodpos \sqcap \vgoodpos \sqsubseteq \goodpos
			\qquad \qquad
			\goodpos \sqcap \croot \sqsubseteq \subgoal_1
			\]
			\caption{Additional rules to guarantee consistent positions.}
			\label{figure:extra-rules-positions}
		\end{figure*}
		To achieve the above, we add the rules in Figure~\ref{figure:extra-rules-positions}.
		Intuitively, we ensure that $\subgoal_1$ can only be obtained if a concept $\goodpos$ is seen at the root $a$, which is identified with a dedicated assertion $\croot(a)$.
		To derive $\goodpos$ at a given node $e$, we require all its node successors (there are at most two, due to Claim~\ref{claim:h-v-paths}) to already satisfy $\goodpos$ and for their respective positions to be valid successors of the position of $e$.
		In particular, nodes with position $\cNE$ trivially satisfy this condition as they expect no successors.
		Note that this also guarantees Point~3 in the above claim, as nodes along arbitrary long $\hrole$-$\vrole$-path of nodes starting from $a$ need not to satisfy $\goodpos$.
		
		\begin{toappendix}
			Figure~\ref{figure:extra-rules-tiles} presents the rules ensuring the consistency of the tiling, omitted in the main part of the paper as they are very similar in nature to those presented on Figure~\ref{figure:extra-rules-positions}.
			\begin{figure*}
				\centering
				\begin{align*}
					\exists \rpos.p & \sqsubseteq \hvalid
					& & \textrm{ for each } p \in \{ \cSE, \cE, \cNE \}
					\\
					\exists \rpos.p \sqcap \exists \hastile.{t} \sqcap \exists \hrole.(\valid \sqcap \exists \hastile.{t'}) & \sqsubseteq \hvalid
					& & \textrm{ for each } p \in \positions \setminus \{ \cSE, \cE, \cNE \} \textrm{ and } (t, t') \in \hconstraints
					\\
					\exists \rpos.p & \sqsubseteq \vvalid
					& & \textrm{ for each } p \in \{ \cNW, \cN, \cNE \}
					\\
					\exists \rpos.p \sqcap \exists \hastile.{t} \sqcap \exists \hrole.(\valid \sqcap \exists \hastile.{t'}) & \sqsubseteq \vvalid
					& & \textrm{ for each } p \in \positions \setminus \{ \cNW, \cN, \cNE \} \textrm{ and } (t, t') \in \vconstraints
					\\
					\exists \rpos.\cC & \sqsubseteq \bvalid
					\\
					\exists \rpos.\cNE \sqcap \exists \hastile.{t} & \sqsubseteq \bvalid
					& & \textrm{for each } t \in \tiles \text{ s.t.\ the north and east colors of } t \text{ are } \blank
					\\
					\exists \rpos.\cN \sqcap \exists \hastile.{t} & \sqsubseteq \bvalid
					& & \textrm{for each } t \in \tiles \text{ s.t.\ the north color of } t \text{ is } \blank
					\\
					& \cdots & & \cdots
					\\
					\exists \rpos.\cSW \sqcap \exists \hastile.{t} & \sqsubseteq \bvalid
					& & \textrm{for each } t \in \tiles \text{ s.t.\ the south and west colors of } t \text{ are } \blank
					\\
					\hvalid \sqcap \vvalid \sqcap \bvalid & \sqsubseteq \valid
					\\
					\valid \sqcap \croot & \sqsubseteq \subgoal_2
				\end{align*}
				\caption{Additional rules to guarantee a consistent tiling. Predicate $\bvalid$ verifies the color constraint on the borders.}
				\label{figure:extra-rules-tiles}
			\end{figure*}
		\end{toappendix}
		
		We denote $H$ the set of pairs of tiles $(t, t') \in \tiles \times \tiles$ such that the right color of $t$ is the same as the left color of $t'$, so that $(t, t') \in H$ iff $t$ is a valid immediate left-neighbor of $t'$.
		Similarly, we denote $V$ the set of pairs of tiles $(t, t') \in \tiles \times \tiles$ such that the top color of $t$ is the same as the bottom color of $t'$.
		With rules similar to those on Figure~\ref{figure:extra-rules-positions}, we can ensure consistency of the tiling if the concept $\subgoal_2$ is satisfied, which is summarized in the following claim:
		\begin{claimrep}
			\label{claim:tiles}
			Let $\Imc$ be a minimal model of $\kb$ s.t.\ $a^\Imc \in \subgoal_2^\Imc$.
			Then, for every $d \in \node^\Imc$, there exists a unique $t \in \tiles$, that we denote $tile(d)$, such that $(d, a_t) \in \hastile^\Imc$.
			Furthermore, for every $d, e \in \node^\Imc$, the following properties hold: 
			\begin{enumerate}
				\item if $(d, e) \in \hrole^\Imc$, then $(tile(d), tile(e)) \in \hconstraints$;
				\item if $(d, e) \in \vrole^\Imc$, then $(tile(d), tile(e)) \in \vconstraints$;
				\item if $pos(d) \in \positions \setminus \{ \cC \}$, then $tile(d)$ has color $\blank$ on the corresponding $pos(d)$-border(s).
			\end{enumerate}
		\end{claimrep}
		\begin{proof}
			Uniqueness of the successor by the role $\hastile$ is straightforward: if an element $d \in \node^\Imc$ has several $\hastile$-successor, then dropping any of the edge from the interpretation of $\tile^\Imc$ preserves modelhood, which contradicts minimality of $\Imc$.
			That this successor is among the $a_t$ for $t \in \tiles$, as well as Points~1, 2, 3, then follow from the assumption that $a^\Imc \in \subgoal_2^\Imc$ in the same manner as in the proof of Claim~\ref{claim:positions}.
		\end{proof}

		It now remains to address the main problem that is how to guarantee that, in minimal models satisfying $\goal$, the $\hrole$-$\vrole$-paths of nodes from $a$ collapse into an actual grid.
		The idea is to force the satisfaction of a concept $\cX$ somewhere along one of these $\hrole$-$\vrole$-paths.
		If this instance of $\cX$ is placed on a node $e$ where paths are forming the intended grid, that is $e$ is both the $\rh$-$\rv$- and the $\rv$-$\rh$-successor of an element $d$, then it triggers a \emph{flooding} concept $\cflood$.
		The concept $\cflood$ then propagates along \emph{all} $\hrole$-$\vrole$-paths and makes the concept $\cX$ also satisfied everywhere, following the intuition already highlighted in Example~\ref{example:flooding-technique}.
		If a model does not feature a proper grid, then the flooding can be avoided by placing $\cX$ somewhere the paths are not closing as a grid.
		Thus, a minimal model that does not feature a grid is not flooded.
		On the other hand, if a minimal model features a proper grid, then it is impossible to avoid the flooding and in particular, $\cflood$ holds at $a$.
		We can thus use the conjunction of $\croot$ and $\cflood$ as the final goal.
		
		We now explain how to force the placement of $\cX$ by `guessing' an $\hrole$-$\vrole$-{path} of nodes from the root.
		Being along that path is represented by a concept $\cpx$, and we force the root $a$ to satisfy this concept as soon as the previous subgoals are satisfied with the axiom:
		\[
		\croot \sqcap \subgoal_1 \sqcap \subgoal_2 \sqsubseteq \cpx
		\]
		Now, if a node satisfies the concept $\cpx$, and depending on its own position, it either satisfies $\cX$ or propagates the concept $\cpx$ either horizontally or vertically.
		Three of the corner positions actually have no choice (\emph{e.g.}\ the north-east corner cannot propagate any further thus must satisfy $\cX$):
		\begin{align*}
			\cpx \sqcap \exists \rpos.\cNE & \sqsubseteq \cX
			\\
			\node \sqcap \cpx \sqcap \exists \rpos.\cNW & \sqsubseteq \exists \rh. (\node \sqcap \cpx)
			\\
			\node \sqcap  \cpx \sqcap \exists \rpos.\cSE & \sqsubseteq \exists \rv. (\node \sqcap \cpx)
		\end{align*}
		For the remaining positions, the possible options are represented by dedicated $\rx$-successors as follows:
		\begin{align*}
			\exists \rpos.p \sqcap \cpx & \sqsubseteq \exists \rx . \cx
			& & \textrm{ for } p \in \{ \cC, \cN, \cE \}
			\\
			\exists \rpos.p \sqcap \cpx & \sqsubseteq \exists \rx . \cgoh
			& & \textrm{ for } p \in \{ \cC, \cN, \cS, \cW, \cSW \}
			\\
			\exists \rpos.p \sqcap \cpx & \sqsubseteq \exists \rx . \cgov
			& & \textrm{ for } p \in \{ \cC, \cE, \cS, \cW, \cSW \}.
		\end{align*}
		The choice is then made via an extra $\rx$-successor that may or may not collapse with the possible options: for $p \in \positions \setminus \{ \cNE, \cNW, \cSE \}$, consider the axioms:
		\begin{align*}
			\exists \rpos.p \sqcap \cpx & \sqsubseteq \exists \rx . \chosen
			\\
			\exists \rx.(\chosen \sqcap \cx) & \sqsubseteq \cX
			\\
			\node \sqcap \exists \rx.(\chosen \sqcap \cgoh) & \sqsubseteq \exists \rh . (\node \sqcap \cpx)
			\\
			\node \sqcap \exists \rx.(\chosen \sqcap \cgov) & \sqsubseteq \exists \rv . (\node \sqcap \cpx)
		\end{align*}
		Note that the two latter axioms could justify additional instances of roles $\rh$ and $\rv$.
	However, due to $\node \sqcap \cpx$ being more specific than $\node$, Claim~\ref{claim:h-v-paths} still holds.
		
		It is crucial that the (up to $3$) $\rstyle{s}$-successors above, carrying the different possibilities $\cx$, $\cgoh$, $\cgov$, do not collapse together.
		To prevent this, we use an error-detection mechanism that reports back to the root $\istyle{a}$, via the following axioms:
		\[
		\hspace*{-1.8ex}\begin{array}{c}
			\exists \rx.(\cx \sqcap \cgoh) \sqsubseteq \error \quad \exists \rx.(\cx \sqcap \cgov) \sqsubseteq \error \quad 
			\exists \rx.(\cgoh \sqcap \cgov) \sqsubseteq \error
			\smallskip\\
			\exists \rh.\error \sqsubseteq \error
			\qquad
			\exists \rv.\error \sqsubseteq \error.
		\end{array}
		\]
		If the error concept holds at $a$, we force the model to collapse in a way that cannot satisfy the final goal predicate.
		This is achieved by introducing an auxiliary element $c$ that could act as an horizontal and vertical successor node for $a$, except that $c$ misses the concept name $\node$.
		Consider the following assertions:
		\[
		\hspace*{-1.2ex}\begin{array}{c}
			\rh(a, c), \rv(a, c), \cpx(c), \rpos(c, c), \hastile(c, c), \newtop(c),
			\rspy(c, a).
		\end{array}
		\]
		Now the trick is to promote $c$ to be a node whenever $\error$ holds on $a$, which is achieved by the axiom $\exists \rspy.\error \sqsubseteq \node$.
		Notice that the interpretation $\Imc_0$ obtained by interpreting every predicate as in the ABox, except for $\node^{\Imc_0} := \{ a, c \}$, is a model of $\kb$.
		Therefore, in a minimal model $\Imc$, if the error predicate is to be seen along a $\hrole$-$\vrole$-path of nodes from $a$, then it triggers the concept $\error$ on $a$, thus $c$ is an instance of $\node$ and therefore $\Imc_0 \subseteq \Imc$, thus $\Imc = \Imc_0$ by minimality of $\Imc$.
		We summarize this latter trick in the following claim:
		\begin{claimrep}
			\label{claim:choices-are-disjoint}
			If $\Imc$ is a minimal model of $\kb$, then $(\cx \sqcap \cgoh)^\Imc = (\cx \sqcap \cgov)^\Imc = (\cgoh \sqcap \cgov)^\Imc = \emptyset$.
		\end{claimrep}
		\begin{proof}
			If $\Imc = \Imc_0$, the claim is immediate.
			Otherwise, using Claim~\ref{claim:h-v-paths} and the axioms involving $\error$, it is clear that $\node^\Imc \cap \error^\Imc = \emptyset$.
			In particular, instances of concepts $\cx$, $\cgoh$, $\cgov$ that are $\rx$-successor of a node $d \in \node^\Imc$ are distinct.
			It is easily verified that there are no other instance of $\cx$, $\cgoh$, $\cgov$ (as those three concepts can only be justified by instances of $\node$).
		\end{proof}

		\begin{figure}
			\begin{tikzpicture}[every node/.append style={font=\small, scale=1, inner sep=0pt, outer sep=0pt}, xscale=1.4, yscale=1.4, line cap=round, line join=round]
				
				\node at (0, 0) (a) [] {$a$};
				\node [above left=.1cm of a] {$\cpx$};
				\node [below left=.1cm of a] {$\croot$};
				\node (agov) at ([shift={(70:1)}]a) [label=80:{$\cgov, \chosen$}] {$\circ$};
				\node (agoh) at ([shift={(50:1)}]a) [label=20:{$\cgoh$}] {$\circ$};
				\path
				(a.center) edge [->] node [above, sloped, near end, inner sep = 2pt] {$\rx$} (agoh)
				(a.center) edge [->] node [above, sloped, near end, inner sep = 2pt] {$\rx$} (agov)
				;
				\node at (2, 0) (a10) {$\bullet$};
				\node at (4, 0) (a20) {$\bullet$}; 
				\node at (2, 2) (a11) {$\bullet$};
				\node [above left=.1cm of a11] {$\cpx$};
				\node (a11gov) at ([shift={(70:1)}]a11) [label=80:{$\cgov$}] {$\circ$};
				\node (a11isX) at ([shift={(55:1.2)}]a11) [label=above right:{$\cx$}] {$\circ$};
				\node (a11goh) at ([shift={(40:1)}]a11) [label=20:{$\cgoh, \chosen$}] {$\circ$};
				\path
				(a11.center) edge [->] node [above, sloped, near end, inner sep = 2pt] {$\rx$} (a11goh)
				(a11.center) edge [->] node [above, sloped, near end, inner sep = 2pt] {$\rx$} (a11isX)
				(a11.center) edge [->] node [above, sloped, near end, inner sep = 2pt] {$\rx$} (a11gov)
				;
				
				\node at (0, 2) (a01) {$\bullet$};
				\node [above left=.1cm of a01] {$\cpx$};
				\node (a01gov) at ([shift={(70:1)}]a01) [label=80:{$\cgov$}] {$\circ$};
				\node (a01goh) at ([shift={(50:1)}]a01) [label=20:{$\cgoh, \chosen$}] {$\circ$};
				\path
				(a01.center) edge [->] node [above, sloped, near end, inner sep = 2pt] {$\rx$} (a01goh)
				(a01.center) edge [->] node [above, sloped, near end, inner sep = 2pt] {$\rx$} (a01gov)
				;
				\node at (2, 2) (a21) {$\bullet$};
				\node at (4.2, 1.8) (a21) {$\bullet$};
				\node at (3.8, 2.2) (a21') {$\bullet$};
				\node [above left=.1cm of a21'] {$\cpx, \cX$};
				\node (a21'gov) at ([shift={(70:1)}]a21') [label=80:{$\cgov$}] {$\circ$};
				\node (a21'isX) at ([shift={(55:1.2)}]a21') [label=above right:{$\cx, \chosen$}] {$\circ$};
				\node (a21'goh) at ([shift={(40:1)}]a21') [label=20:{$\cgoh$}] {$\circ$};
				\path
				(a21'.center) edge [->] node [above, sloped, near end, inner sep = 2pt] {$\rx$} (a21'goh)
				(a21'.center) edge [->] node [above, sloped, near end, inner sep = 2pt] {$\rx$} (a21'isX)
				(a21'.center) edge [->] node [above, sloped, near end, inner sep = 2pt] {$\rx$} (a21'gov)
				;
				
				\path[every edge/.append style={very thick, ->}, every node/.append style={near end}]
				(a.center) edge [->] node [left, inner sep = 2pt] {$\vrole$} (a01)
				(a.center) edge [->] node [below, inner sep = 2pt] {$\hrole$} (a10)
				(a01.center) edge [->] node [below, inner sep = 2pt] {$\hrole$} (a11)
				(a10.center) edge [->] node [left, inner sep = 2pt] {$\vrole$} (a11)
				(a10.center) edge [->] node [below, inner sep = 2pt] {$\hrole$} (a20)
				(a20.center) edge [->] node [left, inner sep = 2pt] {$\vrole$} (a21)
				(a11.center) edge [->] node [below, inner sep = 2pt] {$\hrole$} (a21')
				;
				
				\node at (1,-.5) [inner sep=3pt](asw) {$a_\cSW$};
				\node [below=.1cm of asw] {$\cSW$};
				\node at (3,-.5) [inner sep=3pt] (as) {$a_\cS$};
				\node [below=.1cm of as] {$\cS$};
				\node at (-.5, 1)[inner sep=3pt] (aw) {$a_\cW$};
				\node [below=.1cm of aw] {$\cW$};
				\node at (4.6,.5)[inner sep=3pt] (ac) {$a_\cC$};
				\node [below=.1cm of ac] {$\cC$};

				\path[every edge/.append style={dashed, ->}]
				(a.center) edge [out=-90]  (asw)
				(a10.center) edge [out=-90]  (as)
				(a20.center) edge [out=-90, in=30]  (as)
				(a01.center) edge [out=180, in=90] (aw)
				(a11.center) edge [out=-40]  (ac)
				(a21.center) edge [out=20, in=80]  (ac)
				(a21'.center) edge [out=10, in=80] (ac)
				;
				
				\node at (0, 0) (a) [fill=white, inner sep=3pt] {$a$};
				
			\end{tikzpicture}
			\caption{
				A part of the South-West corner of a model.
				Dashed arrows represent the role $\rpos$.
				Instances of the concept $\node$ are $a$ and the $\bullet$-anonymous elements.
				This small portion fails to collapse as a proper $3 \times 2$ grid.
				Instances of $\cpx$ can thus follow the defective $\vrole$-$\hrole$-$\hrole$-path to place concept $\cX$ in a way that avoids the flooding.
			}
			\label{figure:failed-grid}
		\end{figure}
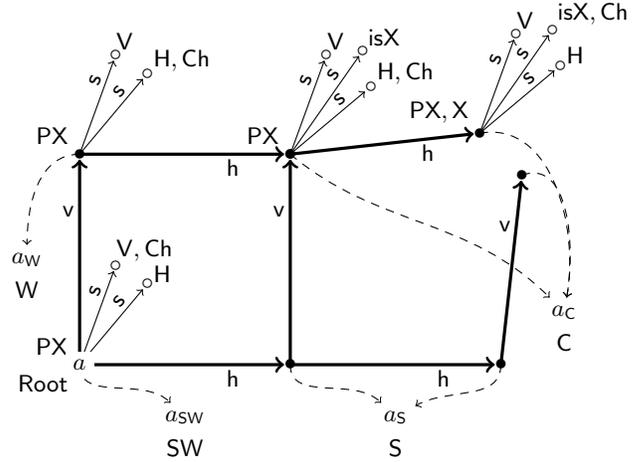
		
		We now trigger the flooding as previously described, with the axiom: $\exists \rh .\exists \rv .\cX \sqcap \exists \rv .\exists \rh .\cX \sqsubseteq \cflood$.
		Figure~\ref{figure:failed-grid} illustrates how $\cX$ can be placed to avoid triggering this latter rule if the model does not close as a grid.
		Now, propagating $\cflood$ back to $a$ is easily performed by the axioms: 
		\begin{align*}
			\exists \rh.\cflood  & \sqsubseteq \cflood &
			\exists \rv.\cflood  & \sqsubseteq \cflood.
		\end{align*}
		This propagates the flooding predicate to nodes along the $\hrole$-$\vrole$-path back to $a$, but not to \emph{all} nodes along \emph{all} $\hrole$-$\vrole$-paths.
		To achieve this latter part, recall that Point~3 of Claim~\ref{claim:positions} guarantees that, wherever we are on a path, there is always a further node with position $\cNE$, \emph{i.e.} a node that is connected to $a_{\cNE}$ by role $\rpos$.
		We can thus ensure that $\cflood$ propagates to every node by adding the following assertion and axiom:
		\[
			\raux(a_{\cNE}, a) \qquad
			\exists \rpos. (\exists \raux. \cflood)  \sqsubseteq \cflood.
		\]
		As announced, we now require $\cflood$ to ``flood'' the model by making all the $\cX$-related concepts and choices satisfied:
		\hspace*{-.2ex}\begin{align*}
			\cflood & \sqsubseteq \cX \sqcap \cpx
			\\
			\exists \rpos.p \sqcap \cflood & \sqsubseteq \exists \rx . (\cx \sqcap \chosen)
			& & \hspace*{-1.5ex}  \textrm{ for } p \in \{ \cC, \cN, \cE \}
			\\
			\exists \rpos.p \sqcap \cflood & \sqsubseteq \exists \rx . (\cgoh \sqcap \chosen)
			& & \hspace*{-1.5ex} \textrm{ for } p \in \{ \cC, \cN, \cS, \cW, \cSW \}
			\\
			\exists \rpos.p \sqcap \cflood & \sqsubseteq \exists \rx . (\cgov \sqcap \chosen)
			& & \hspace*{-1.5ex} \textrm{ for } p \in \{ \cC, \cE, \cS, \cW, \cSW \}
		\end{align*}
		We conclude the construction of the KB $\kb$ by adding the axiom $\cflood \sqcap \croot \sqsubseteq \goal$.
		Before proceeding to prove the reduction, it is useful to state one last property that summarizes some basic properties of minimal models of $\kb$ and highlights the effect of the latter flooding axioms:
		\begin{claimrep}
			\label{claim:flooding}
			Let $\Imc$ be a minimal model of $\kb$.
			Then concept names from $\{ \croot \} \cup \positions \cup \tiles$, as well as role names $\rspy$ and $\raux$ are interpreted as specified in the ABox.
			Furthermore, if $a^\Imc \in \goal^\Imc$, then $\subgoal_1^\Imc = \subgoal_2^\Imc = \goal^\Imc = \{ a^\Imc \}$ and $\node^\Imc = \cflood^\Imc$.
		\end{claimrep}
		\begin{proof}
			It is immediate that $\croot^\Imc = \{ a^\Imc \} \cup \positions \cup \tiles$, as otherwise one could simply drop any other element occurring in the interpretation of these concept names and obtain a model $\Jmc \subset \Imc$ contradicting the minimality of $\Imc$.
			The argument is similar for role names $\rspy$ and $\raux$.
			
			Assume now that $a^\Imc \in \goal^\Imc$.
			The facts that $a^\Imc \in \subgoal_1^\Imc \cap \subgoal_2^\Imc$ is trivial as otherwise one could simply drop $a^\Imc$ from $\goal^\Imc$ set the interpretation of $\cX$, $\cpx$, $\cx$, $\cgoh$, $\cgov$ and $\rx$ to the empty set to obtain a model simpler that $\Imc$.
			That $\subgoal_1^\Imc$ and $\subgoal_2^\Imc$ do not contain other elements than $a^\Imc$ is also immediate from $\croot^\Imc = \{ a^\Imc \}$.
			
			Now, from $a^\Imc \in \subgoal_1^\Imc$, we can apply Claim~\ref{claim:positions}.
			Point~3 in Claim~\ref{claim:positions} joint with $\cNE^\Imc = \{c_\cNE\}$, yields that every $\hrole$-$\vrole$-paths of nodes from $a$ can be extended to reach an element $e$ such that $(e, c_\cNE) \in \rpos^\Imc$ $(\dagger)$.
			From $a^\Imc \in \goal^\Imc$, it follows that $a^\Imc \in \cflood$, as otherwise we could simply remove $a$ from $\goal^\Imc$ while preserving modelhood.
			Joint with $(\dagger)$, this guarantees that $\cflood$ holds on every element along an $\hrole$-$\vrole$-path of nodes from $a$.
			By Claim~\ref{claim:h-v-paths}, this yields $\node^\Imc \subseteq \cflood^\Imc$ as desired.

		\end{proof}

		We can now 
        prove that the reduction is correct.
		\begin{claimrep}
			\label{claim:reduction}
			$(T, b) \in \recttile$ iff there exists a minimal model of $\kb$ that satisfies $\goal$.
		\end{claimrep}
		
		\begin{proofsketch}
			
			The $(\Rightarrow)$ direction is the easier one.
			From the tile assignment $\sigma$ of a finite rectangle, one can define a corresponding model $\Imc$ that satisfies concept $\goal$ and features an actual grid, with in particular $\node^\Imc = \cflood^\Imc$.
			The minimality of $\Imc$, and notably the necessity of $\node^\Imc = \cflood^\Imc$, is then established by arguing that the flooding cannot be avoided due to $\Imc$ encoding a proper grid.
			Thus, wherever $\cX$ is placed, it always triggers the flooding.
			
			The $(\Leftarrow)$ direction is more involved.
			From a minimal model $\Imc$ that satisfies $\goal$, we identify in $\Imc$ a valid tile assignment on a finite rectangle.
			To do so, we say that an element $d \in \node^\Imc$ has coordinate $(i, j)$ if there exists a $\hrole$-$\vrole$-path $d_0, r_1, d_1 \dots, r_n, d_n$ of nodes from $a$ with $d_n = d$ and such that there are exactly $i-1$ occurrences of $\hrole$ and $j-1$ of $\vrole$ in $r_1, \dots, r_n$.
			Note that, by Claim~\ref{claim:h-v-paths}, each $d \in \node^\Imc$ has at least one such coordinate.
			We then check that if two elements in $\Imc$ share a coordinate, then they are equal, which is the key of our construction.  
			This is argued by contradiction: assume there are two $\hrole$-$\vrole$-paths $d_0, r_1, d_1 \dots, r_n, d_n$ and $e_0, s_1, e_1 \dots, s_n, e_n$ of nodes from $a$, that both take the same numbers $i$ of $\hrole$-edges and $j$ of $\vrole$-edges, but $d_n \neq e_n$.
			We prove there is a $1 \leq k \leq n$ such that, placing the one necessary instance of $\cX$ at $d_k$, we obtain a model $\Jmc$ of $\kb$ that avoids flooding, thus contradicting the minimality of $\Imc$.
		\end{proofsketch}

		\begin{proof}
		$(\Rightarrow)$. 
		Assume that $(T, b) \in \recttile$, that is there exists a rectangle of horizontal length $M$ and vertical length $N$ and a tile assignment:
			\[
				\sigma : \{ 1, \dots, M \} \times  \{ 1, \dots, N \} \rightarrow \tiles
			\]
		such that:
		\begin{itemize}
			\item for every $1 \leq m \leq M-1$ and every $1 \leq n \leq N$, we have $(\sigma(m, n), \sigma(m+1, n)) \in H$ and;
			\item for every $1 \leq m \leq M$ and every $1 \leq n \leq N-1$, we have $(\sigma(m, n), \sigma(m, n+1)) \in V$.
		\end{itemize}
		We denote $G := \{ 1, \dots, M \} \times  \{ 1, \dots, N \}$ the `grid'.
		We construct a minimal model $\Imc$ of $\kb$ such that $\Imc \models \goal(a)$.
		The domain of $\Imc$ is: 
		\[
			\Delta^\Imc :=  G
			 \cup \{ a_p \mid p \in \positions \} 
			\cup \{ a_t \mid t \in \tiles \} 
			\cup \{ c, c_{\cx}, c_{\cgoh}, c_\cgov \}
		\]
		$\Imc$ interprets individuals in the trivial way, apart from $a$ that is interpreted as $(1, 1)$ for convenience.
		Concept names from $\positions \cup \{ \tile_t \mid t \in \tiles \} \cup \{ \croot \}$, as well as role names $\rspy$ and $\raux$ are interpreted as specified in the ABox.
		Goal-related concept names are interpreted as:
		\begin{align*}
			\subgoal_1^\Imc = \subgoal_2^\Imc = \goal^\Imc = \{ a^\Imc \} = \{ (1, 1) \}.
		\end{align*}
		Choice-related concepts are interpreted as follows:
		\[
			\cx^\Imc := \{ c_{\cx} \} \quad \cgoh^\Imc := \{ c_{\cgoh} \} \quad \cgov^\Imc := \{ c_{\cgov} \}  \quad \error^\Imc = \emptyset.
		\]
		All remaining concept names are interpreted as $G$.
		The remaining role names are interpreted as follows:
		\begin{align*}
			\hrole^\Imc := ~ &
			\{ ((i, j), (i+1, j)) \mid 1 \leq i \leq M-1, 1 \leq j \leq N \}
			\\
			&
			\cup \{ (a, c) \}
			\\
			\vrole^\Imc := ~ &
			\{ ((i, j), (i, j+1)) \mid 1 \leq i \leq M, 1 \leq j \leq N-1 \}
			\\
			&
			\cup \{ (a, c) \}
			\\
			\rpos^\Imc := ~ &
			\{ ((1, 1), a_\cSW), ((M, 1), a_\cSE) \}
			\\
			&
			\cup
			\{ ((1, N), a_\cNW), ((M, N), a_\cNE)\}
			\\
			&
			\cup 
			\{ ((i, 1), a_\cS) \mid 2 \leq i \leq M-1 \}
			\\
			&
			\cup 
			\{ ((1, j), a_\cW) \mid 2 \leq j \leq N-1 \}
			\\
			&
			\cup 
			\{ ((i, j), a_\cC) \mid 2 \leq i \leq M-1, 2 \leq j \leq N-1 \}
			\\
			&
			\cup 
			\{ ((i, N), a_\cN) \mid 2 \leq i \leq M-1 \}
			\\
			&
			\cup 
			\{ ((M, j), a_\cE) \mid 2 \leq j \leq N-1 \}
			\\
			\hastile^\Imc := ~ &
			\{ (d, a_t) \mid d \in G, \sigma(d) = t \}
			\\
			\rx^\Imc := ~ &
			\{ ((i, j), c_\cx) \mid 2 \leq i \leq M, 2 \leq j \leq N, i + j \neq N + M \}
			\\
			&
			\cup
			\{ ((i, j), c_\cgoh) \mid 1 \leq i \leq M-1, 1 \leq j \leq N \}
			\\
			&
			\cup
			\{ ((i, j), c_\cgov) \mid 1 \leq i \leq M, 1 \leq j \leq N-1 \}
		\end{align*}
		It is now readily verified that $\Imc$ is minimal.
		In particular, notice that the interpretations of $\hrole$ and $\vrole$ enforce a proper grid.
		Therefore, the propagation of $\cpx$ from $(1, 1)$ along a path in that grid will always result in concept $\cX$ being placed at some element $(i+1, j+1) \in G$ for $1 \leq i \leq M-1$ and $1 \leq j \leq N-1$.
		The concept $\cflood$ is then necessary at element $(i, j) \in G$, which further ensures that $\cflood$, $\cX$, $\cpx$ being interpreted as the whole $G$ is necessary.
		
		$(\Leftarrow)$.
		Assume that there exists a minimal model $\Imc$ such that $\goal^\Imc \neq \emptyset$.
		We would like to identify in $\Imc$ a finite grid along with a valid tile assignment.
		To do so, we say that an element $d \in \node^\Imc$ has coordinate $(i, j)$ if there exists a $\hrole$-$\vrole$-path $d_0, r_1, d_1 \dots, r_n, d_n$ of nodes from $a$ with $d_n = d$ and such that there are exactly $i-1$ occurrences of $\hrole$ and $j-1$ of $\vrole$ in $r_1, \dots, r_n$.
		Note that, by Claim~\ref{claim:h-v-paths}, each $d \in \node^\Imc$ has at least one such coordinate.
		
		We now prove that if two elements in $\Imc$ share a coordinate, then there are equal.
		Assume by contradiction that we could find two $\hrole$-$\vrole$-paths $d_0, r_1, d_1 \dots, r_n, d_n$ and $e_0, s_1, e_1 \dots, s_n, e_n$ of nodes from $a$, that both take the same number $i$ of $\hrole$-edges and $j$ of $\vrole$-edges, but $d_n \neq e_n$.
		It is clear that the only element with coordinate $(1, 1)$ is $a^\Imc$, and from Claim~\ref{claim:h-v-paths}, that there exists only one element with coordinate $(2, 1)$, and only one element with coordinate $(1, 2)$.
		Notice that in the remaining cases, we have $n = i+j-2 \geq 2$.
		W.l.o.g., assume that $i$ is minimal for this property, so that each coordinate $(i', j')$, with $i' < i$, refers to a unique node.
		Similarly, assume that $j$ is minimal w.r.t.\ $i$, so that each coordinate $(i, j')$, with $j' < j$, refers to a unique node.
		From these two minimality properties, it is now clear that $d_{n-2} = e_{n-2}$.
		Note that if $r_{n-1} = s_{n_1}$, then $r_{n} = s_n$ and Claim~\ref{claim:h-v-paths}, applied twice, ensures $d_n = e_n$ as desired.
		Therefore, $r_{n-1} = s_n$ and $r_n = s_{n-1}$ (guaranteed from $d_n$ and $e_n$ having the same coordinate $(i, j)$!).
		We now describe how to obtain an interpretation $\Jmc \subset \Imc$ that will contradict the minimality of $\Imc$.
		From Claim~\ref{claim:flooding}, we have $\node^\Imc = \cflood^\Imc$.
		In particular, each $d \in \node^\Imc$ has \emph{all} the $\rx$-successors corresponding to its position $pos(d)$, and all carrying the concept $\chosen$.
		If $d$ has a $\rx$ successor to an instance of $\cx \sqcap \chosen$, we denote this element $d\cdot\cx$; we denote similarly $d\cdot\cgoh$ and $d\cdot\cgov$ the other possible $\rx$-successors (whose existence, again, depends on $pos(d)$).
		To obtain $\Jmc$, start from $\Imc$ but replace the interpretation of concept names $\cX$, $\cpx$, $\cgoh$, $\cgov$, $\cflood$, $\goal$ and of role name $\rx$ by:
		\begin{align*}
			\cX^\Jmc := ~ & 
			\{ d_n \}
			\\
			\cx^\Jmc := ~ &
			\{d_n \cdot \cx \mid d_n \cdot \cx \text{ is defined} \}
			\\
			\cpx^\Jmc := ~ & 
			\{ d_0, \cdots, d_n \}
			\\
			\cgoh^\Jmc := ~ &
			\{ d_k \cdot \cgoh \mid 0 \leq k < n, r_{k+1} = \hrole, d_k \cdot \cgoh \text{ is defined} \}
			\\
			\cgov^\Jmc := ~ &
			\{ d_k \cdot \cgov \mid 0 \leq k < n, r_{k+1} = \vrole, d_k \cdot \cgov \text{ is defined} \}
			\\
			\cflood^\Jmc := & \emptyset
			\\
			\goal^\Jmc := & \emptyset
			\\
			\rx^\Jmc := ~ & 
			\{ (d_n, d_n \cdot \cx) \mid d_n \cdot \cx \text{ is defined} \}
			\\
			& \cup 
			\{ (d_k, d_k \cdot \cgoh) \mid 0 \leq k < n, r_{k+1} = \hrole, d_k \cdot \cgoh \text{ is defined} \}
			\\
			& \cup
			\{ (d_k, d_k \cdot \cgov) \mid 0 \leq k < n, r_{k+1} = \vrole, d_k \cdot \cgov \text{ is defined} \}.
		\end{align*}
		It can now be verified that the resulting $\Jmc$ is indeed a model of $\kb$.
		In particular, $d_n \neq e_n$ guarantees that $\cflood$ being empty does not contradicts the axiom $\exists \rh .\exists \rv .\cX \sqcap \exists \rv .\exists \rh .\cX \sqsubseteq \cflood$: the concept $\cflood$ does not need to hold at element $d_{n-2}$ ($= e_{n-2}$).
		Furthermore, Claim~\ref{claim:choices-are-disjoint} ensures that interpretations of concepts $\cX$ and $\cpx$ are complete.
		It is clear that $\Jmc \subset \Imc$, which provides the desired contradiction.
		
		We thus obtained that every coordinate is realized at most once in $\Imc$.
		Using Claim~\ref{claim:tiles} (recall $a^\Imc \in \subgoal_2^\Imc$), this guarantees that the mapping
		$\sigma(i, j) := tile(d)$ for $d \in \node^\Imc$ is well-defined.
		Points~1, 2 and 3 of Claim~\ref{claim:tiles} also ensure that the placement of Wang tiles is valid.
		It remains to verify that $\node^\Imc$ realizes the coordinates of a \emph{finite rectangle}.
		This is obtained from Claim~\ref{claim:positions}: Points~3 and 4 guarantee that the number of realized coordinates is finite.
		Set $M$ as the maximal horizontal coordinate and $N$ as the maximal vertical one.
		Points~1 and 2 joined with each coordinate being realized at most once guarantees that every coordinate $(i, j)$ for $1 \leq i \leq M$, $1 \leq j \leq N$ is realized at least (thus exactly) once.
	\end{proof}

	\section{Acyclicity to the Rescue} 
    \label{decidability}
 
    {In the light of Theorem~\ref{theorem:el-is-undecidable-without-top}, is there any hope for minimal model reasoning in DLs? 
    In our search for positive results, we turn for inspiration to \emph{pointwise circumscription}, where decidability results have been obtained for rather expressive fragments of $\ALCIO$ known to be undecidable in classical circumscription, notably including cases where roles are minimized. Pointwise circumscription coincides with standard circumscription for large classes of acyclic TBoxes, suggesting that terminological cycles may play a key role in the infeasibility of minimization. This is also supported by the heavy use of cyclic inclusions in our undecidability proof. 
    We thus turn our attention to acyclicity notions, and find that the excursion is fruitful: minimal model 
    reasoning becomes much more manageable for \emph{acyclic TBoxes.}} 

\subsection{Strong Acyclicity}         
   Following \cite{DistefanoS24}, for an $\mathcal{ALCIO}$ concept $C$ in negation normal form ($\mathrm{NNF}$), we 
   define the sets $Occ^+(C)$ and $Occ^-(C)$  of 
   predicates that occur in $C$ \emph{positively} and    \emph{negatively}, respectively: 
  \begin{align*}
    	 \mathit{Occ}^{+}(A) = {} & \mathit{Occ}^{-}(\neg A)= \{A\}  & \text{ with } A\in \cnames  \\ 
    	 \mathit{Occ}^{+}(\neg A)= {} & \mathit{Occ}^{-}(A)= \emptyset & \text{ with } A\in \cnames \\  
    	 \mathit{Occ}^{+}(\{a\})= {} & \{\{a\}\} \quad {\mathit{Occ^-}(\{a\})= \emptyset}  & \text{ with } a \in \inames \\ 
   	  \mathit{Occ}^{\pm}(C\,{\circ}\,D) = {} & \mathit{Occ}^{\pm}(C)\,{\cup}\,\mathit{Occ}^{\pm}(D) &\circ\in\{\sqcup, \sqcap\} \\ 
    \mathit{Occ}^{+}(\exists r.C)= {} & \{r\}\cup \mathit{Occ}^{+}(C) \\ 
    \mathit{Occ}^{-}(\exists r.C)= {} &  \mathit{Occ}^{-}(C) \\ 
      \mathit{Occ}^{+}(\forall r.C)= {} &  \mathit{Occ}^{+}(C) \\ 
    	  \mathit{Occ}^{-}(\forall r.C)= {} &  \{r\}\cup \mathit{Occ}^{-}(C)   
  \end{align*}
    {We let $\mathit{Occ}^{\pm}(C \sqsubseteq D) = \mathit{Occ}^{\pm}( \mathit{NNF}(\lnot C \sqcup D))$ for a concept inclusion $C \ISA D$ and $\pm \in \{+,-\}$.}
    The \emph{dependency graph} $\mathit{DG}(\Tmc)$ of an $\mathcal{ALCIO}$ TBox $\Tmc$ has as nodes all the
    concept names and role 
    names and all concepts of the form $\{a\}$ {or $\top$} 
    that appear in $\Tmc$, and there is an edge
    from  $P_1$ to $P_2$ if 
    $P_1\in \mathit{Occ}^{-}(\alpha)$ and $P_2\in \mathit{Occ}^{+}(\alpha)$ for some $\alpha \in \T$.
    We say that $\Tmc$ is \emph{strongly acyclic} if $\mathit{DG}(\Tmc)$ is acyclic {and no node is reachable from $\top$.}

 {This notion can be seen as a generalization of the one usually considered for terminologies \cite{DBLP:conf/dlog/BaaderN03} which is satisfied, for example, by the well-known medical terminology \textsc{SnomedCT}. Under the classical semantics, terminologies often rely on concept-equivalences and are therefore cyclic TBoxes. If we take the definitions in a terminology $A \doteq C$ as inclusions $C \ISA A$, we may regain acyclicity (and may enjoy lower complexity).
 Example \ref{snomedct} illustrates that, under the minimal model semantics, keeping only one of the two directions might be innocuous since the other inclusion is enforced by predicate minimization.
} 
  \begin{example}\label{snomedct}
A patient has been diagnosed with pneumoconiosis, which can be caused by  
various organic dust types; some types are very serious. 
\emph{Baritosis} is caused by barium dust, as stated in the following 
Snomed CT definition:\footnote{Based on \url{https://pmc.ncbi.nlm.nih.gov/articles/PMC4422531/}} 
\[ \mathsf{Pneumoconiosis} \sqcap \exists caus\_agent. \mathsf{Barium\_Dust} \sqsubseteq \mathsf{Baritosis} \] 
        Under the classical semantics, there are models where the patient is diagnosed with   baritosis due to a causative agent, barium dust, that is not justified in a real finding but simply made true to cause the baritosis diagnosis. In the minimal models, in contrast, baritosis can only be diagnosed on the basis of justified clinical findings.
    \end{example}

Standard circumscription has been studied for 
acyclic terminologies \cite{BonattiLW09}, but unlike in our setting,  acyclicity does not help reducing the complexity. Reasoning about general $\mathcal{ALCIO}$ KBs can be reduced to acyclic ones, but the reduction uses \emph{varying} predicates. 
   
Despite their restrictions, acyclic TBoxes are quite expressive, and strongly acyclic $\mathcal{EL}$ 
can force minimal models satisfying a concept of interest to have an exponential size, as illustrated by Example~\ref{binarytree}.
We later show that this exponential size is also sufficient, in the sense that every concept satisfiable in a minimal model also is in one with exponential size, and this even for the relaxed notion of weak acyclicity.

\begin{example}\label{binarytree}
To generate a binary tree with $2^n$ leaves, consider the assertion $\cn{L}_0(a)$ and axioms $\cn{L}_i \sqsubseteq \cn{\exists r}_i. \cn{L}_{i+1} \sqcap \exists \cn{l}_i. \cn{L}_{i+1}$ for all $0 \leq i < n$.
We want to enforce all leaves to be different objects in minimal models that satisfy a concept of interest.
For this, we add axioms that attempt to produce a second tree \emph{starting from its leaves}.
The latter are identified by concept $\cn{L}'_0$, which is made available at leaves of the first tree via $\cn{L}_{n}\sqsubseteq \cn{L}_{0}'$.
Further levels of the second tree, towards its root, are generated with the two assertions $\cn{Left}(o)$ and $\cn{Right}(o')$ and the following axioms for $0 \leq  j < n$:
\begin{align*}
	\cn{L}_j' & \sqsubseteq \exists \cn{pick}. \top
	&
	\cn{L}_j'\sqcap \exists \cn{pick}. \cn{Left} & \sqsubseteq \exists \cn{l}'_j. \cn{L}_{j+1,l}' 
	\\
	\cn{L}_{j+1,l}' \sqcap \cn{L}_{j+1,r}' & \sqsubseteq \cn{L}_{j+1}'
	&
	\cn{L}_j'\sqcap \exists \cn{pick}. \cn{Right} & \sqsubseteq \exists \cn{r}'_j. \cn{L}_{j+1,r}'
\end{align*}
A minimal model of this strongly acyclic $\EL$ KB can only satisfy the concept $\cn{L}_n'$ (representing the root of the second tree) if its interpretation of the first tree produces at least $2^n$ instances of $\cn{L}_{n}$, i.e. of $\cn{L}_0'$.
\end{example}

Before moving to the complexity results, we provide another illustrative example:
 in strongly acyclic $\EL$ 
we can assign truth values to objects and compare these assignments. This trick will be useful in the hardness proofs below. 

\begin{example}\label{silly}
Consider the ABox with assertions $\cn{N_1}(a)$, $\cn{N_2}(b)$, $\cn{T}(v_1)$, $\cn{F}(v_2)$, and the strongly acyclic $\EL$ TBox with the following axioms, for $i \in \{ 1, 2\}$ and $C \in \{ \cn{T}, \cn{F} \}$:
\[
\begin{array}{cc}
	\cn{N}_i \sqsubseteq \exists \cn{val}. \cn{TV} & \cn{N_1} \sqsubseteq \exists \cn{read}.\top 
\\[3pt]
	\exists \cn{val}. C \sqsubseteq C' 
	& 
	C'\sqcap \exists \cn{read}. (\cn{N_2}\sqcap C') \sqsubseteq \cn{Goal}

\end{array}
\]
The nodes $a$ and $b$ each pick a truth value via role $\cn{val}$, which is copied to the node as $\cn{T'}$ (resp.\ $\cn{F'}$) \emph{only if} it is $\cn{T}$ (resp.\ $\cn{F}$). 
The node $a$ then reads the value at some object via role $\cn{read}$, and the $\cn{Goal}$ concept is satisfied exactly when $a$ reads the value at $b$ (the only instance of $\cn{N}_2$ in a minimal model is $b$) and they both picked the same truth value. 
\end{example}

\paragraph{Strongly Acyclic $\EL$ and Pointwise Circumscription.}  
  To show the decidability of strongly acyclic $\mathcal{ELIO}_\bot$ we rely on results on \emph{pointwise circumscription} \cite{DistefanoOS23}, a local approximation of standard circumscription where minimization is allowed only locally, at one domain element. 
  In our definition of minimal models, 
  predicates 
  are minimized \emph{globally},  across the entire interpretation. 
In pointwise circumscription,
we refine  the relation $\subseteq$ by only considering \emph{pointwise comparable} interpretations.

\begin{definition}[Pointwise Comparison]\label{pwc:comparison}
	Given interpretations $\Imc$ and $\Jmc$, we write $\Imc \sim^\bullet \Jmc $ if there exists $e \in \Delta^{\Imc}$ such that:
	\begin{enumerate}[(i), align= left]
		\item for all $A\in \cnames$, $A^{\Imc} \cap \Delta = A^{\Jmc} \cap \Delta$, and 
		\item for all $r\in \rnames$, $r^{\Imc} \cap (\Delta\times \Delta) = r^{\Jmc} \cap (\Delta\times \Delta)$,
	\end{enumerate}
	where $\Delta = \Delta^{\Imc} \setminus \{e\}$.
\end{definition}

        \begin{definition}[Pointwise Minimal Model]
            Given a KB $\Kmc$, a model $\Imc$ is pointwise minimal if there exists no $\Jmc\subsetneq \Imc$ such that $\Jmc\models \Kmc$ and $\Jmc\sim^\bullet \Imc$.
        \end{definition}
        To emphasize the difference, we sometimes refer to minimal models in the sense of this paper as \emph{globally} minimal models.
 In general, the set of pointwise minimal models does not coincide with the set of globally minimal models.
 \begin{example}
 	\label{example:global-versus-pointwise}
 Let $\kb := (\{\exists \cn{r. A\sqsubseteq A}\}, \{\rstyle{r}(a,b), \rstyle{r}(b,a)\})$. The interpretation $\Imc$ such that $\Delta^{\Imc}=\{d,e\}$, $a^\I=d$ and $b^\I=e$, $\cn{A}^{\Imc}=\{d,e\}$ and $\cn{r}^{\Imc}=\{(d,e),(e,d)\}$ is a pointwise minimal model of our $\EL$ KB $\kb$ that is not globally minimal. 
 \end{example}
  
{Pointwise circumscription has  better computational properties than standard circumscription. 
The modal depth of a KB $\K$, denoted with $\mathrm{md}(\K)$, is defined as the maximal number of nested quantifiers occurring in $\K$.
When all roles are minimized, concept satisfiability  is complete for $\NExp$
in the fragment $\mathcal{ALCIO}^{d\leq 1}$ of $\mathcal{ALCIO}$ with modal depth one \cite{DistefanoOS23}. 
In contrast, the problem is undecidable for full $\ALCIO$.}

Standard normalization techniques \cite{bookdls} do not preserve minimal models of $\ALCIO$ KBs in general, but, they do for  
	$\mathcal{ELIO}_\bot$ KBs.
   \begin{propositionrep}
        	\label{normalization:mms}
            Any KB $\Kmc$ in $\mathcal{ELIO}_\bot$ can be transformed in polynomial time into an KB $\Kmc'$ in $\mathcal{ELIO}_\bot$ with $\mathrm{md}(\K')\leq 1$ such that $\Kmc'$ is a conservative extension of $\Kmc$ under the minimal model semantics and preserving strong acyclicity and weak acyclicity (the latter is defined in Section~\ref{section:weak-acyclicity}).
        \end{propositionrep}

         \begin{proof}
         The notion of modal depth can be lifted to complex concepts too in the trivial way. We denote the modal depth of a concept $C$ with $\mathrm{md}(C)$. 
         
        	Assume a KB $\Kmc=(\Amc, \Tmc)$ in $\mathcal{ELIO}_\bot$. W.l.o.g. we can assume that each $C\sqsubseteq D\in \Kmc$ is such that (1) $D=\exists r. B$ where $B$ is a $\mathcal{ELIO}$ concept or (2) $D=\bot$. Given an inclusion $C\sqsubseteq D$, we perform the following steps:
        	\begin{itemize}
        		\item if $\mathrm{md}(C)>1$, for each $\exists r. B\in conc(C)$ where $\mathrm{md}(B)=1$, we introduce a fresh concept name $C_{\exists r. B}$, replace each occurrence of $\exists r. B$ in $C$ with $C_{\exists r. B}$ and introduce the inclusion $\exists r. B\sqsubseteq C_{\exists r. B}$;
        		\item if $\mathrm{md}(D)>1$, i.e. if $\mathrm{md}(B)\geq 1$, we introduce a fresh concept name $D_{B}$, replace $D$ with $\exists r. D_{B}$  and introduce the inclusion $D_{B}\sqsubseteq B$.
        	\end{itemize}
        	We obtain a new inclusion $C'\sqsubseteq D'$ such that $\mathrm{md}(C')=\mathrm{md}(C)-1$ and $\mathrm{md}(D')=\mathrm{md}(\exists r. D_B)= 1$, together with a set of inclusions of the form $\exists r. B\sqsubseteq C_{\exists r. B}$ and $D_B\sqsubseteq {\exists r. B}$ where $\mathrm{md}(\exists r. B)=\mathrm{md}(D)-1$. Let $\Kmc'$ be the resulting KB. It is known that $\Kmc'$ is a conservative extension of $\Kmc$ under the classical semantics. 
        	
        	We show that the same holds under the minimal model semantics. Assume an interpretation $\Imc$ such that $\Imc\models \Kmc$ and $\Imc$ is minimal, we can extend it to a minimal model $\Imc'$ of $\Kmc'$ by simply extending the interpretation function of $\Imc$ to the fresh concept names introduced by the steps above. For each $C_{\exists r. B}$, we state that $C_{\exists r. B}^{\Imc'}=(\exists r. B)^{\Imc}$. Observe that it immediately follows that ${C'}^{\Imc'}=C^{\I}$. 
        	
To define a minimal extension for $D_{B}$, we need to be more careful.      	
Intuitively, when we assign elements to $D_{B}$, we look at domain elements that are satisfying $B$. Differently from $C_{\exists r. B}$, we cannot simply define $D_{B}^{\I'}$ as the set of all elements satisfying $B$. Indeed, each $e\in {C}^{\I}$ can provide a justification for only one occurrence of $D_B^{\I'}$. 

We place $D_B$ occurrences in $\I'$ by selecting carefully a subset of $B^{\I}$. Given a domain element $d\in \Delta^\Imc$ such that $d\in C^{\Imc}$, we consider the set $\nu_\Imc(d,\exists r. B)=\{d'\in \Delta^\Imc\vert (d,d')\in r^{\Imc}\text{ and }d'\in B^{\Imc}\}$. Let $U_\Imc(C, \exists r. B)\subseteq \bigcup_{d\in C^{\Imc}} \nu_\Imc(d,\exists r. B)$ such that (1) $U_\Imc(C,\exists r. B)\cap \nu_\Imc(d,\exists r. B) \not =\emptyset$, for each $d\in C^{\Imc}$ and (2) $U_{\Imc}(C,\exists r. B)$ is $\subseteq$-minimal. With (1) we ensure that each $d\in C^{\I}$ has at least one $\exists r. B$ witness in $U_{\Imc}(C,\exists r. B)$. With (2), we take the least amount of such $\exists r. B$ witnesses for each $d\in C^{\I}$ that still satisfy the requirement of (1).

We state that $D_{\exists r. B}^{\Imc'}=U_\Imc(C,\exists r. B$). We argue that $\Imc'$ is minimal. Observe that from the minimality of $\Imc$, for each $p\in sig(\K)$, $p^{\I'}$ is minimal. From the latter, also $C_{\exists r. B}^{\I'}$ is minimal. We are only left to argue that $D_B^{\I'}$ is minimal. Assume that there exists $\Jmc$ such that $\Jmc\models \K'$ and $\J\subset \I'$. From the above observations, then it must be that there exists $e\in D_B^{\I'}$ such that $e\not \in D_B^{\J}$. Since $e\in D_B^{\I'}$, then $e\in U_I(C,\exists r. B)$. Hence, there exists $d\in C^{\I}$ such that $(d,e)\in r^\I$ and $e\in B^{\I}$. Consider the set $U'=U_{\Imc}(C, \exists r. B)\setminus \{e\}$.
Trivially, we have that $U'\subseteq U$. It also easy to observe that $U'$ satisfies (1). Indeed, since $\J$ agrees with $\I'$ over the signature of $\K$ and $C_{\exists r. B}$, then ${C'}^{\J}={C'}^{\O}$. Since $\J$ is a model of $\K'$, for each $d\in {C'}^{\J}={C'}^{\I'}=C^{\I}$, we have that $d\in \exists r. D_B^{\J}$. Hence, for each $d\in {C'}^{\J}$ there exists $d'\not = e$ such that $(d,d')\in r^{\Jmc}$ and $d'\in D_B^{\J}$. From $\Jmc\subseteq \Imc'$, we have that $d'\in D_B^{\Imc}$, hence $d'\in U_\I(C,\exists r. B)$.
Since $d'\not =e$, for each $d\in C^{\I}$, we have that $U'\cap \nu_\I(d,\exists r.B)\not =\emptyset$. A contradiction. Therefore, $\Imc'$ is a minimal model of $\K'$.

Conversely, assume that $\Imc'$ is a minimal model of $\Kmc'$ and let $\Imc$ be the interpretation agreeing with $\Imc$ over the signature of $\Kmc$. We show that $\Imc$ is minimal. 
        	Assume that there exists $\Jmc\subset \Imc$ such that $\Jmc\models \Kmc$. Similarly to the above, we can construct a model $\Jmc'$ of $\Kmc'$ such that $\Jmc'\subset \Imc'$, deriving a contraction.
        	
        	To do so, it is sufficient to observe that if $d\in C^{\Jmc}$ then there exists $d'\in \Delta^{\Imc}$ such that $(d,d')\in r^{\Jmc}\subseteq r^{\Imc}$ and $d'\in B^{\Jmc}\subseteq B^{\Imc}$. We can pick $U_\Jmc(C, \exists r. B)$ such that $U_\Jmc(C, \exists r. B)\subseteq U_\Imc(C,\exists r. B)$. We can define an interpretation $\Jmc'$ such that $\Delta^{\Jmc'}=\Delta^{\Jmc}$, $p^{\Jmc'}=p^{\Jmc}$, for all $p\in sig(\K)$, $C_{\exists r. B}^{\Jmc'}=(\exists r. B)^{\Jmc}$ and $D_B^{\Jmc'}=U_{\Jmc}(C, \exists r. b)$. It is easy to observe that $\Jmc'\models \K'$. Since $\Jmc\subset \Imc$ and $U_{\Jmc}(C,\exists r.B)\subseteq U_{\Imc}(C,\exists r.B)$, it immediately follows that such that $\Jmc'\subset \Imc'$. A contradiction. 
        	
        	By iterating the steps above, for each inclusion $C\sqsubseteq D\in \Kmc$, we obtain a conservative extension $\Kmc'$ of $\Kmc$ under the minimal model semantics. Furthermore, $d(\Kmc')=d(\Kmc)-1$. We can further apply the normalization to $\Kmc'$ until we obtain a KB $\Kmc''$ such that $d(\Kmc'')=1$. Since being a conservative extension is a transitive relation, we have that $\Kmc''$ is a conservative extension of $\Kmc$. Thus, we obtain the thesis. To show that we preserve weak and strong acyclicity, it is sufficient to observe that the normalization rules do not introduce cyclic dependencies among the fresh concept names.
        \end{proof}

For strongly acyclic KBs in $\mathcal{ALCIO}^{d\leq 1}$, minimal models and pointwise minimal models coincide \cite{DistefanoS24}. We thus inherit the following result.
        \begin{theorem}\label{alciodone}
           $\mms$ in strongly acyclic $\mathcal{ALCIO}^{d\leq 1}$ is in $\NExp$. 
        \end{theorem}
        
The complexity result of Theorem \ref{alciodone} applies to strongly acyclic $\mathcal{ELIO}_\bot$ without restrictions on the modal depth. 
        \begin{theorem}
            $\mms$ in strongly acyclic $\ELIO_\bot$ is in $\NExp$.
        \end{theorem}
        \begin{proof}
        Let $\mathcal{ELIO}_\bot^{d\leq 1}$ be the fragment of $\mathcal{ELIO}_\bot$ where concept expressions have \emph{modal depth} at most one. Pointwise minimal satisfiability in $\mathcal{ELIO}_\bot^{d\leq 1}$ is in $\NExp$. In strongly acyclic $\ELIO_\bot$, the set of pointwise minimal models coincides with the set of minimal models. The claim then follows from Proposition \ref{normalization:mms}.
        \end{proof}
    	  \begin{theoremrep}\label{nexp:el:strong:ac}
            $\mms$ in strongly acyclic $\mathcal{EL}$ is $\NExp$-hard. 
        \end{theoremrep} 
        \begin{proofsketch}
        We provide a reduction from the \emph{torus tiling problem} \cite{Tobies99} to $\mms$. Using the construction in Example \ref{binarytree}, we construct a $\K$ and a goal concept $\cn{Goal}$ such that the satisfaction of $\cn{Goal}$ in a minimal model $\I$ ensures that: 
        \begin{enumerate*}[(1)]
        \item in $\I$ we can embed a tree of depth $2n$, where each leaf encodes (in binary) a pair of coordinates $(x,y)$, with $0\leq x,y\leq 2^n-1$;
        \item a torus is embedded in the leaves of the tree in a way such that the horizontal and vertical successors respect the tiling conditions. 
        \end{enumerate*}
        To achieve the latter desiderata, we construct a subgoal concept $\cn{G}_{2n}$ ensuring that each leaf encoding the pair $(x,y)$ has as horizontal successor the pair $(x+1,y)$ and as vertical successor the pair $(x,y+1)$. To check that $\cn{G}_{2n}$ is satisfied at all the leaves, we propagate a concept $\cn{LeafGrid}$ back to the root, using the following axioms:
        \begin{align*}
        \exists \cn{l}_{i}. \cn{G}_{i+1}\sqcap \exists \cn{r}_{i}. \cn{G}_{i+1} &\sqsubseteq \cn{G}_{i}\quad \text{ with }0\leq i < 2n\\
        \cn{G}_0\sqsubseteq \cn{LeafGrid}
        \end{align*}
        where the roles $\cn{l}_i$ and $\cn{r}_i$ are as in Example \ref{binarytree}.
         In a minimal model, $\cn{LeafGrid}$ is satisfied at the root of the tree if $\cn{G}_{2n}$ is satisfied at all the leaves.
        \end{proofsketch}
      \begin{toappendix} 
      \input{nexpHardnessProof}
      \end{toappendix}
        
    \subsection{Weak Acyclicity}
    \label{section:weak-acyclicity}
We now define \emph{weak acyclicity}, which is an important notion for TGDs in the database literature. We refine the above notion by annotating some edges in $\mathit{DG}(\Tmc)$ as $\star$-edges.
    For a concept $C$ in $\mathit{NNF}$, we define: 
  \begin{align*}
       \mathit{Occ}_\exists^{+}(A) = {} &   \mathit{Occ}_\exists^{+}(\neg A) = \emptyset & \text{ with } A \in \cnames \\
    	 \mathit{Occ}_\exists^{+}(\{o\})= {} &  \emptyset  & \text{ with } a \in \inames \\
   	  \mathit{Occ}_\exists^{+}(C\,{\circ}\,D) = {} &  \mathit{Occ}_\exists^{+}(C)\,{\cup}\,\mathit{Occ}_\exists^{+}(D) & \circ\in\{\sqcup, \sqcap\} \\
    	  \mathit{Occ}_\exists^{+}(\exists r.C)= {} & \{r\}\cup \mathit{Occ}^{+}(C) \\
   	  \mathit{Occ}_\exists^{+}(\forall r.C)= {} &  \mathit{Occ}_\exists^{+}(C) 
  \end{align*}
    We let $\mathit{Occ}_\exists^{+}(C \sqsubseteq D) = \mathit{Occ}_\exists^{+}( \mathit{NNF}(\lnot C \sqcup D))$ for a concept inclusion $C \ISA D$.
    In $\mathit{DG}(\Tmc)$, there is a $\star$-edge from $P_1$ to $P_2$ if for some $\alpha\in \Tmc$, $P_1\in \mathit{Occ}^{-}(\alpha)$ and $P_2\in \mathit{Occ}_\exists^{+}(\alpha)$.
    Notice that for all $C$, we have $\mathit{Occ}_\exists^{+}(C)\subseteq \mathit{Occ}^{+}(C)$, and thus all $\star$-edges are also basic edges.
    We call a TBox $\Tmc$ \emph{weakly acyclic} if there is no cycle in $\mathit{DG}(\Tmc)$ that goes through an $\star$-edge {and no node is reachable from $\top$ in $\mathit{DG}(\Tmc)$.} 
    Clearly, every strongly acyclic TBox is also weakly acyclic.
    The $\EL$ TBox in Example~\ref{example:global-versus-pointwise} is weakly acyclic but not strongly acyclic.

    Fortunately, even for weakly acyclic $\ELIO_\bot$, we can always find a model whose size is at most single exponential. 

	\begin{lemmarep}
			\label{lemma:exp-size-bound}
		Weakly acyclic $\ELIO_\bot$ has the small model property:
		if $C$ is satisfied in some minimal model of $\K = (\T,\A)$, then it is satisfied in a minimal model $\Jmc$ whose domain has size bounded by $\sizeof{\inames(\kb)} \times (\sizeof{\tbox} \times 2^{\sizeof{\tbox}})^\sizeof{\tbox}$.
	\end{lemmarep}
    
    \begin{proofsketch}
    	Let $\Imc$ be such a minimal model.
    	It suffices to prove that its active domain \emph{i.e.}\ the subset of elements from $\Delta^\Imc$ that occur in the interpretation of at least one concept or role name, has the claimed size.
    	Since no node is reachable from $\top$ in $\depgraph(\tbox)$, every single fact in $\Imc$ somehow stems from the individuals occurring in $\kb$.
    	We thus start from those and track which successors they might require based on their types, \emph{i.e.\ }the combinations of concepts they satisfy in $\Imc$.
    	For example, an element $a$ with type $\{ A \}$ requires an $r$-successor with type containing $B$ if $\tbox$ has an axiom $A \sqsubseteq \exists r.B$; since $\Imc$ is a model, there exists such a successor $e$.
    	Now, instead of directly iterating the above by looking at the type of $e$ in $\Imc$ (say, $\{ A, B \}$), we first restrict this type to the concepts that are reachable from $A$ via \emph{at least} one $\star$-edge of $\depgraph(\tbox)$ (for element $e$, we thus restrict its type $\{ A, B \}$ to $\{ B \}$).
    	Indeed, no such concept can further require $A$: it would form a cycle in $\depgraph(\tbox)$ going through the said $\star$-edge, contradicting $\tbox$ being weakly acyclic.
    	Therefore, the restricted types we successively consider become empty after at most $\sizeof{\depgraph(\tbox)}$ iterations.
    	In particular, the number of reached elements is bounded as desired; and the restriction of $\Imc$ to these elements is a model, which must be the active domain of $\Imc$ as $\Imc$ is minimal.
    \end{proofsketch}

\begin{proof}
	Consider a weakly acyclic $\ELIO_\bot$ KB $\kb = (\tbox, \abox)$.
	Using Proposition~\ref{normalization:mms}, we safely assume that every axiom has modal depth at most $1$.
	Let $\Imc$ be a minimal model of $\kb$.
	We prove that the active domain of $\Imc$, \emph{i.e.}\ the subset $\Delta \subseteq \Delta^\Imc$ of elements that occur in at least one $\cstyle{A}^\Imc$ or one $\rstyle{p}^\Imc$ for $\cstyle{A} \in \cnames$ and $p \in \rnames$, has size at most $\sizeof{\inames(\kb)} \times (\sizeof{\tbox} \times 2^{\sizeof{\tbox}})^\sizeof{\tbox}$.
	Therefore, restricting $\Imc$ to its active domain immediately yields an interpretation $\Jmc$ with the desired size and satisfying the very same concepts as $\Imc$.
	Note that this is a much stronger statement than needed as it does not depend at all on the considered reasoning task (that is, trying to satisfy the concept of interest).
	
	We proceed by contradiction.
	Assume that the active domain $\Delta$ of $\Imc$ is bigger than $\sizeof{\abox} \times 4^\sizeof{\tbox}$.
	We will construct an interpretation $\Imc' \subset \Imc$, which will contradicts the minimality of $\Imc$.
	
	Given two nodes $u, v$ in $\depgraph(\tbox)$, we denote $u \leadsto^\star v$ if there exists a path from $u$ to $v$ that uses at least one $\star$-edge.
	A type is a subset of vertices from $\depgraph(\tbox) := (\Vmc, \Emc)$.
	Given two types $t, s \subseteq \Vmc$, we denote $t \leadsto^\star s$ if for every $v \in s$, there exists $u \in t$ such that $u \leadsto^\star v$.
	The type $\type_\Imc(e)$ of an element $e \in \Delta^\Imc$ in the interpretation $\Imc$ is $\{ \cstyle{A} \in \cnames(\kb) \mid e \in \cstyle{A}^\Imc \} \cup \{ \rstyle{r} \in \rnames(\kb) \mid e \in (\exists \rstyle{r})^\Imc \cup (\exists \rstyle{r}^-)^\Imc \}$
	The type $\type^t_\Imc(e)$ of an element $e \in \Delta^\Imc$ in the interpretation $\Imc$ w.r.t.\ another type $t$ is the maximal (w.r.t.\ set inclusion) type $s$ such that $s \subseteq \type_\Imc(e)$ and $t \leadsto^\star s$.
	
	Given elements $e_0, \dots, e_n \in \Delta^\Imc$, with $n \geq 0$, possibly inverse roles $\rstyle{r}_1, \dots, \rstyle{r}_n$ and types $t_1, \dots, t_n \subseteq \Vmc$, we say that a sequence $e_0 (\rstyle{r}_1, t_1) e_1 \dots (\rstyle{r}_n, t_n) e_n$ is a \emph{dependency path} in $\Imc$ if it satisfies the following conditions:
	\begin{enumerate}
		\item $e_0 \in \inames(\kb)$;
		\item For every $1 \leq i \leq n$, $t_i = \type_\Imc^{t_{i-1}}(e_i)$, with $t_0 := \type_\Imc(e_0)$;
		\item For every $1 \leq i \leq n$, $\rstyle{r}_i \in t_i$ and $(e_{i-1}, e_i) \in \rstyle{r}_i^\Imc$.
	\end{enumerate}
	Notice that since $\tbox$ is weakly acyclic, all dependency paths have their length $n$ being at most $\sizeof{\Vmc}$ as otherwise the non-emptiness of the last type (Point~3) joined with Point~2 in the above would yield a path of $\star$-edges in $\depgraph(\tbox)$ with length $> \Vmc$, thus forming at least one cycle.
	
	Now, for every $d \in \Delta^\Imc$, $\rstyle{r} \in \rnamespm$ and $t \subseteq \Vmc$, if there exists a dependency path ending with $d (\rstyle{r}, t) e$ for some $e \in \Delta^\Imc$, then we chose such an $e$ and denote $\successor_{\rstyle{r}, t}(d)$ this element.
	We now define $W$ as the set of all dependency paths $e_0 (\rstyle{r}_1, t_1) e_1 \dots (\rstyle{r}_n, t_n) e_n$ such that for all $1 \leq i \leq n$, we have $e_i = \successor_{\rstyle{r}_i, t_i}(e_{i-1})$.
	Notice that $\sizeof{W} \leq \sizeof{\inames(\kb)} \times (\sizeof{\tbox} \times 2^{\sizeof{\tbox}})^\sizeof{\tbox}$.
	We define $\Imc'$ as the interpretation with domain $\Delta^{\Imc'} := \Delta^\Imc$ and that interprets every concept name $\cstyle{A}$ and role name $\rstyle{r}$ as follows:
	\begin{align*}
		\cstyle{A}^{\Imc'} := ~ & \{ e \mid e \in \inames(\kb) \cap \cstyle{A}^\Imc \}
		\cup \{ e \mid \text{there exists } w (\rstyle{r}, t) e \in W \text{ with } \cstyle{A} \in t \}
		\\
		\rstyle{r}^{\Imc'} := ~ & \{ (d, e) \mid \rstyle{r}(d, e) \in \abox \}
		 \cup \{ (d, e) \mid \text{there exists } w d (\rstyle{r}, t) e \in W \}
		 \cup \{ (e, d) \mid \text{there exists } w d (\rstyle{r}^-, t) e \in W \}.
	\end{align*}
	It is transparent that $\Imc' \subseteq \Imc$ and that the active domain of $\Imc'$ has size at most $\sizeof{W}$, thus strictly smaller than the one of $\Imc$, hence $\Imc' \subset \Imc$.
	To contradict the minimality of $\Imc$, it remains to verify that $\Imc'$ is indeed a model of $\kb$.
	
	It is immediate that $\Imc' \models \abox$.
	For the TBox $\tbox$, we recall that the TBox has modal depth at most $1$ and we treat below the most interesting cases of CIs.
	Importantly, since $\tbox$ is weakly acyclic, no node is reachable from $\top$ in $\depgraph(\tbox)$, and in particular there are no CI in $\tbox$ with shape $\top \sqsubseteq \cstyle{C}$.
	
	Case of $\exists \rstyle{r}.\cstyle{A} \sqsubseteq \cstyle{B}$.
	Notice it guarantees that there is an edge from $\rstyle{r}$ to $\cstyle{B}$ in $\depgraph(\tbox)$.
	Assume we have $(d, e) \in \rstyle{r}^{\Imc}$ with $e \in \cstyle{A}^{\Imc'}$.
	We need to prove $d \in \cstyle{B}^{\Imc'}$.
	We distinguish cases based on the definition of $\rstyle{r}^{\Imc'}$.
	If $\rstyle{r}(d, e) \in \abox$, then $e \in \inames$ and thus $e \in \cstyle{A}^{\Imc}$.
	Therefore $d \in (\exists \rstyle{r}.\cstyle{A})^\Imc$ thus $d \in \cstyle{B}^\Imc$ as $\Imc$ is a model.
	It follows that $d \in \cstyle{B}^{\Imc'}$.
	If there exists $w d (\rstyle{r}, t) e \in W$, then notice that by Point~3 from the definition of dependency paths, we have $\rstyle{r} \in t$.
	Furthermore, from $e \in \cstyle{A}^{\Imc'}$, we get $\cstyle{A} \in \type_\Imc(e)$, thus $\cstyle{B} \in \type_\Imc(d)$.
	Now if $d \in \inames(\kb)$, then immediately we get $d \in \cstyle{B}^{\Imc'}$, otherwise $w$ must end by some $(\rstyle{r_0}, t_0)$.
	It remains to prove that $\cstyle{B} \in t = \type_\Imc^{t_0}(d)$.
	By definition $\rstyle{r} \in t$ thus there exists $\cstyle{C}_0 \in t_0$ such that $\cstyle{C}_0 \leadsto^\star \rstyle{r}$.
	Recall that there is an edge from $\rstyle{r}$ to $\rstyle{B}$ in $\depgraph(\tbox)$, which guarantees $\cstyle{C}_0 \leadsto^\star \cstyle{B}$.
	The case of $w d (\rstyle{r}^-, t) e \in W$ is treated similarly.
	
	Case of $\cstyle{A} \sqsubseteq \exists \rstyle{r}.\cstyle{B}$.
	Notice it guarantees that there is a $\star$-edge from $\cstyle{A}$ to $\cstyle{B}$ in $\depgraph(\tbox)$.
	Assume we have $d \in \cstyle{A}^{\Imc'}$, we have two cases.
	If $d \in \inames(\kb)$, then $\cstyle{A} \in \type_\Imc(d)$ and from $\Imc$ being a model we obtain that there exists $(d, e) \in \rstyle{r}^\Imc$ with $e \in \cstyle{B}^\Imc$.
	Therefore $d(\rstyle{r}, t)e$ is a dependency path, where $t := \type_\Imc^{\type_\Imc(d)}(e)$.
	Notice that, from the $\star$-edge from $\cstyle{A}$ to $\cstyle{B}$ and $\cstyle{B} \in \type_\Imc(e)$, we obtain that $\cstyle{B} \in t$.
	This guarantees that $\successor_{\rstyle{r}, t}(d)$ is defined, and it follows that $(d, \successor_{\rstyle{r}, t}(d)) \in \rstyle{r}^{\Imc'}$ and $\successor_{\rstyle{r}, t}(d) \in \cstyle{B}^{\Imc'}$.
	It remains to treat the case of $w (\rstyle{r}_0, t_0) d \in W$ with $\cstyle{A} \in t_0$.
	From $d \in \cstyle{A}^{\Imc}$, we can argue again that $\successor_{\rstyle{r}, t}(d)$ is defined for some type $t$ that contains $\cstyle{B}$, which in turn concludes the argument.
	
	Case of $\cstyle{A} \sqcap \cstyle{B} \sqsubseteq \cstyle{C}$ (note that $\cstyle{C}$ can be $\bot$, in which case the proof below proceeds by contradiction).
	Notice it guarantees that there is an edge from $\cstyle{A}$ to $\cstyle{C}$ and from $\cstyle{B}$ to $\cstyle{C}$ in $\depgraph(\tbox)$.
	Assume we have $d \in (\cstyle{A} \sqcap \cstyle{B})^{\Imc'}$.
	If $d \in \inames(\kb)$, then we immediately obtain $d \in (\cstyle{A} \sqcap \cstyle{B})^{\Imc}$, thus $d \in \cstyle{C}^\Imc$ as $\Imc$ is a model of $\tbox$ (which is a contradiction if $\cstyle{C} = \bot$).
	It follows that $d \in \cstyle{C}^{\Imc'}$.
	Otherwise we have $w_1 (\rstyle{r}_1, t_1) d \in W$ with $\cstyle{A} \in t_1$ and $w_2 (\rstyle{r}_2, t_2) d \in W$ with $\cstyle{B} \in t_2$.
	These respectively yield $\cstyle{A} \in \type_\Imc(d)$ and $\cstyle{B} \in \type_\Imc(d)$.
	From $\Imc$ being a model of $\tbox$, we derive $\cstyle{C} \in \type_\Imc(d)$ (which is again a contradiction if $\cstyle{C} = \bot$).
	To obtain that $\cstyle{C} \in t_1$, which concludes the argument, it now suffices to recall that $\cstyle{A} \in t_1$ and that there is an edge from  $\cstyle{A}$ to $\cstyle{C}$ in $\depgraph(\tbox)$.
	\end{proof}

This lemma is our key to deriving tight complexity bounds for the weakly acyclic setting.
    	\begin{theoremrep}
		\label{theorem:el-is-nexpnp-no-top-weakly acyclic}
		$\mms$ in weakly acyclic $\ELIO_\bot$ is $\NExp^\NP$-complete.
		The lower bound holds  already for $\EL$. 
	\end{theoremrep}
	\begin{proofsketch}
          The upper bound immediately follows from Lemma \ref{lemma:exp-size-bound}.
       Indeed, we can use a naive  procedure that ``guesses'' an exponentially large candidate model $\I$ of the input KB, and checks non-existence of a smaller model $\J\subsetneq \I$ using an  NP oracle.

          For the lower bound, we provide a reduction from (the complement of) \emph{succinct} \ccol \cite{DBLP:journals/tods/EiterGM97} to $\mms$.  
	Our reduction combines the ideas behind 
    encoding of Example \ref{example:flooding-technique} using {flooding}, and those illustrated in Example \ref{binarytree} that allows us to succinctly represent the exponentially large graph.
We construct a weakly acyclic KB $\K$ in $\mathcal{EL}$ and, 
{similarly as in Example \ref{binarytree}, 
we define some \emph{subgoal} concepts which are needed for the goal concept to be satisfied, and use them to ensure that in every minimal model we can find the following trees:}

\noindent$\bullet$ A tree $T_C$ per each color $C\in \{R,G,B\}$ of depth $n$, where each leaf corresponds to a vertex of the input graph with a color assignment. Using a minimality argument, we ensure that such trees have disjoint sets of leaves. 

\noindent$\bullet$  A tree $T_G$ of depth $n$ where each leaf corresponds to a node of the input graph. We craft a subgoal concept $\cn{Col}$ that must be satisfied at each leaf of $T_G$. The subgoal $\cn{Col}$ ensures that \begin{enumerate*}[(a)]
\item each leaf in $T_G$ is connect to the leaf in $T_C$ corresponding to the same node in the graph, for each $C\in \{R,G,B\}$;
\item at least one of the leaves is marked as the chosen color.
\end{enumerate*} 

\noindent$\bullet$ A tree $T_{V}$ of depth $2n$ where each leaf corresponds to a variable $v_{i,j}$. By crafting a dedicated subgoal concept, that must be satisfied at each leaf of $T_{V}$, we ensure that each variable has a unique truth assignment. 

\noindent$\bullet$  A large tree $T_F$, of depth $6n+2$, where each leaf encodes a tuple $(u,v,x,y, \sigma^1, \sigma^2)$ where: $u,v$ are vertexes in the input graph $G$, $x,y$ are variables and $\sigma^1,\sigma^2 \in \{0,1\}$ (polarities of the variables). 
Each $(u,v,x,y,\sigma^1,\sigma^2)$ represents a possible edge in the input graph $G$. Simulating the computation of a family of circuits (following \cite{BonattiLW09}), we mark the leaves of $T_F$ corresponding to \emph{real} edges in the graph $G$.

We construct two subgoal concepts that are satisfied if each leaf encoding $(u,v,x,y,\sigma^1, \sigma^2)$ is connect via dedicated roles to the leaves corresponding to $u,v$ (in $T_G$) and $x,y$ (in $T_{V}$). From such role connections, the leaves in $T_F$ can import the color assignment and the truth assignment. By evaluating the clause encoded in each $(x,y,\sigma^1, \sigma^2)$, we mark the leaves in $T_F$ corresponding to \emph{true} edges in $G$. If a bad color assignment is detected, the structure gets \emph{flooded} similarly to Example \ref{example:flooding-technique}. In particular, all the leaves of $T_C$, for each $C\in \{R,G,B\}$ are marked as chosen. A concept $\cn{Flood}$ is then propagated at all the leaves of $T_F$.
In particular, the satisfaction of $\cn{Flood}$ forces the leaves to be assigned to all colors. 
Instead of checking the satisfaction the subgoals at the (exponentially many) leaves, we transfer the check  to
the root of the tree as follows: 
\begin{align*}
	\exists \cn{l}_j. \cn{Goal}\sqcap \exists \cn{r}_j. \cn{Goal}\sqsubseteq \cn{Goal}\text{ for all }0\leq j\leq \mathrm{tree\;depth}
\end{align*}
In a minimal model, $\cn{Goal}$ is satisfied at the root $r$ of the tree if and only if $\cn{Goal}$ is true at each leaf. 
Finally, we construct a final goal concept that is satisfied in a minimal model if all the roots of the different trees satisfy their respective goals.
{We emphasize that the
propagation of the $\cn{Flood}$ concept 
requires axioms that are  
not strongly acyclic.}
	\end{proofsketch}
  \begin{toappendix}
    \input{nexptimeNPproof}
   \end{toappendix}
	\paragraph{Data Complexity} We also look at the data complexity of \mms under acyclicity restrictions.
        \begin{theoremrep}
         \mms for weakly acyclic
          $\ELIO_\bot$ 
          is $\Sigma^P_2$-complete in
          data complexity. The lower bound applies already to $\EL$.
        \end{theoremrep}

\begin{proof}
  We use essentially the same reduction as in Example~\ref{example:flooding-technique}. Specifically,
  we need to simulate the inclusions $V\sqsubseteq T \sqcup F$ and
  $ N \sqsubseteq C_1\sqcup C_2\sqcup C_3$, which are not allowed in
  $\mathcal{EL}$. Instead, we can use inclusions of the form
   $A\sqsubseteq\exists r.B$, which contain a ``hidden''
  disjunction via existential quantification. 
   Dealing with $V\sqsubseteq T \sqcup F$ is not too difficult. We replace it with the following pair:
  \[V \sqsubseteq \exists\mathsf{hasValue}.\top  \quad \exists
     \mathsf{hasValue}.L \sqsubseteq L \quad \mbox{for }L\in \{T,F\}
  \]
  We further add the assertions $T(t_1),F(t_2)$ to $\A$, where
  $t_1,t_2$ are fresh individuals. We need to make sure that (in the
  relevant minimal models) the role $\mathsf{hasValue}$ points to $t_1$
  or $t_2$ for all individuals $v_{i,j}$. We will use a concept name $\mathsf{Ok}$ for this purpose. We add the following inclusion to $\T$:
  \[ \exists \mathsf{hasValue}.L \sqsubseteq \mathsf{Ok} \quad L\in \{T,F\}\]   
  Replacing  $ N \sqsubseteq C_1\sqcup C_2\sqcup C_3$ is a bit more tricky, because we need to do the ``flooding'' trick. We add the assertions $\mathsf{hascolor}(u_i,c_{i}^{\ell}) $ for all $0 \leq i < n$ and $\ell\in \{1,2,3\}$, where each $c_{i}^{\ell} $ is a fresh individual. Moreover, we add $C_1'(c_i^{1}),C_2'(c_i^{2}),C_3'(c_i^{3})$ for all $0 \leq i < n$. 
   Now we add the following inclusions to $\T$:  \[ N \sqsubseteq \exists \mathsf{hascolor}.\mathsf{Choice}\]
\[ \exists \mathsf{hascolor}.(\mathsf{Choice} \sqcap C_i')  \sqsubseteq C_i \qquad i\in \{1,2,3\}\]

We need to make sure that $\mathsf{Ok}$ holds for individuals that correspond to propositional variables and vertics.  For
this, we take any enumeration $w_1,\ldots w_k$ of the individuals of
the form $v_{i,j}$ and $u_j$. Further add assertions
$\mathsf{next}(w_1,w_2),\ldots,\mathsf{next}(w_{k-1},w_k)$
to $\A$. We require that $w_1$ corresponds to some vertex, i.e. it is
the form $u_j$.  We further add
$\mathsf{first}(w_1), \mathsf{last}(w_k)$ to $\A$. We can now ``propagate'' $\mathsf{Ok}$ backwards along this ordering:
\[ \mathsf{Last} \sqcap \mathsf{Ok} \sqsubseteq \mathsf{Marked}\]
\[\mathsf{Ok}\sqcap \exists \mathsf{next}.\mathsf{Marked} \sqsubseteq
  \mathsf{Marked}\]

The goal concept is defined as
\[D=\mathsf{First}\sqcap C_1 \sqcap C_2 \] Finally, we need to redefine the role
$s$. Now every individual $c_i^{k}$ can ``see'' all individuals
$e_{i,j}$: we add $s(c_k^{h},e_{i,j})$ to $\A$, for all
$0 \leq k,i,j < n$ and $u\in\{1,2,3\}$.

 Assume $G$ is a positive instance of \ccol. Then there
 exists a truth value assignment $t$ such that the induced subgraph
 $t(G)$ in not 3-colorable. Let $\Delta$ be the set of individuals in $\A$.
 Take an interpretation $\I$ such that:
 \begin{enumerate}
 \item $\Delta^{\I}=\Delta$ and $a^\I=a$ for all $a\in \Delta$, 
 \item $T^{\I} =\{v_{i,j}\mid v_{i,j} \mbox{ is true in } t\}$,
 \item $F^{\I}=\{v_{i,j}\mid v_{i,j} \mbox{ is false in } t\}$,
 \item $\mathsf{hasValue}^{\I}=T^{\I} \times \{t_1\}\cup F^{\I} \times \{t_2\} $,  
 \item  $ \mathsf{Ok}^{\I}= \mathsf{Marked}^{\I}=\{w_1,\ldots,w_k\}$,
 \item  $ \mathsf{Choice}^{\I}= \{ c_i^{q} \mid  0 \leq i < n, q\in \{1,2,3\}\}$
 \item $C_1^{\I}=C_2^{\I}=C_3^{\I}=\{u_0,\ldots,u_{n-1} \}$,
 \item
   $\mathsf{Sel}^{\I}=\{e_{i,j}\mid (i,j) \mbox{ is an edge in }t(G)
   \}$,
 \item the extension of the remaining predicates is exactly as given in
   $\A$.
 \end{enumerate}
 We can verify that $\I$ is a minimal model of $(\T,\A)$ with  $D^{\I}\neq\emptyset$. Indeed, after a basic inspection, we see that the only possibility for $\I$ to be non-minimal is in case we can remove something from $\mathsf{Choice}$ or one of $C_i$ while maintaining a model. But then it must be the case that the flooding-inducing inclusion never fires, which means that $\mathsf{Choice}$ and  $C_i$ can be minimized down to a proper coloring of $t(G)$, which would contradict the assumption that $t(G)$ is not 3-colorable.

 For the other direction, assume some minimal model $\I$ of $(\T,\A)$
 with $D^{\I}\neq\emptyset$. We let $t$ be the truth assignment that
 makes a proposition $v_{i,j}$ true iff $v_{i,j}\in T^{\I}$. It not
 difficult to see that $t(G)$ is not 3-colorable, and thus $G$ is a
 positive instance of \ccol. Towards a contradiction, assume that $t(G)$ is 3-colorable and this is witnessed by a color assignment $f:\{0,\ldots,n-1\}\rightarrow\{1,2,3\}$. Take the interpretation $\J$ such that:
 \begin{itemize}
 \item $\Delta^{\J}=\Delta^{\I}$,
 \item $a^{\J}=a^{\I}$, for all individuals $a$,
 \item $T^{\J}=\{v_{i,j}\mid v_{i,j} \mbox{ is true in } t\}$,
 \item $F^{\J}=\{v_{i,j}\mid v_{i,j} \mbox{ is false in } t\}$,

\item $\mathsf{hasValue}^{\I}=T^{\I} \times \{t_1\}\cup F^{\I} \times \{t_2\} $,  
\item  $ \mathsf{Ok}^{\I}= \mathsf{Marked}^{\I}=\{w_1,\ldots,w_k\}$,

   \item $C_j^{\J}=\{u_i^{\I}\mid f(i)=j\}$ for $j\in\{1,2,3\}$,

 \item  $ \mathsf{Choice}^{\I}= \{ c_i^{f(i)} \mid  0 \leq i < n\}$,

 \item
   $\mathsf{Sel}^{\I}=\{e_{i,j}\mid (i,j) \mbox{ is an edge in }t(G)
   \}$,
 \item the extension of the remaining predicates is exactly as given in
   $\A$.
    
 \end{itemize}
 It can be verified that $\J$ is a model of $(\T,\A)$. Moreover, since
 the initial interpretation $\I$ is ``flooded'' (i.e., since
 $D\neq\emptyset$), we can also verify that $\J \subset \I$, which
 contradicts the assumption that $\I$ is a minimal model of $(\T,\A)$.
\end{proof}
\begin{proofsketch}
  For the upper bound we use
  Lemma~\ref{lemma:exp-size-bound}. Assuming that the TBox is fixed, 
  if there exists a minimal model $\Imc$ of $(\T,\A)$ that
  satisfies a concept $C$ of interest, then there exists
  such an interpretation $\Jmc$ whose domain 
  is bounded by
  $c\times \sizeof{\inames(\kb)}$, where $c$ is a constant that only
  depends on $\T$. In other words, the size of $\Jmc$ that witnesses
  $C$ is polynomial in the size of $\A$. Note that,
  given a candidate $\Jmc$ as above, we can use NP oracle to check
  whether $\Jmc$ is (non-)minimal. This yields the $\Sigma^{P}_{2}$
  upper bound.

  For the lower-bound, we can mainly use the reduction that was
  described in Example~\ref{example:flooding-technique}. Specifically,
  we need to simulate the inclusions $V\sqsubseteq T \sqcup F$ and
  $ N \sqsubseteq C_1\sqcup C_2\sqcup C_3$, which are not allowed in
  $\mathcal{EL}$. Instead, we can use inclusions of the form
   $A\sqsubseteq\exists r.B$, which contain a ``hidden''
  disjunction via existential quantification. 
\end{proofsketch}

	\section{Minimal Models in Related Formalisms}\label{section:related}

We make a very brief excursion into the DL-Lite family, and briefly discuss our results in the setting of databases with \emph{tuple-generating dependencies (TGDs)}.
We also look at the impact of the UNA on minimal model reasoning in $\EL$. 

\subsection{DL-Lite} 
We did not study DL-Lite in this paper, and the feasibility of \mms in this family of DLs remains an intriguing question for future work. We only present one interesting result that hints that the problem will not be easy. In very stark contrast to the previously known 
$\NL$-membership for \mms  in \dllitecore\xspace \cite{BonattiD0S23}, already in \dllitehorn we have \ExpSpace-hardness.

	\begin{theoremrep}
		\label{theorem:dllite-horn-is-nexp-hard}
		$\mms$ in $\dllitehorn$ is $\ExpSpace$-hard.
	\end{theoremrep} 
	\noindent We believe that this bound, proved by reducing the acceptance problem of a Turing machine with exponential space, is likely to be tight, but leave the question for future work. 
	
	\begin{toappendix}
		
		\newcommand{\bit}{\cstyle{B}}
		\newcommand{\succbit}{\cstyle{S}}
		\newcommand{\predbit}{\cstyle{P}}
		\newcommand{\firstzero}{\cstyle{FZ}}
		\newcommand{\firstone}{\cstyle{FO}}
		\newcommand{\firstoftape}{\cstyle{FoT}}
		\newcommand{\lastoftape}{\cstyle{LoT}}
		\newcommand{\notfirstoftape}{\overline{\cstyle{FoT}}}
		\newcommand{\notlastoftape}{\overline{\cstyle{LoT}}}
		\newcommand{\address}{\cstyle{Pos}}
		\newcommand{\head}{\cstyle{H}}
		\newcommand{\nothead}{\overline{\cstyle{H}}}
		\newcommand{\generate}{\rstyle{gen}}
		\newcommand{\asc}{\cstyle{Asc}}
		\newcommand{\dsc}{\cstyle{Dsc}}
		\newcommand{\ascOK}{\cstyle{FullTape}}
		\newcommand{\dscOK}{\cstyle{HalfTape}}
		\newcommand{\report}{\rstyle{report}}
		\newcommand{\symb}{\cstyle{Symb}}
		\newcommand{\state}{\cstyle{Q}}
		\newcommand{\transition}{\rstyle{trans}}
		\newcommand{\copysymbol}{\cstyle{Copy}}
		\newcommand{\replaceby}{\cstyle{Replace}}
		\newcommand{\import}{\rstyle{get}}
		\newcommand{\export}{\rstyle{send}}
		\newcommand{\imported}{\cstyle{Got}}
		\newcommand{\exported}{\cstyle{OldSymb}}
		\newcommand{\moved}{\cstyle{M}}
		\newcommand{\towrite}{\cstyle{Write}}

		We reduce from the problem of deciding whether an exponentially space-bounded deterministic Turing machine (DTM) accepts a given input.
		
		To fix notation, we recall that an DTM $\Mmc$ is specified by a
		$6$-tuple $\Mmc = (Q, \Sigma, \Gamma, \delta, q_0, g)$ where:
		\begin{itemize}
			\item $Q$ is the finite set of states;
			\item $\Sigma$ is the finite input alphabet;
			\item $\Gamma \supseteq \Sigma$ is the finite tape alphabet
			with a special blank symbol $\textvisiblespace \in \Gamma
			\setminus \Sigma$;
			\item $\delta : Q \times \Gamma \rightarrow Q \times \Gamma \times \{ \triangleleft,  \triangleright \}$ is the transition function;
			\item $q_0 \in Q$ is the initial state;
			\item $F \subseteq Q$ is the set of final states.
		\end{itemize}
		Note that without loss of generality, we consider DTM that:
		\begin{itemize}
			\item never tries to go to the left of the initial position;
			\item never revisits the initial state.
		\end{itemize} 
		We say that $\Mmc$ is \emph{exponentially space-bounded} if there
		exists a polynomial $p$ such that on input $x$, $\Mmc$ visits only the
		first $2^{p(\sizeof{x})}$ tape cells.
		
		We construct a KB $\kb = (\tbox, \abox)$ and a concept $\goal$ such that $\goal$ is minimally satisfiable w.r.t.\ $\kb$ iff $\Mmc$ accepts (\emph{i.e.}\ reaches a final state) on input $x$.
		
		Our first challenge is to maintain, in our models, an exponentially-long tape as the run of the DTM progresses.
		Each step in the run is associated with its own copy of the tape; and each copy of the tape is represented using exponentially-many elements in the model.
		Positions of each cell along (a copy of) that tape are encoded using $2n$ concept names $\bit_i^b$, for $0 \leq i \leq n-1$ and $b \in \{0, 1\}$.
		Doing so, a binary integer on $n$ bits $b_{n-1}\dots b_0$ is represented by the conjunction of concepts $\bit_{n-1}^{b_{n-1}}\dots \bit_0^{b_0}$.
		We guarantee that each cell can only have one position with the following axioms:
		\[
		\bit_i^{0} \sqcap \bit_i^{1} \sqsubseteq \bot \text{ for } 0 \leq i \leq n-1
		\]
		The tape will always be generated starting from the cell of the head until both endpoints of the tape are reached.
		We use a concept to distinguish the position being First on the Tape ($\firstoftape$) from others ($\notfirstoftape$); and similarly, a concept to distinguish the position being Last on the Tape ($\lastoftape$) from others ($\notlastoftape$).
		These concepts are connected to the binary encoding by the following axioms:
		\begin{align*}
			\bigsqcap_{i = 0}^{n-1} \bit_i^0 & \sqsubseteq \firstoftape
			&
			\bit_i^1 & \sqsubseteq \notfirstoftape &  \text{for } 0 \leq i \leq n-1
			\\
			\bigsqcap_{i = 0}^{n-1} \bit_i^1 & \sqsubseteq \lastoftape
			&
			\bit_i^0 & \sqsubseteq \notlastoftape & \text{for } 0 \leq i \leq n-1
		\end{align*}
		We generate positions one by one and thus need to perform the usual increment/decrement on the binary encoding.
		The position of the succeeding cell is stored using $2n$ concepts $\succbit_i^b$ and is handled in the standard way with the following axioms, where $\firstzero_i$ identifies the First Zero (starting from the weakest bit) as being the $i^\text{th}$ bit $b_i$''.
		\begin{align*} 
			\bit_i^0 \sqcap \bigsqcap_{j = 0}^{i-1} \bit_j^1 & \sqsubseteq \firstzero_i 
			& & \text{for } 0 \leq i \leq n-1
			\\
			\bit_i^b \sqcap \firstzero_j & \sqsubseteq \succbit_i^{1-b} 
			& & \text{for } 0 \leq i  \leq j \leq n-1, b \in \{ 0, 1\}
			\\
			\bit_i^b \sqcap \firstzero_j & \sqsubseteq \succbit_i^{b} 
			& & \text{for } 0 \leq j < i \leq n-1, b \in \{ 0, 1\}
		\end{align*}
		We proceed as well to represent the position of the preceding cell with $2n$ concepts $\predbit_i^b$, concepts $\firstone_i$, this time for First One, and axioms:
		\begin{align*} 
			\bit_i^1 \sqcap \bigsqcap_{j = 0}^{i-1} \bit_j^0 & \sqsubseteq \firstone_i 
			& & \text{for } 0 \leq i \leq n-1
			\\
			\bit_i^b \sqcap \firstone_j & \sqsubseteq \predbit_i^{1-b} 
			& & \text{for } 0 \leq i  \leq j \leq n-1, b \in \{ 0, 1\}
			\\
			\bit_i^b \sqcap \firstone_j & \sqsubseteq \predbit_i^{b} 
			& & \text{for } 0 \leq j < i \leq n-1, b \in \{ 0, 1\}
		\end{align*}
		Now that the representation and manipulation of positions is clarified, we describe in full extend how the whole generation of the computation and associated copies of the tape works.
		We start with a basic encoding of the initial state, position of the head and initial part of the tape (encoding input $x$) using assertions in the ABox.
		The rest of the initial tape will then be generated with blank symbols using some extra rules.
		Once the end of the initial tape has been properly initialized, its last position will ``report'' back to cell of the head, which will trigger the first inductive step in the representation of the computation.
		Assume we are at step $s_k$ of the DTM computation, and that the cell of the head sees state $q$, reads symbol $c$, and receives confirmation that the tape of step $s_k$ is fully generated via a concept $\ascOK$, then we have three cases.
		If the transition function leads to a final state, then we immediately derive the goal concept, with the following axioms:
		\[
		\head \sqcap \state_q \sqcap \symb_c \sqcap \ascOK \sqsubseteq \goal
		\]
		for every $q \in Q$ and $c \in \Gamma$ such that $\delta(q, c) = (q', c', m)$ with $q' \in F$.
		Otherwise, we properly take the transition to step $s_{k+1}$ by generating a combination of roles $\transition_{q', c', m, i, b}$, where $\delta(q, c) = (q', c', m)$, that points to the new head position as encoded by the chosen combination of $i$'s and $b$'s.
		To this end, we differentiate on the value of $m \in \{ \triangleleft, \triangleright \}$ and rely either on the stored position of the preceding or succeeding cell accordingly.
		This is done via the following axioms:
		\[
		\head \sqcap \state_q \sqcap \symb_c \sqcap \ascOK \sqcap \predbit_i^b \sqsubseteq \exists \transition_{q', c', \triangleleft, i, b}
		\]
		for every $q \in Q$ and $c \in \Gamma$ such that $\delta(q, c) = (q', c', \triangleleft)$ with $q' \notin F$, and every $0 \leq i \leq n-1$ and $b \in \{ 0, 1 \}$.
		\[
		\head \sqcap \state_q \sqcap \symb_c \sqcap \ascOK \sqcap \succbit_i^b \sqsubseteq \exists \transition_{q', c', \triangleright, i, b}
		\]
		for every $q \in Q$ and $c \in \Gamma$ such that $\delta(q, c) = (q', c', \triangleright)$ with $q' \notin F$, and every $0 \leq i \leq n-1$ and $b \in \{ 0, 1 \}$.
		A cell receiving such a role inherits all the relevant information regarding the freshly taken transition, with the following axioms defined for every $q \in \state$, $c \in \{ 0, 1 \}$, $m \in \{ \triangleleft, \triangleright \}$, and $b \in \{ 0, 1 \}$:
		\begin{align*}
		\exists \transition_{q, c, m, i, b}^- & \sqsubseteq \head \sqcap \state_q \sqcap \towrite_c \sqcap \moved_m \sqcap \bit_i^b
		\end{align*}
		To prevent elements from having several states / symbols, we make all the relevant combinations disjoint:
		\begin{align*}
			\state_{q_1} \sqcap \state_{q_2} & \sqsubseteq \bot
			& \text{for } q_1, q_2 \in Q, q_1 \neq q_2
			\\
			\symb_{c_1} \sqcap \symb_{c_2} & \sqsubseteq \bot
			& \text{for } c_1, c_2 \in \Gamma, c_1 \neq c_2
			\\
			\moved_\triangleright \sqcap \moved_\triangleleft & \sqsubseteq \bot
			\\
			\head \sqcap \overline{\head} & \sqsubseteq \bot
		\end{align*}
		where $\overline{\head}$ is a concept name standing for `not the head' that will hold on all non-head cells, as will be clarified further.
		
		There is \emph{a priori} no guarantee that those $n$ generated roles agree to point to the same element, but this will be necessary for the (representation of the) computation to continue and eventually derive $\goal$.
		To ensure this, we require the concept $\address$, standing for ``the present element has a fully specified position'', to hold before continuing the process.
		We add the following axioms:
		\begin{align*}
			\bit_i^b & \sqsubseteq \address_i
			& \text{for } 0 \leq i \leq n-1, b \in \{ 0, 1\}
			\\
			\bigsqcap_{i = 0}^{n-1} \address_i & \sqsubseteq \address
		\end{align*}
		Every cell with a proper position and on a tape that is not the initial tape is then required to read the symbol that was present at the same position on the previous tape.
		This is managed by two successive generations of combinations of roles.
		First, a combination of $\import_{i, b}$ roles points to the desired position; we add the following axioms for every $0 \leq i \leq n-1$, $b \in \{ 0, 1 \}$ and $q \in Q \setminus \{ q_0 \}$:
		\begin{align*}
			\state_q \sqcap \address \sqcap \bit_i^b & \sqsubseteq \exists \import_{i, b}
		\end{align*}
		Then, if the full combination is received, a predicate $\imported$ is derived at the corresponding cell on the previous tape.
		\begin{align*}
			\exists \import^-_{i, b} & \sqsubseteq \imported_{i} \sqcap \bit_i^b
			& & \text{for } 0 \leq i \leq n-1, b \in \{ 0, 1 \}
			\\
			\bigsqcap_{i = 0}^{n - 1} \imported_{i} & \sqsubseteq \imported
		\end{align*}
		In turns, this triggers the same combination of role $\export_{i,b,c}$, this time also specifying the symbol $c$ to read.
		We add, for every $0 \leq i \leq n-1$, $b \in \{ 0, 1 \}$ and $c \in \Gamma$:
		\begin{align*}
			\imported \sqcap \symb_c \sqcap \bit_i^b & \sqsubseteq \exists \export_{i,b,c}
			& & 
		\end{align*}
		Finally, if the full combination of the above is received, then a predicate $\exported_c$ is derived.
		We add, for every $c \in \Gamma$:
		\begin{align*}
			\exists \export^-_{i, b, c} & \sqsubseteq \exported_{i, c} \sqcap \bit_i^b
			& & \text{for } 0 \leq i \leq n-1, b \in \{ 0, 1 \}
			\\
			\bigsqcap_{i = 0}^{n - 1} \exported_{i, c} & \sqsubseteq \exported_c
		\end{align*}
		Importantly, notice that minimality of considered models will ensure that each cell can only send its symbol once.
		This prevents the tape generated for step $s_{k+2}$ from reading symbols on the tape associated to step $s_k$ (as the latter will already have been read by cells on $s_{k+1}$'s tape).
		Now, whether to keep or overwrite this symbol from the previous step is dictated by dedicated concepts $\copysymbol$ and $\replaceby_{c'}$, where $\replaceby_{c'}$ specifies the symbol that is to be written instead:
		\begin{align*}
			\exported_c \sqcap \copysymbol & \sqsubseteq \symb_c
			& & \text{for } c \in \Gamma
			\\
			\exported_c \sqcap \replaceby_{c'} & \sqsubseteq \symb_c
			& & \text{for } c, c' \in \Gamma
			\\
			\symb_c & \sqsubseteq \symb
			& & \text{for } c \in \Gamma
		\end{align*}
		We explain where those predicates $\copysymbol$ and $\replaceby_{c'}$ come from.
		At the position of the head, we clearly want to copy the previous tape (recall that we just moved the head in one of the two directions), thus we add:
		\begin{align*}
			\head & \sqsubseteq \copysymbol
		\end{align*}
		Once a cell is done reading the previous symbol (identified by concept $\symb$), we start generating the missing cells in descending order with dedicated roles $\generate_{\triangleleft, q, i, b, \copysymbol}$, asking at each cell for the symbol of the previous tape to be read before moving on to the next cell.
		Once we reach the initial position, we generate a $\report_{\firstoftape}$ that is intended to collapse back on the cell where the head is located:
		\begin{align*}
			\firstoftape \sqcap \symb & \sqsubseteq \exists \report_{\firstoftape}
			\\
			\notfirstoftape \sqcap \head \sqcap \symb \sqcap \state_q \sqcap \moved_\triangleleft \sqcap \predbit_i^b & \sqsubseteq \exists \generate_{\triangleleft, q, i, b, \copysymbol} 
			\\
			\notfirstoftape \sqcap \head \sqcap \symb \sqcap \state_q \sqcap \moved_\triangleright \sqcap \towrite_c \sqcap \predbit_i^b & \sqsubseteq \exists \generate_{\triangleleft, q, i, b, \replaceby, c}
			\\
			\notfirstoftape \sqcap \nothead \sqcap \symb \sqcap \state_q \sqcap \predbit_i^b & \sqsubseteq \exists \generate_{\triangleleft, q, i, b, \copysymbol}
		\end{align*}
		Those predicates carry all the relevant information for the generated cell:
		\begin{align*}
			\exists \report_{\firstoftape}^- & \sqsubseteq \dscOK \sqcap \head
			\\
			\exists \generate_{\triangleleft, q, i, b, \copysymbol}^- & \sqsubseteq \nothead \sqcap \state_q \sqcap \copysymbol \sqcap \bit_i^b
			\\
			\exists \generate_{\triangleleft, q, i, b, \replaceby, c} ^- & \sqsubseteq \nothead \sqcap \state_q \sqcap \replaceby_{c} \sqcap \bit_i^b
		\end{align*}
		 Now once the concept $\dscOK$ is seen on the head, we restart the process, this times generating the missing cells in the increasing order with roles $\generate_{\triangleright, q, i, b, \copysymbol}$.
		 Once we reach the last position, we generate a $\report_{\lastoftape}$ that is intended to collapse back on the cell where the head is located:
		 \begin{align*}
		 	\lastoftape \sqcap \symb & \sqsubseteq \exists \report_{\lastoftape}
		 	\\
		 	\notlastoftape \sqcap \head \sqcap \symb \sqcap \state_q \sqcap \moved_\triangleright \sqcap \succbit_i^b & \sqsubseteq \exists \generate_{\triangleright, q, i, b, \copysymbol} 
		 	\\
		 	\notlastoftape \sqcap \head \sqcap \symb \sqcap \state_q \sqcap \moved_\triangleleft \sqcap \towrite_c \sqcap \succbit_i^b & \sqsubseteq \exists \generate_{\triangleright, q, i, b, \replaceby, c} 
		 	\\
		 	\notlastoftape \sqcap \nothead \sqcap \symb \sqcap \sqcap \state_q \succbit_i^b & \sqsubseteq \exists \generate_{\triangleright, q, i, b, \copysymbol}
		 \end{align*}
		 Those predicates again carry all the relevant information for the generated cell:
		 \begin{align*}
		 	\dscOK \sqcap \exists \report_{\lastoftape}^- & \sqsubseteq \ascOK \sqcap \head
		 	\\
		 	\exists \generate_{\triangleright, q, i, b, \copysymbol}^- & \sqsubseteq \nothead \sqcap \state_q \sqcap \copysymbol \sqcap \bit_i^b
		 	\\
		 	\exists \generate_{\triangleright, q, i, b, \replaceby, c} ^- & \sqsubseteq \nothead \sqcap \state_q \sqcap \replaceby_{c} \sqcap \bit_i^b
		 \end{align*}
		 It only remains to clarify how the initial input $\bar x = x_0, \dots, x_{n-1}$ is represented in the ABox.
		 For each $0 \leq k \leq n$, we add the assertion $\bit_i^b(\istyle{a}_k)$ if the $i^\text{th}$ bit in the binary encoding of $k$ is $b$.
		 For the symbols, recall that cells with the initial state $q_0$ do not attempt to read the symbol on the corresponding position on the previous tape and are also not subject to the generating rules we describe earlier.
		 We thus directly place the symbols and current state with assertions:
		 $\symb_{x_k}(\istyle{a}_k)$ for each $0 \leq k \leq n -1$ and $\symb_{\textvisiblespace}(\istyle{a}_n)$.
		 We also place the head and initial states by hand: $\head(\istyle{a}_0)$ and $\state_{q_0}(\istyle{a}_{k})$ for each $0 \leq k \leq n$.
		 For the missing cells, we produce generating rules that follow the same logic as the previous ones, relying on roles $\generate_{\textvisiblespace, i, b}$ and producing the missing cells in increasing order starting from the first blank cell that comes immediately after the last one encoding the input (represented by individual $\istyle{a}_n$).
		 \begin{align*}
		 	\address \sqcap \notlastoftape \sqcap \symb_{\textvisiblespace} \sqcap \state_{q_0} \sqcap \succbit_i^b & \sqsubseteq \exists \generate_{\textvisiblespace, i, b}
		 	\\
		 	\exists \generate_{\textvisiblespace, i, b}^- & \sqsubseteq \symb_{\textvisiblespace} \sqcap \state_{q_0} \sqcap \bit_i^b
		 \end{align*}
		
		We now prove that $\Mmc$ accepts on input $\bar x$ iff the concept $\goal$ is minimally satisfiable w.r.t.\ the above $\kb$.
		
		\newcommand{\Smc}{\mathcal{S}}
		To this end, we first describe the intended model of $\kb$: the one that does represent the computation of $\Mmc$ on input $\bar x$.
		Let $\Smc := s_0, \dots, s_k, \dots$ the (potentially infinite) successive configurations in the run of $\Mmc$ on input $\bar x$.
		If this sequence if finite, we denote $s_K$ the final configuration.
		For an integer $m$ between $0$ and $2^n -1$, we denote $b_i(m)$ its $i$-th bit in binary so that $m = \sum_{i = 0}^{n-1} b_i(m)2^i$.
		The symbol on the $m$-th cell on the tape at configuration $s_k$ is denoted $t_k(m)$.
		The position of the head at configuration $s_k$ is denoted $h(k)$.
		The state at configuration $s_k$ is denoted $Q(k)$.
		The movement $\{ \triangleright, \triangleleft\}$ taken on the transition from $s_k$ to $s_{k+1}$ is denote $mov(k)$.
		`The' model $\Imc$ that represents the run of $\Mmc$ on input $\bar x$ has domain $\Delta^\Imc := \Smc \times \{ 0, \dots, 2^n - 1 \}$.
		The interpretation of concept names $\address_i$, $\address$, $\symb$ is $\Delta^\Imc$.
		The interpretation of other concept names is defined according to the following rules: an element $(s_k, m) \in \Delta^\Imc$ satisfies:
		\[
		\begin{array}{c}
			\bit_i^b \text{\ ~iff\ ~} b_i(m) = b
			\qquad
			\succbit_i^b \text{\ ~iff\ ~} b_i(m+1) = b \text{ and } m \neq 2^n - 1
			\qquad
			\predbit_i^b \text{\ ~iff\ ~} b_i(m-1) = b \text{ and } m \neq 0
			\smallskip\\
			\firstzero_i \text{\ ~iff\ ~} b_i(m) = 1 \text{ for all } j < i \text{ and } b_i(m) = 0
			\qquad
			\firstone_i \text{\ ~iff\ ~} b_i(m) = 0 \text{ for all } j < i \text{ and } b_i(m) = 1
			\smallskip \\
			\lastoftape \text{\ ~iff\ ~} m = 2^n - 1
			\qquad
			\overline{\lastoftape} \text{\ ~iff\ ~} m \neq 2^n - 1
			\qquad
			\firstoftape \text{\ ~iff\ ~} m = 0
			\qquad
			\overline{\firstoftape} \text{\ ~iff\ ~} m \neq 0
			\smallskip \\
			\head \text{\ ~iff\ ~} h(k) = m
			\qquad
			\overline{\head} \text{\ ~iff\ ~} h(k) \neq m
			\qquad
			\dscOK \text{\ ~iff\ ~} h(k) = m 
			\qquad 
			\ascOK \text{\ ~iff\ ~} h(k) = m
			\smallskip \\
			\symb_c \text{\ ~iff\ ~} t_k(m) = c
			\qquad
			\state_q \text{\ ~iff\ ~} Q(k) = q
			\smallskip \\
			\moved_\triangleleft \text{\ ~iff\ ~} k > 0 \text{ and } h(k) = h(k-1) - 1
			\qquad
			\moved_\triangleright \text{\ ~iff\ ~} k > 0 \text{ and } h(k) = h(k-1) + 1
			\smallskip \\
			\imported_i \text{\ ~iff\ ~} k \neq K
			\qquad
			\imported \text{\ ~iff\ ~} k \neq K
			\qquad
			\exported_{i, c} \text{\ ~iff\ ~} k > 0 \text{ and } t_{k - 1}(m) = c
			\qquad
			\exported_c \text{\ ~iff\ ~} k > 0 \text{ and } t_{k - 1}(m) = c
			\smallskip \\
			\towrite_c \text{\ ~iff\ ~} k > 0  \text{ and } t_k(h(k)) = c
			\qquad
			\copysymbol \text{\ ~iff\ ~} k > 0 \text{ and } h(k - 1) \neq m
			\qquad
			\replaceby_{c} \text{\ ~iff\ ~} k > 0 \text{ and } h(k - 1) = m
			\smallskip \\
			\goal \text{\ ~iff\ ~} k = K \text{ and } m = h(k)
		\end{array}
		\]
		The interpretation of roles is given in the same manner: a pair $((s_\ell, m_1), (s_k, m_2)) \in (\Delta^\Imc)^2$ satisfies:
		\begin{align*}
			\transition_{q, c, m, i, b} & \text{\ ~iff\ ~} k = \ell + 1, Q(\ell) = q, m_1 = h(\ell), t_\ell(m_1) = c, b = b_i(m_1), m_2 = h(\ell + 1), m = mov(\ell)
			\\
			\generate_{\triangleright, q, i, b, \towrite_c} & \text{\ ~iff\ ~} k = \ell \geq 0, Q(\ell) = q, m_1 \geq h(\ell), m_2 = m_1 + 1, b = b_i(m_2), m_2 = h(\ell - 1), t_\ell(h(k)) = c
			\\
			\generate_{\triangleright, q, i, b, \copysymbol} & \text{\ ~iff\ ~} k = \ell \geq 0, Q(\ell) = q, m_1 \geq h(\ell), m_2 = m_1 + 1, b = b_i(m_2), m_2 \neq h(\ell - 1)
			\\
			\generate_{\triangleleft, q, i, b, \towrite_c} & \text{\ ~iff\ ~} k = \ell \geq 0, Q(\ell) = q, m_1 \leq h(\ell), m_2 = m_1 1 1, b = b_i(m_2), m_2 = h(\ell - 1), t_\ell(h(k)) = c
			\\
			\generate_{\triangleleft, q, i, b, \copysymbol} & \text{\ ~iff\ ~} k = \ell \geq 0, Q(\ell) = q, m_1 \leq h(\ell), m_2 = m_1 - 1, b = b_i(m_2), m_2 \neq h(\ell - 1)
			\\
			\generate_{\textvisiblespace, i, b} & \text{\ ~iff\ ~} k = \ell = 0, m_2 = m_1 + 1, m_1 \geq n, b = b_i(m_2)
			\\
			\import_{i, b} & \text{\ ~iff\ ~} \ell = k + 1, m_1 = m_2, b = b_i(m_1)
			\\
			\export_{i, b, c} & \text{\ ~iff\ ~} k = \ell + 1, m_1 = m_2, b = b_i(m_1), t_\ell({m_1}) = c
			\\
			\report_{\firstoftape} & \text{\ ~iff\ ~} k = \ell, m_1 = 2^n - 1, m_2 = h(k)
			\\
			\report_{\lastoftape} & \text{\ ~iff\ ~} k = \ell, m_1 = 0, m_2 = h(k)
		\end{align*}
		It is tedious but not hard to verify the following claim:
		\begin{claim}
			\label{claim:intended-model-is-minimal}
			$\Imc$ is a minimal model of $\kb$.
		\end{claim}
		Modelhood can be checked rule by rule, simply examining the definition of $\Imc$.
		For minimality, one can use the following order $<$ on elements of $\Delta^\Imc$ to verify that each fact is indeed necessary, so that the whole model is in fact minimal.
		Given $(s_\ell, p), (s_k, q) \in \Delta^\Imc$, we denote $(s_\ell, p) < (s_k, q)$ iff one of the following (mutually exclusive) conditions holds:
		\begin{itemize}
			\item $\ell < k$;
			\item $\ell = k$ and $q < p \leq h_k$;
			\item $\ell = k$ and $p \leq h_k < q$;
			\item $\ell = k$ and $h_k < p < q$.
		\end{itemize}
		It is not hard to verify that $<$ defines a strict total order on $\Delta^\Imc$.
		Intuitively, it roughly describes which elements a given one depends on (it depends on its predecessors w.r.t.\ $<$): the starting elements, $(s_0, 0)$ up to $(s_0, n)$ correspond to the ABox individuals, with $(s_0, 0)$ being the smallest element in $\Delta^\Imc$ w.r.t.\ $<$.
		Since the head is initially placed at position $0$, the other elements on the first tape are ordered according to the fourth bullet in the definition of $<$, which corresponds to how the initial blank cells are generated in increasing order w.r.t.\ their coordinate.
		Now, once a tape is complete, the current head may satisfy the predicate $\ascOK$ (obtained from the last cell pointing back to the head with role $\report_{\lastoftape}$), which triggers the generation of the head on the next tape (unless the $\goal$ is derived in case the represented computation reached a final state).
		This explains the first bullet in the definition of $<$.
		Now, in general, the new head will not be on the first cell of the tape (with coordinate $0$), and recall that we first generate the cells in descending order from the head, which corresponds to the second bullet in the definition of $<$.
		Once the first cell of the current tape has been reached, it can report back to the current head with concept $\dscOK$ (by pointing the role $\report_{\firstoftape}$ on the current head), which triggers the generation of cells in increasing order from the head position (hence the fourth bullet in the definition of $<$, and then the third is useful again).

		We now prove the $(\Rightarrow)$ direction of the claim: assume that $\Mmc$ accepts on input $\bar x$.
		Therefore the sequences $s_0, \dots, s_M$ of successive configurations is finite.
		It then follows that, by definition of $\Imc$, the element $(s_M, h(M)) \in \Delta^\Imc$ does not start generating the next tape but instead satisfies $\goal$.
		Claim~\ref{claim:intended-model-is-minimal} guarantees modelhood and minimality, which concludes the $(\Rightarrow)$ direction.

		To prove the $(\Leftarrow)$ is more technical; and we sketch the most important arguments.
		Assume that there exists a model $\Jmc$ that satisfies $\goal$.
		We aim to prove that $\Jmc$ is isomorphic to $\Imc$.
		The intuition is that `as soon as' a model $\Jmc$ stops representing the actual computation of $\Mmc$ on input $\bar x$, as described in $\Imc$, then it may produce up to $n$ extra elements (\emph{e.g.}\ corresponding to the endpoints of a combination of roles that failed to collapse together), but those non-properly-collapsed elements cannot possibly support $\goal$. 
		In other words, the model $\Imc$ is the \emph{only one} (up to isomorphism) that can possibly satisfy $\goal$.
		To formalize this, we iteratively construct an injective homomorphism $\rho : \Imc \rightarrow \Jmc$ using the assumption that $\Jmc$ satisfies $\goal$.
		It is then immediate that $\Jmc = \rho(\Imc)$ (as otherwise $\rho(\Imc)$ would contradict the minimality of $\Jmc$).
		From there, it follows that $\Imc$ satisfies $\goal$ and, since $\Imc$ represents the run of $\Mmc$ on $\bar x$, that $\Mmc$ accepts $\bar x$ as $\goal$ may only be inferred from an accepting state.
		
		We start from $\rho_0 : \{ (s_0, 0), \dots, (s_0, n) \} \rightarrow \Delta^\Jmc$ being defined as $(s_0, i) \mapsto \istyle{a}_k^\Jmc$ (recall that the $\istyle{a}_k$ are the individuals from the ABox).
		The domain of $\rho$ is then extended according to $<$: if $\rho$ is defined on $\Delta \subseteq \Delta^\Imc$, we take the next element w.r.t.\ $<$ that is not in $\Delta$ yet.
		As we start from $\{ (s_0, 0), \dots, (s_0, n) \}$ and that $<$ is a linear order, this next element is always well-defined.
		Given $e \in \Delta^\Imc$, let us denote $\Delta_e := \{ d \in \Delta^\Imc \mid d < e \}$, and we denote $\rho_e$ the mapping defined from $\Delta_e$ to $\Delta^\Jmc$.
		A subtle point in the construction in that we cannot guarantee that all intermediate $\rho_e$ are actually homomorphisms $\Imc\vert_{\Delta_e} \rightarrow \Jmc$.
		This is due to the predicates $\report_{\firstoftape}$, $\report_{\lastoftape}$, $\dscOK$, $\ascOK$, $\imported_{i}$, $\imported$ that are derived to comply with the existential needs of some elements further in the order.
		In fact, one could refine the order $<$ to work directly on the facts of $\Imc$, roughly following the one presented here, apart from those special predicates.
		We omit such a cumbersome construction and rather claim that the following invariant property can be proven using the presented definition of $<$:
		\begin{claim}
			\label{claim:invariant}
			For every element $(s_\ell, q) \in \Delta^\Imc$, the mapping $\rho_{(s_\ell, q)} : \Delta_{(s_\ell, q)} \rightarrow \Delta^\Jmc$ defines an injective homomorphism from $\Imc\vert_{\Delta_{(s_\ell, q)}} \rightarrow \Jmc$, except on the following facts satisfied by $\Imc$:
			\begin{itemize}
				\item $\dscOK((s_\ell, h(\ell)))$ and $\report_{\firstoftape}((s_\ell, 0), (s_\ell, h(\ell)))$ if $0 < q \leq h(k)$;
				\item $\ascOK((s_\ell, h(\ell)))$ and $\report_{\lastoftape}((s_\ell, 2^n - 1), (s_\ell, h(\ell)))$ if $q < 2^n - 1$;
				\item $\imported_{i}((s_{\ell - 1}, p))$ and $\imported((s_{\ell - 1}, p))$ if $\ell > 0$ and $(s_\ell, p) > (s_\ell, q)$;
				\item $\imported_{i}((s_{\ell}, p))$ and $\imported((s_{\ell}, p))$ if $s_\ell$ is not the last configuration and $(s_\ell, p) \leq (s_\ell, q)$;
				\item $\goal((s_{\ell}, p))$ if $s_\ell$ is the last configuration and $p < 2^n - 1$.
			\end{itemize}
		\end{claim}
		It is not hard to check that this invariant is satisfied by $\rho_0$ (which is also $\rho_{(s_0, n)}$).
		We sketch the induction step.
		Assume constructed $\rho_{(s_\ell, q)}$ such that is satisfies Claim~\ref{claim:invariant}.
		There are 3 cases to distinguish:
		\begin{itemize}
			\item If $s_\ell$ is the last configuration and $q = 2^n - 1$, then we are done: the induction hypothesis actually guarantees that $\rho_{(s_\ell, q)} : \Imc \rightarrow \Jmc$ is an injective homomorphism.
			\item If $s_\ell$ is not the last configuration and $q = 2^n -1$, then, the induction hypothesis ensures that $\rho((s_\ell, h(\ell)))$ triggers the transition rules:
			\[
			\head \sqcap \state_{Q(\ell)} \sqcap \symb_{t_\ell(h(\ell))} \sqcap \ascOK \sqcap \predbit_i^{b_i(h(\ell))} \sqsubseteq \exists \transition_{Q(\ell + 1), t_{\ell + 1}(h(\ell + 1)), m, i, b_i(h(\ell))}
			\]
			for every $0 \leq i < n$, and for one $m \in \{ \triangleright, \triangleleft \}$.
			Since $\Jmc$ is a model, the element $\rho_{(s_\ell, p}((s_\ell, h(k)))$ must have $n$ corresponding successors in $\Jmc$.
			If those successors $e_1, \dots, e_n$ are not the very same element, then it can be verified that one can obtain a model $\Jmc' \subset \Jmc$ restricted to $\rho_{(s_\ell, q)}(\Delta_{(s_\ell, q)}) \cup \{ e_1, \dots, e_n \}$, that does not satisfies $\goal$.
			The minimality of $\Jmc$ then forces $\Jmc$ to be $\Jmc'$, which contradicts $\Jmc$ satisfying $\goal$.
			Therefore, $e_1 = e_2 = \dots = e_n$, and we use this single successor to define $\rho_{(s_{\ell + 1}, h(\ell + 1))}$ as the extension of $\rho_{(s_\ell, q)}$ that additionally maps $(s_{\ell + 1, h(\ell + 1)})$ to $e_1$.
			Verifying that $\rho_{(s_{\ell + 1}, h(\ell + 1))}$ now satisfies Claim~\ref{claim:invariant} can be argued as above, meaning that if the interpretation of missing roles does not behave exactly as in $\Imc\vert_{\Delta_{(s_{\ell + 1}, h(\ell + 1))}}$, then one can extract again some $\Jmc' \subseteq \Jmc$ that does not satisfy $\Jmc$. 
			\item If $q < 2^n -1$, then the argument is similar to the previous case, but this time involves a combination of generated roles $\generate_{m, q, i, b_i, o}$ for $0 \leq i < n$ and fixed state $q$, direction $m$, and `instruction' $o$ ($\copysymbol$ or $\replaceby_{c}$).
		\end{itemize}
		
		As every intermediate $\rho_e$ extends the previous ones, we obtain the desired $\rho$ by setting:
		\[
		\rho := \bigcup_{e \in \Delta^\Imc} \rho_e.
		\]
		Note that if the sequence $s_1, \dots, s_k, \dots$ is infinite, then $\rho$ is guaranteed to be a homomorphism are no fact can be `left-behind' by Claim~\ref{claim:invariant} for more than 1 step $s_k$.
		If $s_1, \dots, s_K$ is finite, then one can verify that Claim~\ref{claim:invariant} applied to $(s_K, 2^n - 1)$ covers all facts from $\Imc$.
		In both cases, $\rho : \Imc \rightarrow \Jmc$ is a injective homomorphism as desired.
		 
	\end{toappendix}

 \subsection{Tuple Generating Dependencies}

 $\EL$ without $\top$ can be seen as a small fragment of
 \emph{Tuple Generating Dependencies (TGDs)}, which are prominent in
 the Database Theory literature~(see, e.g., \cite{FAGIN200589,CALI201287}. Thus our lower bounds carry
 over to minimal model reasoning in TGDs, for problems like
 \emph{brave entailment} of an atom, or for checking non-emptiness of
 a relation in some minimal model of a database and input
 TGDs. Specifically, an $\EL$ TBox without $\top$ can be converted
 into the so-called \emph{guarded TGDs} with relations of arity at
 most 2. Minimal model reasoning over TGDs has been explored
 in~\cite{10.1145/3034786.3034794}, where an undecidability result was
 achieved using relations of arities up to 4 in the context of  the \emph{stable model semantics}.
 Our Theorem~\ref{theorem:el-is-undecidable-without-top} implies that
 checking the existence of a stable model for \emph{normal guarded} TGDs is
 undecidable already for theories of the form
 $\Sigma\cup \{ \neg g(\vec{t})\rightarrow \bot\}$, where $\Sigma$ has
 negation-free guarded TGDs with relations of arity $\leq 2$, and
 $g(\vec{t})$ is a ground atom. 
 Similarly, our $\Sigma^{P}_2$ lower bound in data complexity can be used to improve the $\Pi^{P}_2$ lower bound in~\cite{10.1145/3034786.3034794}, that relies on predicates of arity $> 2$, for weakly acyclic TGDs with stable negation.

\subsection{$\EL$ without the UNA}
Finally, we make an interesting observation about the role of the UNA in the presented results. 
Our hardness proofs all rely on the UNA, and use more than one individuals that must be interpreted as different objects in the domain. 
This is no coincidence: if we drop the UNA, \mms in $\ELIO$ and $\EL$ is not only decidable, it is even tractable.

    \begin{theoremrep}
		$\mms$ in $\ELIO$ is $\PTime$-complete; the lower bound holds already for $\EL$.
	\end{theoremrep}
\noindent 
Even without the UNA, some concepts may not be satisfiable in a minimal model, but an $\ELIO$ KB now has a `representative' minimal model with just one element.
This representative model can be computed in polynomial time via a fixpoint computation akin to building a minimal model of a propositional definite Horn logic program. 
The lower bound also follows easily from the latter setting.

	\begin{proof}
		We first establish $\PTime$-membership.
		Consider the following algorithm that takes as input an $\ELIO$ KB $\kb := (\tbox, \abox)$ and a concept $\cstyle{goal}$ of interest. 
		Consider the initial interpretation $\I_0$ whose domain is $e$ and that interpret every individual name $\istyle{a}$, concept name $\cstyle{A}$ and role name $\rstyle{p}$ as follows:
		\begin{align*}
			\istyle{a}^{\Imc_0} := ~ & e
			\\
			\cstyle{A}^{\Imc_0} := ~ &
			\{ e \mid \text{there exists } \cstyle{A}(\istyle{a}) \in \abox \}
			\\
			\rstyle{p}^{\Imc_0} := ~ &
			\{ (e, e) \mid \text{there exists } \cstyle{p}(\istyle{a}, \istyle{b}) \in \abox \}.
		\end{align*}
		Every concept inclusions in $\tbox$ is initially consider `unmarked'.
		Initialize the current interpretation $\Imc$ by setting $\Imc := \Imc_0$.
		Now, given a current interpretation $\Imc$, while there exists a concept inclusion $\cstyle{C} \sqsubseteq \cstyle{D} \in \tbox$ that is unmarked and such that $e \in (\cstyle{C} \sqcap \lnot \cstyle{D})^\Imc$, then: extend $\Imc$ by adding $e$ (resp.\ $(e, e)$) to $\cstyle{A}^\Imc$ (resp.\ $\rstyle{p}^\Imc$) for each concept name $\cstyle{A}$ (resp.\ role name $\rstyle{p}$) occurring in $\cstyle{D}$, and `mark' the concept inclusion $\cstyle{C} \sqsubseteq\cstyle{D}$.
		When there is no more concept inclusion $\cstyle{C} \sqsubseteq \cstyle{D} \in \tbox$ that is unmarked and such that $e \in (\cstyle{C} \sqcap \lnot \cstyle{D})^\Imc$, then let $\Imc_f$ be the last current interpretation.
		If $e \in \cstyle{goal}^{\Imc_f}$, accepts, otherwise rejects.
		
		Note that the procedure is fully deterministic and run in the worst case in quadratic time w.r.t.\ the input $\kb$.
		We claim that it accepts iff $\cstyle{goal}$ is satisfiable in a minimal model of $\kb$, which will conclude the proof of the upper bound.
		
		\noindent
		\emph{Soundness:} 
		Assume the procedure accepts.
		Then it is trivial that $\Imc_f$ is a minimal model of $\kb$ that satisfies $\cstyle{goal}$.
		
		\noindent
		\emph{Completeness:} 
		Consider a minimal model $\Imc$ satisfying $\cstyle{goal}$ and assume by contradiction that $\Imc_f$ does not satisfy $\cstyle{goal}$.
		Drop the interpretation of all predicates that are satisfied in $\Imc$ but not in $\Imc_f$.
		Doing so, we obtain an interpretation $\Jmc \subsetneq \Imc$.
		We prove that $\Jmc$ is a model, which contradicts the minimality of $\Imc$.
		Consider a concept inclusion $\cstyle{C} \sqsubseteq \cstyle{D}$ and $d \in \cstyle{C}^{\Jmc}$.
		In particular, $d \in \cstyle{C}^{\Imc}$ thus $d \in \cstyle{D}^\Imc$ as $\Imc$ is a model of $\kb$.
		Furthermore, by definition of $\Jmc$, every predicate it satisfies is also satisfied in $\Imc_f$, and thus $e \in \cstyle{C}^{\Imc_f}$.
		Therefore, by construction of $\Imc_f$, we obtain $e \in \cstyle{D}^{\Imc_f}$.
		Thus all predicates occurring in $\cstyle{D}$ are satisfied in $\Imc_f$ and thus have been kept when constructing $\Jmc$ from $\Imc$.
		Recalling that $d \in \cstyle{D}^\Imc$, it follows that $d \in \cstyle{D}^{\Jmc}$.
		
		We now proceed to prove $\PTime$-hardness and reduce the well-known $\PTime$-complete problem of deciding whether an input Boolean circuit $\Cmc$ with $n$ input gates accepts a Boolean input $b_1, \dots, b_n$.
		For each input gate $G_k$ of $\Cmc$, we add the assertion $G_k(\istyle{a})$ if $b_k = 1$ and add assertion $\overline{G_k}(\istyle{a})$ otherwise.
		For each conjunctive gate $G$ of $\Cmc$ that takes as input gates $G'$ and $G''$, we add axioms:
		\[
			G' \sqcap G'' \sqsubseteq G
			\qquad
			\overline{G'} \sqsubseteq \overline{G}
			\qquad
			\overline{G''} \sqsubseteq \overline{G}.
		\]
		For each disjunctive gate $G$ of $\Cmc$ that takes as input gates $G'$ and $G''$, we add axioms:
		\[
		G' \sqsubseteq G
		\qquad
		G'' \sqsubseteq G
		\qquad
		\overline{G'} \sqcap \overline{G''} \sqsubseteq \overline{G}.
		\]
		For each negation gate $G$ of $\Cmc$ that takes as input $G'$, we add axioms:
		
		\[
		\overline{G'} \sqsubseteq G
		\qquad
		G' \sqsubseteq \overline{G}.
		\]
		Let $G_f$ be the output gate of $\Cmc$.
		It is straightforward to verify that the concept $G_f$ is satisfiable in a minimal model of the above KB iff $\Cmc$ outputs true on input $b_1, \dots, b_n$.
	\end{proof}

	\section{Conclusion}

We have explored the challenges of reasoning with minimal models in description logics, and shown that enforcing minimality across all predicates leads to undecidability even in the lightweight $\EL$. This directly implies that minimal model reasoning in very restricted fragments of guarded TGDs with  predicate arities $\leq 2$ is also undecidable. 
Strong and weak acyclicity conditions allowed us to regain decidability and establish tight bounds on combined and data complexity in $\EL$, $\ELIO$, and even a fragment of $\ALCIO$.
Some of these bounds are inherited for the recently studied setting of pointwise circumscription, providing further evidence that local, pointwise minimization is about the best we can do if we are interested in minimal models in DLs. 
It remains to be explored whether acyclicity conditions and pointwise minimization might also be useful in the richer setting of TGDs. 
One of the most intriguing avenues left open for further investigation is DL-Lite: we know that minimization is almost for free in $\dllitecore$, but makes reasoning $\ExpSpace$-hard for its extension  $\dllitehorn$. 
We hope that this variant and even more expressive extensions like $\dllitebool$ may be decidable, and plan to look for tight matching complexity bounds.

	\section*{Acknowledgments}
	
	This work was partially supported by the Austrian Science Fund (FWF) projects PIN8884924, P30873 and
	10.55776/COE12.
	\newline\noindent
	The authors acknowledge the financial support by the Federal Ministry of Research, Technology and Space of Germany and by Sächsische Staatsministerium für Wissenschaft, Kultur und Tourismus in the programme Center of Excellence for AI-research ``Center for Scalable Data Analytics and Artificial Intelligence Dresden/Leipzig'', project identification number: ScaDS.AI

	\appendix

	\bibliographystyle{kr}
	\bibliography{main}

\end{document}

%% file: nexpHardnessProof.tex
        
        \begin{proof}
        We provide a reduction from the exponential torus $T_{2^n\times 2^n}$ tiling problem \cite{Tobies99}, where $n\in \mathbb{N}$. Let $\oplus_n$ denote the sum modulo $n$. An instance of the exponential torus tiling problem is a triple $P=(T,H,V)$, where $T$ is a set of tiles and $H$ and $V$ are the horizontal and vertical conditions, respectively. An \emph{initial condition} is a tile $t_0\in T$. 
        A map $\tau\colon \{0, \cdots, 2^n-1\}\times
         \{0,\cdots, 2^n-1\}\rightarrow T$ is a solution for $P$ given an initial condition $c$ if the following conditions are satisfied for all $i,j<2^n-1$:
         \begin{itemize}
         \item $(\tau(i,j), \tau(i\oplus_{2^n} 1, j))\in H$, and
         \item $(\tau(i,j), \tau(i,j\oplus_{2^n} 1))\in V$, and
         \item $\tau(0,0)=t_0$.
         \end{itemize}

            We construct a KB $\Kmc$ in strongly-acyclic $\mathcal{EL}$ and a concept $C_0$ and reduce checking the existence of a solution to $P$ to satisfiability of $C_0$ w.r.t $\Kmc$ under the minimal model semantics. 

            Let us construct $\Kmc$. Relevantly, we are going to do it without using $\top$ on the left-hand side of inequalities. We still use $\top$ in scope of quantifiers occurring on the right-hand side. The latter occurrences can be replaced by a `dummy' concept name without affecting the proof. 

            As in Example \ref{binarytree}. First, we need to ensure that we can produce a tree of depth $2n$ with $2^{2n}$. Each leaf of such tree will correspond to a pair $(x,y)\in \{0,\cdots, 2^n-1\}\times \{0, \cdots, 2^n-1\}$. 
       We craft a \emph{goal concept} that is satisfied in a minimal model if the tree to have $2^{2n}$ leaves. 
            \begin{align}
                &\cn{Root}(a)\\
                &\cn{Root}\sqsubseteq \cn{L_0}\label{treegen1}\\
                &\cn{L_i}\sqsubseteq \cn{\exists r}_i. \cn{L}_{i+1}\sqcap \cn{\exists l_i. L}_{i+1}\text{ for all $i< 2n$}\label{treegen2}
            \end{align}
        Clearly, since we are in $\mathcal{EL}$, the axioms above are not enough to ensure that we have exactly exponentially many leaves (at level $2n)$. 
We generate another tree from the $\cn{L}_{2n}$-labelled leaves.
We require that at each level of this second tree, each nodes `decides' whether it produces a right successor or a left successor. We simulate disjunction with qualified existentials on the left and predicate minimization.
        \begin{align}
            \cn{Left}(o) &\qquad \cn{Right}(o')\\
            \cn{L}_{2n}&\sqsubseteq \cn{L}_{0}'\label{revtreeone}\\
            \cn{L}_j'&\sqsubseteq \exists \cn{pick}\\
            \cn{L}_j'\sqcap \exists \cn{pick}. \cn{Left}&\sqsubseteq \exists \cn{l}'_j. \cn{L}_{j+1,l}'\\
            \cn{L}_j'\sqcap \exists \cn{pick}. \cn{Right}&\sqsubseteq \exists \cn{r}'_j. \cn{L}_{j+1,r}'\\
            \cn{L}_{j+1,l}' \sqcap \cn{L}_{j+1,r}'&\sqsubseteq \cn{L}_{j+1}' \text{ for all }0\leq j\leq 2n-1\\
            \cn{L}_{2n,l}'\sqcap \cn{L}_{2n,l}'&\sqsubseteq \cn{Tree}\label{tree}
        \end{align}
Using a minimality argument, we will argue that $\cn{Tree}$ is satisfied if only if it is true at the root of a tree with exponentially many leaves (all labeled with $\cn{L}_{2n}$). 

        We now assign to each leaf (domain elements satisfying $\cn{L}_{2n}$) a binary vector encoding its $(x,y)$ coordinates in the square $2^{n}\times 2^{n}$. 	
        \begin{align}
            \cn{One}(b)&\qquad \cn{Zero}(c)\\
            \cn{L}_{2n} &\sqsubseteq \exists \cn{p}_i^x \sqcap \exists \cn{p}_i^y \quad \text{for all $i\leq n-1$}\label{cpickone}\\
            \exists \cn{p}_i^x. \cn{One}\sqsubseteq \cn{X}_i \qquad&\qquad
            \cn{\exists p}_i^x. \cn{Zero}\sqsubseteq \cn{\bar{X}}_i\\
            \cn{\exists p}_i^y. \cn{One}\sqsubseteq \cn{Y}_i
            \qquad&\qquad
            \cn{\exists p}_i^y. \cn{Zero}\sqsubseteq \cn{\bar{Y}}_i\\
            \cn{X}_i \sqsubseteq \cn{Pick}_{i,x}\qquad &\qquad
            \cn{\bar{X}}_{i}\sqsubseteq \cn{Pick}_{i,x}\quad \\
            \cn{Y}_i \sqsubseteq \cn{Pick}_{i,y}\qquad &\qquad
            \cn{\bar{Y}}_{i}\sqsubseteq \cn{Pick}_{i,y}\quad \\        
            \bigsqcap_{i\leq n} \cn{Pick}_{i,x}\sqcap \bigsqcap_{i\leq n} \cn{Pick}_{j,y} &\sqsubseteq \cn{P}\label{cpickfinale}
            \end{align}
The satisfaction of the concept $\cn{P}$ at each leaf of the tree ensures that each leaf picked a set of coordinates. We will require that $\cn{P}$ is satisfied at each of leaf together with another set of concept ensuring a family of properties. 
To to do, we will use the tree structure to ensure that a concept $\cn{C}$ is satisfied at the root of the tree if and only if \emph{all} the leaves satisfy certain concepts, including $\cn{P}$.

        We now embed a grid using the $2^{2n}$ leaves available. 
        First we connect all of them via two roles $\cn{v}$ and $\cn{h}$, 
        standing for vertical and horizontal successor.
        \begin{align}\label{hvsucc}
            \cn{L}_{2n}\sqsubseteq \exists \cn{h}. \cn{H}\sqcap \exists \cn{v}. \cn{V}\\
            \cn{L}_{2n} \sqcap \cn{H} \sqcap \cn{V} \sqsubseteq \cn{Good}
        \end{align}
        Observe that all the leaves can only send exactly one $V$ and one $H$. Thus  if all the leaves satisfy $\cn{Good}$, then they have exactly one $\cn{v}$-predecessor and exactly one $\cn{h}$-predecessor.
        

        We now ensure that each pair $(x,y)$ is connected to $(x+1,y)$ via $\cn{h}^-$ and to $(x,y+1)$ via $\cn{v}^-$. 
        To do so, we import a copy of the coordinates of one leave to its $h$ and $v$ predecessors and compare to the existing one. 
        Firs of all, given any pair $(x,y)$, moving vertically does not affect the $x$-coordinate and moving horizontally does not affect the $y$-coordinate. For all $i$ such that $0\leq i<n$:
        \begin{align}
            &\cn{\bar{Y}}_i\sqsubseteq \forall \cn{h}^-. \cn{\bar{Y}}_i^h\qquad \cn{Y}_i\sqsubseteq \forall \cn{h}^-. \cn{Y}_i^h\label{counting1}\\
            &\cn{\bar{X}}_i\sqsubseteq \forall \cn{v}^-. \cn{\bar{X}}_i^v\qquad \cn{X}_i\sqsubseteq \forall \cn{v}^-.\cn{X}_i^v
            \end{align}
            
       For each pair $(x,y)$, we copy in the $h$ predecessor the bit encoding of $x+1$. For all $i$ such that $0\leq i<n$ we consider the following axioms:
            \begin{align}\label{counting2}
                &\cn{X}_0\sqcap \dots \sqcap \cn{X}_{i-1}\sqcap \cn{X}_i\sqsubseteq \forall \cn{h}^-. \cn{\bar{X}}_i^h\\
                &\cn{X}_0\sqcap \dots \sqcap \cn{X}_{i-1}\sqcap \cn{\bar{X}}_i \sqsubseteq \forall \cn{h^-. X}_i^h\\
            &(\cn{\bar{X}}_0\sqcup\cdots\sqcup \cn{\bar{X}}_{i-1})\sqcap \cn{X}_i\sqsubseteq \forall \cn{h^-}. \cn{X}_i^h\\
            &(\cn{\bar{X}}_0\sqcup\cdots\sqcup \cn{\bar{X}}_{i-1})\sqcap \cn{\bar{X}}_i \sqsubseteq \forall \cn{h^-. \bar{X}}_i^h
            \end{align}
            Similarly, for each $(x,y)$ we copy to the $v$ predecessor the bin encoding of $y$. For all $i$ such that $0\leq i<n$, we consider the following axioms
            \begin{align}
            &\cn{Y}_0\sqcap \dots \sqcap \cn{Y}_{i-1}\sqcap \cn{Y}_i\sqsubseteq \forall \cn{v^-. \bar{Y}}_i^v\\
            &\cn{Y}_0\sqcap \dots \sqcap \cn{Y}_{i-1}\sqcap \cn{\bar{Y}}_i \sqsubseteq \forall \cn{v^-. Y}_i^v\\
            &(\cn{\bar{Y}}_0\sqcup\cdots\sqcup \cn{\bar{Y}}_{i-1})\sqcap \cn{Y}_i\sqsubseteq \forall \cn{v^-. Y}_i^v\\
            &(\cn{\bar{Y}}_0\sqcup\cdots\sqcup \cn{\bar{Y}}_{i-1})\sqcap \cn{\bar{Y}}_i \sqsubseteq \forall \cn{v^-. \bar{Y}}_i^v\label{counting4}
            \end{align}
           When we reach $(x,y)$, with $x=2^n-1$ (resp $y=2^n-1$) we require that the $h$ predecessor is the pair $(0,y)$ (resp $(x,0)$).
            \begin{align}
                \cn{X}_1\sqcap \cdots \sqcap \cn{X}_{n}\sqsubseteq \forall \cn{h^-. \bar{X}}_{i}^h\qquad 0\leq i<n\\
                \cn{Y}_1\sqcap \cdots \sqcap \cn{Y}_{n}\sqsubseteq \forall \cn{v^-. \bar{Y}}_{i}^v\qquad 0\leq i<n\label{countingfinal}
            \end{align}
            
        Observe that the axioms above use a richer syntax however, they can be rewritten as equivalent $\mathcal{EL}$ axioms. First, we simply the left-hand side of the axioms by \begin{enumerate*}[(1)]
        \item introducing a fresh concept $A$ for all $\bigsqcup_{i\in I} B$ (where $I$ is an index set) in the left hand side of each axioms
        \item replacing $\bigsqcup_{i\in I} B$ with $A$ and adding an axiom $\bigsqcup_{i\in I} B \sqsubseteq A$.
        \end{enumerate*}
        Then, we rewrite  \begin{enumerate*} [(1)]   
            \item each $\bigsqcup_{i\in I} B \sqsubseteq A_\sqcup$ as the equivalent set of axioms $\{B_i\sqsubseteq A\vert\,i\in I\}$, and
            \item  each $B\sqsubseteq \forall r^-. A$ as the equivalent axiom $\exists r. B\sqsubseteq A$. 
        \end{enumerate*}
        We now check the coordinates to ensure that we are actually connecting leaves in a consistent way (w.r.t coordinates) . To do so, we compare the received coordinates (via $h$ and $v$).
            
        \begin{align}
            \cn{X_i\sqcap {X_i}}^h \sqsubseteq \cn{Ok}_{i,x}^h \qquad&\qquad  \cn{X_i\sqcap {X_i}}^v \sqsubseteq \cn{Ok}_{i,x}^v\quad 0\leq i<n\label{bycomponentx1}\\
            \cn{\bar{X}}_i\sqcap \cn{\bar{X}}_i^h \sqsubseteq \cn{Ok}_{i,x}^h\qquad&\qquad  \cn{\bar{X}}_i\sqcap \cn{\bar{X}}_i^v \sqsubseteq \cn{Ok}_{i,x}^v\label{bycomponentx2}\\
            \bigsqcap_{i\leq n}\cn{Ok}_{i,x}^h\sqsubseteq \cn{Ok}_x^h\qquad&\qquad  \bigsqcap_{i\leq n}\cn{Ok}_{i,x}^v\sqsubseteq \cn{Ok}_x^v\label{xfine}\\
            \cn{Y}_i\sqcap {\cn{Y}_i}^h \sqsubseteq \cn{Ok}_{i,y}^h \qquad &\qquad  \cn{Y}_i\sqcap {\cn{Y}_i}^v \sqsubseteq \cn{Ok}_{i,y}^v\label{bycomponenty1}\\
            \bar{\cn{Y}_i}\sqcap \bar{\cn{Y}_i}^h \sqsubseteq \cn{Ok}_{i,y}^h \qquad&\qquad  \bar{\cn{Y}_i}\sqcap \bar{\cn{Y}_i}^v \sqsubseteq \cn{Ok}_{i,y}^v\label{bycomponenty2}\\
            \bigsqcap_{i\leq n}\cn{Ok}_{i,y}^h\sqsubseteq \cn{Ok}_y^h\qquad&\qquad  \bigsqcap_{i\leq n}\cn{Ok}_{i,y}^v\sqsubseteq \cn{Ok}_y^v\label{yvfine}\\
            \cn{Ok}_x^h\sqcap\cn{Ok}_y^h\sqsubseteq \cn{Ok}_h\qquad&\qquad  \cn{Ok}_x^v\sqcap \cn{Ok}_y^v\sqsubseteq \cn{Ok}_v\label{hvfine}\\
             \cn{Ok}_h\sqcap &\cn{Ok}_v\sqsubseteq \cn{Ok}\label{comparisonfinal}
            \end{align}
  If $\cn{Ok}$ is true at each leaf, we have that the coordinates are consistent with the role $h$ and $v$.
           
We now assign a tile color to each leaf of the tree and ensure that vertical and horizontal conditions are satisfied. 

            \begin{align}
            \cn{A}_t'(t)&\qquad \text{for all }t\in T\label{tiling0}\\
            \cn{A}_t'  &\sqsubseteq \cn{TileColor}\qquad \text{for all }t\in T\label{tiling1}\\
            \cn{L}_{2n}\sqcap (\bigsqcup_{i=0}^{n-1} \cn{X}_i \sqcup \cn{Y}_i)&\sqsubseteq \exists \cn{tile.} \label{tiling2}\\
            \exists \cn{tile. A}_t' &\sqsubseteq \cn{A}_t\label{tiling3}\\
            \exists \cn{tile.} \cn{TileColor} &\sqsubseteq \cn{Tiled}\label{tiling4}\\
            \bigsqcap_{i=0}^{n-1} (\cn{\bar{X}}_i\sqcap \cn{\bar{Y}}_i) &\sqsubseteq A_{t_0}\sqcap \cn{Tiled}\label{initialcond}
        \end{align} 
If $\cn{Tiled}$ is true at all leaves, each of them picked a tile color. We now ensure that choice respects the horizontal and vertical conditions. 

        For all $t\in T$ we introduce two concepts $\cn{HNeigh}_t$ and $\cn{VNeigh}_t$.
        \begin{align}
            \cn{A}_{t'}&\sqsubseteq \cn{HNeigh}_{t'}\text{ for all $t'$, such that $(t,t')\in H$}\label{tiling5}\\
            \cn{HNeigh}_{t} \sqcap \exists \cn{h. A}_t &\sqsubseteq \cn{Hok}\label{tiling6}\\
            \cn{VNeigh}_{t} \sqcap \exists \cn{v. A}_t &\sqsubseteq \cn{Vok}\label{tiling7}\\
            \cn{Hok\sqcap Vok} &\sqsubseteq \cn{HVok}\label{tiling8}
        \end{align}
        If $\cn{HVok}$ is true at each leaf, then the tiling is consistent. 
       We now have a collection of concepts that ensure certain properties to be satisfied if the concepts are true at all leaves. We collect them into one big conjunction. 
       \begin{align}
       \cn{P}\sqcap \cn{Good} \sqcap\cn{Ok}\sqcap \cn{Tiled} \sqcap \cn{HVok}\sqsubseteq \cn{G}\label{subgoals}
       \end{align}
       We propagate at each layer $i$ of the tree a concept $\cn{G}_{i}$ as follows. 
       \begin{align}
       \cn{G}&\sqsubseteq \cn{G}_{2n}\label{goaltoroot1}\\
       \cn{\exists l}_j. \cn{G}_{j+1} \sqcap \cn{\exists r}_j. \cn{G}_{j+1}&\sqsubseteq \cn{G}_{j+1} \quad 0\leq j <2n\label{goaltoroot2}\\
       \cn{G_0}&\sqsubseteq \cn{LeafGrid}\label{goaltoroot3}
       \end{align}
        
        Let $\Kmc$ be the constructed KB. It is straightforward to see that $\Kmc$ is in strongly-acyclic $\mathcal{EL}$.
        We show the following.
        \begin{lemma}\label{claim:nexp}
            $P$ has a solution if and only if $\cn{Root} \sqcap \cn{Tree} \sqcap  \cn{LeafGrid}$ is satisfiable in a minimal model of $\Kmc$.
        \end{lemma}
        \begin{proof}
        \emph{(Only if).} Assume that $P$ has a solution $\tau$. We use $\tau$ to define a minimal model $\I$ satisfying the concept $\cn{Root} \sqcap \cn{Tree} \sqcap  \cn{LeafGrid}$.
        We define the domain of $\Imc$ as the set:
        \begin{align*}
        \Delta^{\Imc}=&\{(x,y) \vert 0\leq x,y\leq 2^n-1\} \cup \{r, d_l, d_r, b_0, b_1\}
        \cup
         \bigcup_{1\leq l < 2n}\{n_{l}^{k_l}\vert 0\leq k_l \leq 2^l-1\} 	\cup \{c_t\vert\,t\in T\} 
        \end{align*}
        Intuitively, we use the elements $n_{k}^{k_l}$ as nodes of the tree rooted in $r$ with leaves $(x,y)$.
        
        Observe that we can map $\mathbb{N}\times \mathbb{N}$ into $\mathbb{N}$ bijectively with the map $f((x,y))=\frac{(x+y)(x+y+1)}{2}+y$. We shall use this map in the following. 
        The interpretation function is defined as follows.
        \begin{align*}
        a^{\Imc}=r\qquad o^{\Imc}=d_l\qquad
        {o'}^{\Imc}=d_r\qquad b^{\Imc}=b_1\qquad c^{\Imc}=b_0\qquad t^{\Imc}=c_t\qquad\text{ for all $t\in T$}
        \end{align*}
        \begin{align*}
        &\cn{Root}^{\Imc}=\cn{LeafGrid}^{\Imc}=\cn{Tree}^{\Imc}=\{r\}\\
        &\cn{Right}^{\Imc}=\{d_r\}\qquad \cn{Left}^{\Imc}=\{d_l\}\qquad \cn{One}^{\Imc}=\{b_1\}\qquad \cn{Zero}^{\Imc}=\{b_0\}\\
        &\cn{l}_0^{\Imc}=\{(r, n_1^{1})\}\qquad \cn{r}_0^{\Imc}=\{(r,n_1^2)\}\\
        &\cn{l}_i^{\Imc}=\bigcup_{k=0}^{2^{i}-1}\{(n_i^k, n_{i+1}^{k'})\vert k'=2k\} \text{ for all }0\leq i<2n-1\\
        &\cn{r}_i^{\Imc}=\bigcup_{k=0}^{2^{i}-1}\{(n_i^k, n_{i+1}^{k'})\vert k'=2k+1\}\text{ for all }0\leq i< 2n-1\\
        &{\cn{r}_{2n-1}^{\Imc}=\bigcup_{k=0}^{2^{2n-1}-1}\{(n_{2n-1}^k, (x,y))\vert f((x,y))=2k+1\}}\\
        &\cn{l}_{2n-1}^{\Imc}=\bigcup_{k=0}^{2^{2n-1}-1}\{(n_{2n-1}^k, (x,y))\vert\, 
        f((x,y))=2k\}\\
        &\cn{L}_i^{\Imc}=\{n_i^k\vert 0\leq k\leq 2^{i}-1\}\text{ for all }0\leq i< 2n\\
        &\cn{L}_{2n}^{\Imc}= \{(x,y)\vert 0\leq x,y,\leq 2^{n}-1\}\\  
        &{\cn{L}_j'}^{\Imc}=\cn{L}_{2n-j}^{\Imc}\text{ for all }0\leq j\leq 2n\\
        &{\cn{L}_{j,r}'}^{\Imc}={\cn{L}_{j,l}'}^{\Imc}\text{ for all }1\leq j\leq 2n\\
        &\cn{pick}^\Imc=\{((x,y),d_l)\vert\,f((x,y))\text{ is even}\}
       \cup \{((x,y),d_r)\vert \,f((x,y))\text{ is odd}\}
        \cup\{(n_j^k,d_l)\vert\,k\text{ is even}\}
        \cup \{(n_j^k,d_r)\vert\,k\text{ is odd}\}\\
        &{\cn{r}'_j}^{\Imc}= (r_{2n-(j+1)}^-)^{\Imc}\text{ for all }0\leq j\leq 2n\\
        &{\cn{l}'_j}^{\Imc}= (l_{2n-(j+1)}^-)^{\Imc}\text{ for all }0\leq j\leq 2n
        \end{align*}
        
        \begin{minipage}{0.5\textwidth}
        \begin{align*}
        &\cn{X_i}^{\Imc}=\{(x,y)\vert bit_i(x)=1\}\qquad 0\leq i<n\\
        &\cn{\bar{X}_i}^{\Imc}=\{(x,y)\vert bit_i(x)=0\} \qquad 0\leq i<n\\
        &\cn{Y_i}^{\Imc}=\{(x,y)\vert bit_i(y)=1\}\qquad 0\leq i<n\\
        &\cn{\bar{Y}_i}^{\Imc}=\{(x,y)\vert bit_i(y)=0\}\qquad 0\leq i<n\\
        &\cn{Pick_{i,x}}^{\Imc}=\{(x,y)\vert 0\leq x,y,\leq 2^{n}\}\qquad 0\leq i<n\\
        &\cn{P}^{\Imc}= \{(x,y)\vert 0\leq x,y,\leq 2^{n}\}\\
        &\cn{h}^{\Imc}=\{((x\oplus_{2^n}1,y), (x,y))\vert 0\leq x, y\leq 2^n\}\\
        &\cn{v}^\Imc=\{((x,y\oplus_{2^n}1),(x,y))\vert  0\leq x, y \leq 2^n\}\\
        &\cn{V}^{\Imc}=\cn{H}^{\Imc}= \{(x,y)\vert 0\leq x,y,\leq 2^{n}\} \\
        &(\cn{X}_i^h)^{\Imc}=(\cn{X}_i^v)^{\Imc}=\cn{X}_i^{\Imc}\qquad 0\leq i<n\\
        &\cn{Good}^{\Imc}=\{(x,y)\vert 0\leq x,y<2^n\}\\
        &{\cn{A}'_t}^{\Imc}=\{c_t\}\qquad \text{for all}\;t\in T\\
		&\cn{TileColor}^{\Imc}=\bigcup_{t\in T} {\cn{A}_t'}^{\Imc}\\
		&{\cn{A}_t}^{\Imc}=\{(x,y)\vert \tau((x,y))=t\}\qquad \text{for all}\;t\in T\\
		&\cn{tile}^{\Imc}=\{((x,y), c_t)\vert \tau((x,y))=t\}\\
		&\cn{Tiled}^{\Imc}=\{(x,y)\vert 0\leq x,y< 2^n-1\}\\
		&\cn{HNeight}_{t}^{\Imc}=\bigcup_{(t,t')\in H} 	\cn{A}_{t'}^{\Imc}
        \end{align*}
        \end{minipage}
        \begin{minipage}{0.5\textwidth}
        \begin{align*}
        &(\cn{pick}_i^x)^{\Imc}=\{((x,y), b_0)\vert\, bit_i(x)=0\}\qquad 0\leq i<n\\
        &(\cn{pick}_i^x)^{\Imc}=\{((x,y), b_1)\vert\, bit_i(x)=0\} \qquad 0\leq i<n\\
        &(\cn{pick}_i^y)^{\Imc}=\{((x,y), b_0)\vert\, bit_i(y)=0\}\qquad 0\leq i<n\\
        &(\cn{pick}_i^y)^{\Imc}=\{((x,y), b_1)\vert\, bit_i(y)=0\}\qquad 0\leq i<n\\
        &\cn{Pick_{i,y}}^{\Imc}=\{(x,y)\vert 0\leq x,y,\leq 2^{n}\}\qquad 0\leq i<n\\
        &(\cn{Y}_i^h)^{\Imc}=(\cn{Y}_i^v)^{\Imc}=\cn{Y}_i^{\Imc}\qquad 0\leq i<n\\
        &(\cn{\bar{X}}_i^h)^{\Imc}=(\cn{\bar{X}}_i^v)^{\Imc}=\cn{\bar{X}}_i^{\Imc}\qquad 0\leq i<n\\
        &(\cn{\bar{Y}}_i^h)^{\Imc}=(\cn{\bar{Y}}_i^v)^{\Imc}=\cn{\bar{Y}}_i^{\Imc}\qquad 0\leq i<n\\
        &(\cn{Ok}_{i,x}^h)^{\Imc}=(\cn{Ok}_{i,x}^v)^{\Imc}= (\cn{Ok}_{x}^h)^{\Imc}= \{(x,y)\vert 0\leq x,y\leq 2^n-1\}\\
        &(\cn{Ok}_{i,y}^h)^{\Imc}=(\cn{Ok}_{i,y}^v)^{\Imc}= (\cn{Ok}_{y}^h)^{\Imc}= \{(x,y)\vert 0\leq x,y\leq 2^n-1\}\\
&\cn{Ok}_h^\Imc=\cn{Ok}_v^{\Imc}=\cn{Ok}^{\Imc}=      \{(x,y)\vert 0\leq x,y\leq 2^n-1\}\\
&\cn{VNeight}_{t}^{\Imc}=\bigcup_{(t,t')\in V} \cn{A}_{t'}^{\Imc}\\
&\cn{Hok}^{\Imc}=\cn{Vok}^{\Imc}=\cn{HVok}^{\Imc}=\{(x,y)\vert\,0\leq x,y<2^n-1\}\\
&\cn{G}^{\Imc}=\cn{G}_{2n}^{\Imc}=\{(x,y)\vert 0\leq x,y<2^n-1\}\\
&\cn{G}_{j}^{\Imc}= \cn{L}_j^\Imc\qquad 0\leq j<2n\\
&\cn{G_0}^{\Imc}=\cn{LeafGrid}^{\Imc}=\cn{Root}^{\Imc}
        \end{align*}
        \end{minipage}

The constructed interpretation is a model of $\Kmc$, the latter can be checked axiom by axiom.
To argue minimality, it is sufficient to observe the following.
\begin{itemize}
\item Each bit encoding of the pairs $(x,y)$ is uniquely determined by $x$ and $y$, furthermore each $\cn{pick}_{i,x}$ and $\cn{pick}_{i,y}$ is functional. 
Hence the concepts $\cn{X}_i$, $\cn{\bar{X}}_i$, $\cn{Y}_i$ and $\cn{\bar{Y}}_i$, with $0\leq i \leq n-1$, cannot be further minimized. For the same reason , $\cn{X}_i^h$, $\cn{\bar{X}}_i^h$, $\cn{Y}_i^h$, $\cn{\bar{Y}}_i^h$, $\cn{X}_i^v$, $\cn{\bar{X}}_i^v$, $\cn{Y}_i^v$ and $\cn{\bar{Y}}_i^v$ have minimal extension too. Observe that each $(x,y)$ has a unique $h$ successor that is the element $(x',y')$ such that $x'\oplus_{2^n} 1=x$ and $y'=y$. Thus, the concepts $\cn{Ok}_{i,\star}^v, \cn{Ok}_{i,\star}^h, \cn{Ok}_\star^v$, $\cn{Ok}_\star^v$, $\cn{Ok}_v$, $\cn{Ok}_h$ and $\cn{Ok}$, with $\star=x,y$, are minimal too. 
Since all pairs $(x,y)$ have a complete bit encoding of their coordinates, $\cn{Pick}_{i,x}$, $\cn{Pick}_{i,y}$ (for all $0\leq i\leq n-1$ and $P$ have minimal extension in $\Imc$.
\item Each domain element in $\{r\}\cup \bigcup_{0\leq l\leq 2n-1}\{n_l^k\vert 0\leq k\leq 2^l\}$ has exactly one left successor and one right successor (via the roles $r_i$ and $l_i$). Thus, the extension of such roles is minimal. Since $r_i$ and $l_i$ are inverse functional, the roles $r'_j$ and $l'_j$ are minimal too.  
\item From the previous observation, all the concepts $L_j, L_j'$ with $0\leq j\leq 2n$ are minimal as well. 
\item The extensions of the roles $h$ and $v$ is also minimal, since each node has exactly one horizontal and one vertical successor. Therefore, $\cn{H}$ and $\cn{V}$ are minimal as well. 
\item Using a similar argument, it is easy to see that all the predicates related to the tile assignment have a minimal extension too. 
\end{itemize}

\emph{(If)} Conversely assume that $\Imc$ is a minimal model of $\Kmc$ such that $(\cn{Root}\sqcap \cn{LeafGrid})^\Imc\not =\emptyset$. First observe that $\cn{Root}^{\Imc}\supseteq (\cn{Root}\sqcap \cn{Tree}\sqcap \cn{LeafGrid})^\Imc$. Since $\cn{Root}$ only occurs in the ABox of $\Kmc$, we have that $\cn{Root}^{\Imc}=\{a^\Imc\}$. Thus $(\cn{Root}\sqcap \cn{Tree} \sqcap \cn{LeafGrid})^\Imc=\{a^{\Imc}\}$. 

Since $\Imc$ is minimal and $a^{\Imc}\in \cn{Tree}^{\Imc}$, from \eqref{tree} $a^{\Imc}\in (\cn{L}_{2n, l}' \sqcap \cn{L}_{2n, r}')^\Imc$. Indeed, if it is not the case, we can construct a smaller model $\J$, by minimizing the concept $\cn{Tree}$ at $a^{\Imc}$.

Observe that in the setting of strong acyclicity, we often use the previous argument to deduce that, for all inclusions $C\sqsubseteq D$, if $d\in D^{\Imc}$ and $\Imc$ is minimal, then $d\in C^{\I}$. 

From the minimality of $\Imc$ and axioms \eqref{revtreeone}-\eqref{tree}, since $a^{\Imc}\in {\cn{L}_{2n,l}'}^{\Imc}$ and $a^{\Imc}\in {\cn{L}_{2n,r}'}^{\Imc}$ it follows that there exists $d', d''\in \Delta^{\Imc}$ such that:
\begin{itemize}
\item $d' \in (\cn{L}_{2n-1}'\sqcap \cn{\exists pick. Left})^\Imc$ and $(d',a^{\Imc}) \in {\cn{l}_{2n}'}^{\Imc}$ and,
\item $d'' \in (\cn{L}_{2n-1}'\sqcap \exists \cn{pick. Left})^\Imc$ and $(d',a^{\Imc}) \in {{\cn{l}_{2n}'}}^{\Imc}$. 
\end{itemize} 
Since $\cn{Right}$ and $\cn{Left}$ never occur on the right-hand side of axioms, from the minimality of $\Imc$ we have that $\cn{Right}^{\Imc}=\{o^{\Imc}\}$ and $\cn{Left}^{\Imc}=\{{o'}^{\Imc}\}$. Hence, we have that $(d',o^{\Imc})\in \cn{pick}^{\Imc}$ and $(d'',{o'}^{\Imc})\in \cn{pick}^{\Imc}$. Since $\cn{pick}$ is minimized, if $d'=d''$ then we lose the minimality of $\Imc$. In such case, a smaller model $\Jmc$ can be obtained by $\Imc$ by setting $\cn{pick}^{\Jmc}=\cn{pick}^{\Imc}\setminus \{(d',o^{\Imc})\}$ and leaving the rest unmodified. 
Hence, $d'\not = d''$. The same argument applied for $a^{\Imc}$ can be applied at $d'$ and $d''$, showing that each of them has a \emph{left}-predecessor and a \emph{right}-predecessor. Formally the following holds. 
\begin{claim}
For each $0< j\leq 2n$, given any $d\in {\cn{L}_j'}^{\Imc}$ there exists two distinct $d',d''\in \Delta^{\Imc}$ such that $d',d''\in {\cn{L}_{j-1}'}^{\Imc}$, $(d',d)\in {\cn{r}_{j-1}'}^{\Imc}$ and $(d'',d)\in {\cn{l}_{j-1}'}^{\Imc}$.
\end{claim}
\begin{proof}
The proof follows the same argument used for $a^\Imc$.
\end{proof}
Furthermore, since $\cn{r}_i'$ and $\cn{l}_i'$ are minimized, each $d\in {L_i'}^{\Imc}$ can only provide a unique $\cn{r}_i'$-successor or a unique $\cn{l}_i'$-successor. 
Hence, there exists $2^{2n}$ distinct elements $d^i$, with $0\leq i\leq 2^{2n}-1$ in $\Delta^{\Imc}$ such that $d^j\in {\cn{L}_0'}^{\Imc}$. From the minimality of $\Imc$ and \eqref{revtreeone}, we have that each $d^i\in \cn{L}_{2n}^{\Imc}$. We now show that each $d^i$ is a leaf of an exponential tree generated by $\cn{Root}$. The latter follows from the next claim.

\begin{claim}
	For all $0<j\leq 2n$ and each $d\in \cn{L}_j^{\Imc}$, there exists $d',d''\in {\cn{L}_{j-1}}^{\Imc}$ such that $(d',d)\in \cn{r}_{j-1}^{\Imc}$ or $(d'',d)\in \cn{l}_{j-1}^{\Imc}$. 
	Furthermore $\vert \{d\vert d\in \cn{L}_j^{\Imc}\}\vert = 2\cdot \vert \{d'\vert d'\in \cn{L}_{j-1}^{\Imc}\}\vert$.
\end{claim}
\begin{proof}
Let $d\in\cn{L}_j^{\Imc}$. If there exists no $d' \in \cn{L}_{j-1}^\Imc$ such that $(d',d)\in \cn{r_{j-1}}^{\Imc}$ and $(d',d)\in \cn{l}_{j-1}^{\Imc}$, then we contradict the minimality of $\Imc$. Indeed, in such a case a smaller model can be produced by minimizing $\cn{L}_j$ at $d$. Thus, there exists $d'\in \cn{L}_{j-1}^{\Imc}$ such that $(d',d)\in \cn{r}_{j-1}^{\Imc}$ or $(d',d)\in \cn{r}_{j-1}^{\Imc}$. 
To show the second part of the claim, observe that for each $j$, each $d \in \cn{L}_j^{\Imc}$ can justify at most two occurrences of $\cn{L}_{j+1}^{\Imc}$: one via the role $\cn{r}_j$ and one via the role $\cn{l}_j^{\Imc}$. Hence, for each $j$, we have that $\vert \cn{L}_j^{\Imc}\vert \leq 2\cdot \vert \cn{L}_{j+1}^{\Imc}\vert$. Conversely, at each $j$, each node needs at least one predecessor via $\cn{r}_{j-1}$ or via $\cn{l}_{j-1}$. Since $\Imc$ is minimal, each node in $\cn{L}_{j}^{\Imc}$ can produce at most one $\cn{r}_{j-1}$ successor and at most one $\cn{l}_{j-1}$ successor. Thus, we have that $ 2 \vert \cn{L}_{j+1}^{\Imc}\vert \leq \vert \cn{L}_{j}^{\Imc}\vert$.  
\end{proof}

From the claim above, it follows that $\vert \cn{L}_{2n}^{\Imc}\vert = 2^n$.
Intuitively speaking, with the previous claims, we proved that in $\I$ there exists two binary trees, sharing the same leaves (the $2^{2n}$ distinct domain elements we denoted with $d^i$) with roots in $a^{\Imc}$. We now show the following claim. 

\begin{claim}
Since $a^{\Imc}\in \cn{LeafGrid}^{\Imc}$, then $d^i \in {\cn{G}}^{\Imc}$ for all $1\leq i\leq 2^n$.  
\end{claim}
\begin{proof}
Since $a^{\Imc}\in \cn{LeafGrid}^{\Imc}$ and $\Imc$ is minimal, from \eqref{goaltoroot3}, $a^{\Imc}\in \cn{G}_0^{\Imc}$. Observe that since $a^{\Imc} \in \cn{L}_0^{\Imc}$, from \eqref{treegen1} the previous claims, there exists $d,d'$ such that $(a^{\Imc}, d)\in \cn{r}_0^{\Imc}$ and $(a^{\Imc}, d')\in \cn{l}_0^{\Imc}$. Since $\Imc$ is minimal, then $d, d'\in G_1^{\Imc}$. Indeed, if $d\not \in G_1^{\Imc}$ (resp. $d''$), since $a^{\Imc}$ has only a unique $\cn{r}_0$ (resp. $\cn{l}_0$) successor (from the minimality of $\I$), a smaller model $\Jmc$ can be obtained from $\Imc$ by minimizing $G_0$ at $a^{\Imc}$.
By iterating the above argument for each $j$, we derive that ${d^i}\in \cn{G}^{\Imc}$, for all $0\leq i\leq 2^{2n}-1$.
\end{proof}
Since each $d^i \in \cn{G}^{\I}$, from \eqref{subgoals} and the minimality of $\I$, we have that $d^i\in (\cn{P\sqcap \cn{Good} \sqcap \cn{Ok} \cn{ Tiled} \sqcap \cn{HVok}})^{\I}$, for all $0\leq i\leq 2^{2n}-1$. 

We now argue that there exists a torus embedded in the set $\{d^i\in \cn{G}^{\I} \vert 0\leq i \leq 2^{2n}-1\}$. In the following we refer to the elements of such set as \emph{leaf elements}.

Since for each $0\leq i< 2^n$, $d^i\in \cn{Good}^{\I}$ and $\I$ is minimal, we have that $d^i\in {\cn{L}_{2n}\sqcap \cn{H} \sqcap \cn{V}}$. Hence there exists $d^j,d^k$ such that $(d^j,d^i)\in \cn{v}^{\Imc}$ and $(d^k,d^i)\in \cn{v}^{\I}$. Since the latter holds for each $d^i$ and each $d^j$ can justify at most one $h$-successor and at most one $v$ successor (since $\I$ is minimal), we have that $d_k$ and $d_j$ are unique. 

Since for each $0\leq k<2^n$ each $d^k\in \cn{P}^{\Imc}$, from the minimality of $\I$ and \eqref{cpickfinale}, it follows that $d^k\in (\bigsqcap_{i\leq n-1} \cn{Pick}_{i,x}\sqcap \bigsqcap_{i\leq n-1} \cn{Pick}_{j,y})^{\Imc}$. Otherwise, a smaller model could be obtained from $\I$ by minimizing $\cn{P}$ at some $d^k$. Hence, for each $0\leq i,j\leq n-1$, $d^k\in \cn{Pick}_{i,x}$ and $d^k\in \cn{Pick}_{j,y}$. Relying again on the minimality of $\I$, we can show the following claim.
\begin{claim}
For all $0\leq k<2^n$, for all $0\leq i<n$:
\begin{itemize}
\item either $d^k\in \cn{X}_i^{\I}$ or $d_k\in \cn{\bar{X}}_i^\Imc$, and 
\item either $d^k\in \cn{Y}_i^{\I}$ or $d_k\in \cn{\bar{Y}}_i^\Imc$.
\end{itemize}
\end{claim}
\begin{proof}
We prove the first item above, the second can be showed analogously. Consider the axioms \eqref{cpickone}-\eqref{cpickfinale}.
By a way of contradiction if $d^k\not \in \cn{X}_i^{\I}$ and $d_k\not \in \cn{\bar{X}}_i^\Imc$, we can obtain a smaller model form $\I$ by minimizing $\cn{Pick_{i,x}}$ at $d^k$. A contradiction. Thus $d^k \in \cn{X}_i^{\I}$ or $d_k\not \in \cn{\bar{X}}_i^\Imc$. Assume that $d_k\in \cn{X}_i^{\I}$ \emph{and} $d^k\in \cn{\bar{X}}_i^{\I}$. Since $\I$ is minimal, it follows that $d^k\in \exists \cn{pick}_i^x. \cn{One}^{\I}$ and $d^k\in \exists \cn{pick}_i^x. \cn{Zero}^{\I}$. However, each $d\in \cn{L}_{2n}^{\Imc}$ can minimally justify only one $\cn{pick}_i^x$. Thus, we contradict the minimality of $\I$: a smaller model can be obtained from $\I$ by minimizing one of the two occurrences of $\cn{pick}_i^x$. Hence we proved the thesis.
\end{proof}

From the claim above, we can associate to each $d^k$ a unique pair of coordinates $(x,y)$ coded in binary via the concepts $\cn{X}_i, \cn{\bar{X}}_i, \cn{Y}_i, \cn{\bar{Y}}_i$, with $0\leq i<n$. 

We now show that each $(x,y)$, with $0\leq x,y<2^n$ occurs exactly once among the $d^k$ elements and that the roles $h$ and $v$ associate to each $d^k$ with coordinate $(x,y)$ the elements corresponding to the pairs $(x', y')$ such that $x'\oplus_{2^n}1=x$ and $y'=y$, and $(x'',y'')$ such that $x=x''$ and $y''\oplus_{2^n}1 =y$. 

Observe that, for each $0\leq k<2^n$, $d^k \in \cn{Ok}^{\I}$. Therefore, $d^k\in \cn{Ok}_h^\I$ and $d^k\in \cn{Ok}_v^{\I}$ (using the minimality of $\I$ and \eqref{comparisonfinal}. 
Recall that for each $d^k$, there exists a unique $h$-successor $d^{q}$ and a unique $v$-successor $d^t$ in the set of the leaf elements $\cn{LeafSet}=\{d^k\in \cn{L}_{2n}^{\Imc}\vert 0\leq k\leq 2^{2n}-1\}$. 
\begin{claim}
For any $d_k$, given $d^j$ such that $(d^k,d^j)\in h^{\I}$, if $d^k$ correspond to the pair $(x,y)$ then $d^j$ corresponds to the pair $(x',y')$ such that $x=x'\oplus_{2^n} 1$ and $y=y'$. Similarly, for any $d_k$, given $d^j$ such that $(d^k,d^j)\in v^{\I}$, if $d^k$ correspond to the pair $(x,y)$ then $d^j$ corresponds to the pair $(x',y')$ such that $x=x'$ and $y=y'\oplus_{2^n} 1$.
\end{claim}
\begin{proof}
We show the first part of the claim, the second can be proved analogously. 

We observed that each $d^k \in \cn{Ok}_h^\Imc$. Thus, from the minimality of $\I$ and \eqref{hvfine}, $d^k\in {\cn{Ok}_x^h}^{\Imc}$ and $d^k\in {\cn{Ok}_y^h}^{\Imc}$. Hence, using again the minimality of $\I$ and \eqref{xfine}, $d^k\in {(\bigsqcap_{i<n} \cn{Ok_{i,x}^h})^{\Imc}}$ and  $d^k\in {(\bigsqcap_{i<n} \cn{Ok_{i,y}^h})^{\Imc}}$.
From \eqref{bycomponentx1}-\eqref{bycomponentx2} and the minimality of $\I$, since $d^k\in {(\bigsqcap_{i<n} \cn{Ok_{i,x}^h})^{\Imc}}$, then for all $i$, we have that 
for all $0\leq i<n$, either $d^k\in (\cn{X}_i \sqcap \cn{X}_i^h)^\I$ or $d^k(\cn{\bar{X}}_i \sqcap \cn{\bar{X}}_i^h)^\I$. Analogously, we derive that $d^k\in (\cn{Y}_i \sqcap \cn{Y}_i^h)^\I$ or $d^k(\cn{\bar{Y}}_i \sqcap \cn{\bar{Y}}_i^h)^\I$.

From \eqref{counting1}-\eqref{countingfinal}, the node $d^k$ with coordinates $(x,y)$ reads via the role $h$, the coordinates $(x',y')$ of $d^j$ (recall that $d^j$ is unique) and copies the pair $(x'',y'')$ using the concepts $\cn{X}_i^h$, $\cn{\bar{X}}_i^h$, $\cn{Y}_i^h$ and $\cn{\bar{Y}}_i^h$, with $i<n$ such that $x'=x''\oplus_{2^n}1$ and $y''=y'$. The thesis.

The same argument can be used to show that if $(d^k,d^j)\in v^{\Imc}$ then the coordinates $(x,y)$ of $d^k$ and the coordinates $(x',y')$ of $d^j$ are such that $x=x'$ and $y=y'\oplus_{2^{n}} 1$. 
\end{proof}
An immediate consequence of the above claim is that given any $d^k$, let $d^j$ and $d^q$ be its $h$-successor and $v$-successor, then $d^j\not =d^q$. Furthermore, each $(x,y)$ occurs exactly once.

We can start embedding an exponential grid in the nodes $\cn{LeafSet}=\{d^j\vert \text{ with }0\leq j\leq 2^{2n}-1\}$. Let $\pi\colon \cn{LeafSet}\rightarrow \{0,\dots, 2^n-1\}\times \{0,\dots, 2^n-1\}$ such that $\pi(d^i)=(x,y)$ if in $d^i$ the concepts $\cn{X}_i, \cn{\bar{X}}_i, \cn{Y_i}$ and $\cn{\bar{Y_i}}$ encode in binary the two coordinates $x$ and $y$. The above mapping is a bijection. 

We now move to show that the leaves are properly tiled. 

Since each $d^j\in \cn{Tiled}^{\Imc}$, using a similar argument to the one used to show that each $d^j$ corresponds to a unique pair of coordinates $(x,y)$ coded in binary, one can show that each $d^i\in\cn{ A}_t^{\Imc}$ for some $t\in T$ and such $t$ is unique. If $d^j$ encodes $(0,0)$, i.e. $\pi(d^j)=(0,0)$, since $\I$ is a model, from \eqref{initialcond}, $d^j\in \cn{A}_{t_0}^{\Imc}$ and since $\I$ is minimal, there is no other $t\in t$ such that $d^j\in \cn{A}_t^\Imc$. If $\pi(d^j)\not = (0,0)$, since $d^j\in \cn{Tiled}^{\Imc}$, from the minimality of $\I$ and given \eqref{tiling4}, it follows that $d^j\in (\exists \cn{tile. TileColor})^{\Imc}$. Observe that since $\I$ is minimal, each $d^j$ can only justify one occurrence of the role $\cn{tile}$ (see \eqref{tiling2}) and $\cn{TileColor}^{\I}=\bigcup_{t\in T} {\cn{A}'_t}^{\Imc}$.
Furthermore, since $A'_t$ do not occur on the right-hand side of the axioms, then ${\cn{A}_t'}^{\Imc}=\{c_t^{\I}\}$. Thus $\cn{TileColor}^{\I}=\bigcup_{t\in T} \{c_t^{\I}\}$.
Thus, there exists a unique $t\in T$ such that $(d^j, c_t)\in \cn{tile}^{\Imc}$ and $c_t\in {\cn{A}_t'}^{\Imc}$. Hence, from \eqref{tiling3}, $d^j\in {\cn{A}_t}^{\Imc}$ and $t$ is unique. 

\begin{claim}
Given any $d^j \in {\cn{A}_{t'}}^{\Imc}$, and $d^k\in \cn{A_{t}}^{\Imc}$, if $(d^j,d^k)\in \cn{h}^{\I}$ then $(t,t')\in H$. Analogously, if $(d^j,d^k)\in \cn{v}^{\I}$ then $(t,t')\in V$.
\end{claim}
\begin{proof}
For the proof of this claim, we refer to the axioms \eqref{tiling1}-\eqref{tiling8}.
Given any $t\in T$ and any $d^j\in {\cn{A}_{t'}}^{\Imc}$ with $t'\in T$ such that $(t,t')\in H$, we have that $d^j\in \cn{HNeigh}_{t'}^{\Imc}$. 
Similarly, for all $t'\in T$ such that $(t,t')\in V$, we have that $d^j\in \cn{VNeigh}_{t'}^{\Imc}$.
Since each $d^j\in \cn{HVok}^{\Imc}$, then from the minimality of $\Imc$, we have that $d^j\in \cn{Hok}^{\Imc}$ and $d^{j}\in \cn{Vok}$ (see \eqref{tiling8}). 
We proceed by showing that the horizontal matching conditions are satisfied.
Therefore, given any $d_j\in \cn{HVok}^{\Imc}$, we have that $d^j\in (\cn{HNeigh_t}\sqcap \exists \cn{h.} \cn{A}_t)^{\I}$. Indeed, since $d^k$ such that $(d^j,d^k)\in h^{\I}$ is uniquely determined, from the minimality of $\I$, it follows that $d^j\in \cn{HNeigh_t}^{\I}$. Hence $(t,t')\in H$.
To show that also the vertical matching conditions are satisfied, the proof is analogous.
\end{proof}

We can now construct a solution for $P$. Recall the mapping $\pi$ we previously defined. We define a map $\tau \colon \mathbb{N}\times \mathbb{N}\rightarrow T$ such that $\tau((x,y))=t$ if and only if $\pi^{-1}((x,y))\in \cn{A}_t^{\Imc}$. We observed that each $d\in \cn{LeafSet}$ belongs to a unique $\cn{A}_t$, with $t\in T$. Therefore, since $\pi$ is a bijection, $\tau$ is well-defined. We show that $(\tau(x,y), \tau(x\oplus_{2^n}1,y))\in H$. Let $d=\pi^{-1}((x,y))$ and $d'\in \pi^{-1}((x\oplus_{2^n}1,y))$. Since each pair $(x,y)$ is encoded exactly once in $\cn{LeafSet}$, from the construction of $\pi$, we have that $(d', d)\in h^{\I}$. Hence, from the last claim, it follows that $(t,t')\in H$.
Analogously, we can show that $(\tau(x,y), \tau(x,y\oplus_{2^n}1))\in V$. Since $\I$ is a model of $\K$, the initial condition is satisfied.
Thus, $\tau$ is a solution. 
        \end{proof}
        \end{proof}

%% file: nexptimeNPproof.tex

\begin{proof}
\emph{Membership.} Upper bound follows trivially from Lemma~\ref{lemma:exp-size-bound}.\\

\emph{Hardness.} We provide a reduction from the complement of succinct \textsc{cert3col} \cite{DBLP:journals/tods/EiterGM97}. Following \cite{BonattiLW09}, an instance of succinct \coccol is an undirected graph $G$ with vertexes $\{0, \cdots, 2^n-1\}$, with $n\in \mathbb{N}$,
such that each edge is labelled by a clause of two literals over a set of Boolean variables $\{v_{i,j}\vert 0\leq i,j<2^n\}$. The graph $G$ is represented as $4n+3$ boolean circuits with $2n$ inputs and one output.
The circuit are defined as follows.
\begin{itemize}
\item $c_E$ (where $E$ stands for edge) takes as input two nodes (coded in binary) and outputs $1$ if there is an edge between them in $G$, it outputs $0$ otherwise.
\item $c_{\sigma}^{(1)}$ that outputs the polarity of the first literal over edges. In particular, $c_{\sigma}^{(1)}$ outputs $1$ if there is an edge between the input nodes and the first literal is positive. While, $c_{\sigma}^{(1)}$ outputs $0$ if there is an edge between the input nodes and the first literal is negative.
\item $c_{\sigma}^{(2)}$ outputs the polarity of the second literal over edges. In particular, $c_{\sigma}^{(2)}$ outputs $1$ if there is an edge between the input nodes and the second literal is positive. While, $c_{\sigma}^{(2)}$ outputs $0$ if there is an edge between the input nodes and the second literal is negative.
\item  the circuits $c_j^{(i)}$, with $i\in \{1,2,3,4\}$ and $0\leq j<n$ computes the labelling $lit_{k_1,k_2}\lor lit_{k_3,k_4}$ if there is an edge between the two input nodes. In particular, each $c_j^{(i)}$, if there is an edge between the two nodes, it computes the $j$-th bit of $k_i$.
\end{itemize}
In all the cases above, except $c_E$, if there is no edge between the input node the output is arbitrary.

A \emph{yes-instance} is a graph $G$ such that for some truth assignment of the variables, the subgraph induce by $\cn{true}$ edges is not $3$-colorable. 

We construct a KB $\Kmc$ and a concept $C_0$ such that $C_0$ is satisfiable in a minimal model of $\Kmc$ if and only if $G$ is a yes-instance of \coccol. We will use them same strategy used in Theorem \ref{nexp:el:strong:ac} to produce trees of bounded depth, with exponentially many leaves. In particular, as pointed out in the proof sketch, we construct 6 trees: \begin{enumerate*}
    \item 3 trees for the 3 colors $R,G$ and $B$, of depth $n$;
    \item 1 tree of depth $n$, which we use for the color assignment of vertexes;
    \item 1 tree of depth $2n$, which we use to define the truth assignment of variables (following the idea sketched in Example \ref{silly});
    \item one final tree, of depth $6n+2$ where each leaf corresponds to a possible edge of the graph and encodes a tuple $(u,v,x,y,\sigma^1,\sigma^2)$ where $u,v$ are vertices of $G$, $x,y$ are variables and $\sigma^1$ and $\sigma^2$ indicate the polarities, i.e. if the variables occur positively or negatively over the edge $(u,v)$.
\end{enumerate*}

We move directly to defining all the axioms. We discuss step by step the requirements we encode via subgoals. The idea is in the spirit very similar to the one used for Theorem \ref{nexp:el:strong:ac}: we ensure that all leaves satisfy a certain property if a concept is satisfied at the leaves. Once we ensure that, we can then add \emph{propagation} axioms, that carry a predicate back to the root if the concept is true at all the leaves.

\emph{First group of axioms: generating the $2^n$ vertexes of the graph and assign colors.} We ensure first that we can produce $2^{n}$ nodes using a binary tree, as done for Theorem \ref{nexp:el:strong:ac}.
\begin{align}
                &\cn{Root}_1(r_1)\label{root1}\\
                &\cn{Root}_1\sqsubseteq \cn{L_0}^{(1)}\label{tree1lev0}\\
                &\cn{L}_i^{(1)}\sqsubseteq \cn{\exists r}_i^{(1)}. \cn{L}_{i+1}^{(1)}\sqcap \cn{\exists l_i. L}_{i+1}^{(1)}\text{ for all $0\leq i< n$}\label{tree1gen}
            \end{align}
We now generate a goal concept $\cn{Tree}_1$ that is satisfiable iff the tree has $2^n$ leaves.
\begin{align}
            \cn{Left}(o) &\qquad \cn{Right}(o')\label{choicelr}\\
            \cn{L}_{n}^{(1)}&\sqsubseteq \lambda_{0}^{(1)}\label{revtree1lev0}\\
            \lambda_j^{(1)}&\sqsubseteq \exists \cn{pick}\label{picksucc}\\
            \lambda_j^{(1)}\sqcap \exists \cn{pick}. \cn{Left}&\sqsubseteq \exists \cn{s}_j^{(1)}. \lambda_{j+1,l}^{(1)}\label{revleft}\\
            \lambda_j^{(1)}\sqcap \exists \cn{pick}. \cn{Right}&\sqsubseteq \exists \cn{d}_j^{(1)}. \lambda_{j+1,r}^{(1)}\label{revright}\\
            \lambda_{j+1,s}^{(1)}\sqcap \lambda_{j+1,d}^{(1)}&\sqsubseteq \lambda_{j+1}^{(1)} \text{ for all }0\leq j\leq n\label{nextlevtree1}\\
            \lambda_{n,s}^{(1)}\sqcap\lambda_{n,d}^{(1)}&\sqsubseteq \cn{Tree}_1\label{closetree1}
\end{align}

We can also produce the bit vector of each vertex using special concepts $\cn{B}_i^{(1)}$ and $\Bar{\cn{B}_i}^{(1)}$, with $0\leq i<n$. In particular, looking at the bit vector using the aforementioned concepts, each element in $L_{n}^{(1)}$ corresponds to a unique vertex of the graphs, and all vertices are represented exactly once.\\

\begin{minipage}{0.45\textwidth}
\begin{align}
\exists \cn{s}_{n-1}^{(1)}.\cn{Tree}_1&\sqsubseteq \bar{\cn{B}}_0^{(1)}\label{bitone}\\
\exists \cn{d}_{n-1}^{(1)}.\cn{Tree}_1&\sqsubseteq \cn{B}_0^{(1)}\\
\exists \cn{s}_{j}^{(1)}. \bar{\cn{B}}_j^{(1)}&\sqsubseteq \bar{\cn{B}}_j^{(1)} \\
\exists \cn{s}_{j}^{(1)}. {\cn{B}}_j^{(1)}&\sqsubseteq {\cn{B}}_j^{(1)}
\end{align}
\end{minipage}
\begin{minipage}{0.45\textwidth}
\begin{align}
\exists \cn{d}_{j}^{(1)}. \bar{\cn{B}}_j^{(1)}&\sqsubseteq \bar{\cn{B}}_i^{(1)} \\
\exists \cn{d}_{j}^{(1)}. {\cn{B}}_j^{(1)}&\sqsubseteq {\cn{B}}_j^{(1)} \\
\exists \cn{s}_j^{(1)}. \cn{\lambda}_{j+1}^{(1)}&\sqsubseteq \cn{\bar{B}}_{j+1}^{(1)}\\
\exists \cn{d}_j^{(1)}. \cn{\lambda}_{j+1}^{(1)}&\sqsubseteq \cn{{B}}_{j+1}^{(1)}\label{bitonefinal}
\end{align}
\end{minipage}\\

We will pick the colors at each leaf of the tree. Before proceeding with the axiom description, let us remark that we cannot simply use 3 assertions and require that the leaves point to one of such dedicated assertions. Indeed, our final goal is to show that we find a \emph{one} truth assignment, such that under \emph{all} color assignments the graph is not 3-colorable. If we use a similar trick used in Theorem \ref{nexp:el:strong:ac} to associate color tiles, in a minimal model we can justify only \emph{one} color assignment. However, we want to be able to \emph{saturate} the structure in the sense of Example \ref{example:flooding-technique}. 

To avoid this issue, we generate the 3 trees, one per color, of depth $n$ and assign to each leaf exactly one color.
We will then connect each leaf encoding in binary a number $x<2^n$ to a leaf in the first tree (the vertex tree) encoding in binary the same number $x$.

To generate a tree of depth $n$ for each of the tree colors red ($R$), blue ($B$) and green ($G$), we use our standard axioms. For $C\in \{R,G,B\}$, we consider the following axioms:
\begin{align}
&\cn{Root}_C(r_C)\label{rootC}\\
                &\cn{Root}_C\sqsubseteq \cn{L_0}^{(C)}\label{colortree}\\
                &\cn{L}_i^{(C)}\sqsubseteq \cn{\exists r}_i^{(C)}. \cn{L}_{i+1}^{(C)}\sqcap \cn{\exists l_i. L}_{i+1}^{(C)}\text{ for all $i\leq n$}\label{colortreegen}
            \end{align}
We now generate a goal concept $\cn{Tree}_C$ that is satisfiable iff the tree has $2^n$ leaves.
\begin{align}
            \cn{L}_{n}^{(C)}&\sqsubseteq \lambda_{0}^{(C)}\label{revtreeclev0}\\
            \lambda_j^{(C)}&\sqsubseteq \exists \cn{pick} \text{ for all }0\leq j< n\label{picksuccc}\\
            \lambda_j^{(C)}\sqcap \exists \cn{pick}. \cn{Left}&\sqsubseteq \exists \cn{s}_j^{(C)}. \lambda_{j+1,l}^{(C)} \text{ for all }0\leq j< n\label{revleftc}\\
            \lambda_j^{(C)}\sqcap \exists \cn{pick}. \cn{Right}&\sqsubseteq \exists \cn{d}_j^{(C)}. \lambda_{j+1,r}^{(C)}  \text{ for all }0\leq j< n\label{revrightc}\\
            \lambda_{j+1,s}^{(C)}\sqcap \lambda_{j+1,d}^{(C)}&\sqsubseteq \lambda_{j+1}^{(C)} \text{ for all }0\leq j< n\label{nextlevtreec}\\
            \lambda_{n,s}^{(C)}\sqcap\lambda_{n,d}^{(C)}&\sqsubseteq \cn{Tree}_C\label{closetreec}
\end{align}
We mark all the leaves of the tree with color $C$ by introducing a concept for each color. The aim of the followin axioms is twofold: \begin{enumerate*}
    \item guarantee that the picked color is uniqueaand that all leaves pick the same color, producing a subgoal $\cn{C}$, with $\cn{C}\in \{\cn{R,G,B}\}$;
    \item if for two colors $C_1$ and $C_2$, the two trees have a leaf in common, we can detect this mistake and avoid satisfying the subgoal mentioned in the previous item.
\end{enumerate*}
We let the leaves pick a color, and copy it.
\begin{align}
\cn{R_A(r)\qquad} &\cn{G_A(g)\qquad B_A(b)}\\
\cn{L}_{2n}^{(C)}&\sqsubseteq \exists \cn{pick\_col}\label{pickcol}\\
\cn{\exists pick\_col}. \cn{C}&\sqsubseteq \cn{C} \qquad \cn{C}\in \{\cn{R, G, B}\}\label{picked}
\end{align}
We refer to these special trees as \emph{color trees}. We keep $\cn{C}$ as one of the subgoals that has to be satisfied at all the leaves. Similarly to Theorem \ref{nexp:el:strong:ac}, we genate a goal concept that is satisfied at the root (in a minimal model) if it is satisfied at all the leaves of the tree.

Similarly to above, we can use $n$ concept $\cn{B}_i^{(C)}$ and $\cn{\bar{B}}_i^{(C)}$, with $0\leq i<n$ per each color $C\in \{R,G,B\}$ to associate each of the leaves to a vertex of the graph.\\
 
\begin{minipage}{0.45\textwidth}
\begin{align}
\exists \cn{s}_{n-1}^{(C)}.\cn{Tree}_1&\sqsubseteq \bar{\cn{B}}_0^{(C)}\label{bitcolorone}\\
\exists \cn{d}_{n-1}^{(C)}.\cn{Tree}_1&\sqsubseteq \cn{B}_0^{(C)}\\
\exists \cn{s}_{j}^{(C)}. \bar{\cn{B}}_j^{(C)}&\sqsubseteq \bar{\cn{B}}_j^{(C)} \\
\exists \cn{s}_{j}^{(C)}. {\cn{B}}_j^{(C)}&\sqsubseteq {\cn{B}}_j^{(C)}
\end{align}
\end{minipage}
\begin{minipage}{0.45\textwidth}
\begin{align}
\exists \cn{d}_{j}^{(C)}. \bar{\cn{B}}_j^{(C)}&\sqsubseteq \bar{\cn{B}}_i^{(C)} \\
\exists \cn{d}_{j}^{(C)}. {\cn{B}}_j^{(C)}&\sqsubseteq {\cn{B}}_j^{(C)} \\
\exists \cn{s}_j^{(C)}. \cn{\lambda}_{j+1}^{(C)}&\sqsubseteq \cn{\bar{B}}_{j+1}^{(C)}\\
\exists \cn{d}_j^{(C)}. \cn{\lambda}_{j+1}^{(C)}&\sqsubseteq \cn{{B}}_{j+1}^{(C)}\label{bitcolorfinal}
\end{align}
\end{minipage}\\

We now require that each leaf of the first tree can see 3 possible choices for the colors and that it has to make a choice. 

\begin{align}
\cn{L}_{2n}^{(1)}\sqsubseteq \exists \cn{col. R'} \sqcap \exists \cn{col. G'} \sqcap \exists \cn{col. B'}\sqcap \exists \cn{col.} \cn{Chosen}\label{performchoice}
\end{align}

With the following axioms, we generate a goal concept $\cn{Col}$ whose satisfiability in a minimal model ensures that, for each leaf node of the first tree, the $\cn{col}$ successor $\cn{R', G', B'}$ are at the leaf of the respective color trees, one of them is also labeled with $\cn{Chosen}$ and the coordinate if the two nodes coincide. Intuitively speaking, we require that the vertex of the graph (encoded in the leaves of the first tree), are connected to all their possible color assignments. One of such color assignmentss is chosen.

\begin{minipage}{0.45\textwidth}
\begin{align}
&\exists \cn{col. (R'\sqcap R)}\sqsubseteq \cn{GoodCol_R}\label{readred}\\
&\exists \cn{col. (G'\sqcap G)}\sqsubseteq \cn{GoodCol_G}\label{readgreen}\\
&\exists \cn{(col. B'\sqcap B)}\sqsubseteq \cn{GoodCol_B}\label{readblue}\\
&\cn{GoodCol_R}\sqcap \cn{GoodCol_G}\sqcap \cn{GoodCol_B}\sqsubseteq \cn{GoodCol}\label{colorconnection}\\
&\exists \cn{col. B}_i^{(C)} \sqcap \cn{B}_i^{(1)} \sqsubseteq \cn{ColCoord}_{i,x}^{(C)}\label{checkcoordp}\\
&\exists \cn{col. \bar{B}}_i^{(C)} \sqcap \cn{\bar{B}}_i^{(1)} \sqsubseteq \cn{ColCoord}_{i}^{(C)}\label{checkcoordn}
\end{align}
\end{minipage}
\begin{minipage}{0.5\textwidth}
\begin{align}
&\bigsqcap_{0\leq i<n} \cn{ColCoord}_{i}^{(C)}\sqsubseteq \cn{CoordMatch}^{(C)}\label{coordmatchc}\\
&\bigsqcap_{C\in \{R,G,B\}} \cn{CoordMatch}^{(C)}\sqsubseteq \cn{CoordMatch}\label{coordmatch}\\
&\exists \cn{col}. (\cn{Chosen}\sqcap \cn{C}) \sqsubseteq \cn{GoodChoice} \sqcap \cn{C}'\; \;\cn{C}\in \{\cn{R}, \cn{G}, \cn{B}\}\label{choose}\\
&\cn{GoodCol}\sqcap \cn{GoodChoice}\sqcap \cn{CoordMatch} \sqsubseteq \cn{Col}\label{coljust}
\end{align}
\end{minipage}\\

We will keep $\cn{Col}$ has one of the subgoals that we need to satisfy at all the leaves.

To given an intuition of what we do in the following, we recall the overall idea of generating subgoals to satisfy at the leaves of the trees.
We generate a few concepts that we aim to satisfy at all the leaves of some tree. Only later, we wrap up all such requirements expressed by some concept name and send a message back to the root of the tree. Intuitively speaking, we require that at each level of the tree, a concept is true at a node only if it is true at is right-successor and its left successor. Therefore, since we aim for minimal models, all the message passing of concepts has to involve all the leaves.

\emph{Second group of axioms: the variable tree and their truth assignment.}
Observe that we assume a quadratic number of variables (in the size of the input graph).
We produce another tree as above (same axioms, just more depth). 
\begin{align}
                &\cn{Root}_2(r_2)\\
                &\cn{Root}_2\sqsubseteq \cn{L_0}^{(2)}\\
                &\cn{L}_i^{(2)}\sqsubseteq \cn{\exists r}_i^{(2)}. \cn{L}_{i+1}^{(2)}\sqcap \cn{\exists l_i. L}_{i+1}^{(2)}\text{ for all $0\leq i<2n$}
            \end{align}
We now generate a goal concept that is satisfiable iff the tree has $2^{2n}$ leaves. Recall that we already have two assertion $\cn{Left}(o)$ and $\cn{Right}(o')$ which we use to decide which successor a node generates.
\begin{align}
            \cn{L}_{2n}^{(2)}&\sqsubseteq \lambda_{0}^{(2)}\\
            \lambda_j^{(2)}&\sqsubseteq \exists \cn{pick}\\
            \lambda_j^{(2)}\sqcap \exists \cn{pick}. \cn{Left}&\sqsubseteq \exists \cn{l}_j^{(2)}. \lambda_{j+1,l}^{(2)}\\
            \lambda_j^{(2)}\sqcap \exists \cn{pick}. \cn{Right}&\sqsubseteq \exists \cn{r}_j^{(2)}. \lambda_{j+1,r}^{(2)}\\
            \lambda_{j+1,l}^{(2)}\sqcap \lambda_{j+1,r}^{(2)}&\sqsubseteq \lambda_{j+1}^{(2)} \text{ for all }1\leq j<2n\\
            \lambda_{2n,l}^{(1)}\sqcap\lambda_{2n,l}^{(2)}&\sqsubseteq \cn{Tree}_2
\end{align}
We can also produce the bit encodings of each vertex using special concepts $\cn{B}_i^{(2)}$ and $\Bar{\cn{B}_i}^{(2)}$, with $0\leq i<2n$. In particular, each element in $\cn{L}_{2n}^{(2)}$ corresponds to a unique variable and each variable is encoded in binary exactly one time. 

\begin{minipage}{0.45\textwidth}
\begin{align}
\exists \cn{s}_{2n-1}^{(2)}.\cn{Tree}_2\sqsubseteq \bar{\cn{B}}_0^{(2)}\label{bittwo}\\
\exists \cn{d}_{2n-1}^{(2)}.\cn{Tree}_2\sqsubseteq \cn{B}_0^{(1)}\\
\exists \cn{s}_{j}^{(2)}. \bar{\cn{B}}_j^{(2)}\sqsubseteq \bar{\cn{B}}_j^{(2)} \\
\exists \cn{s}_{j}^{(2)}. {\cn{B}}_j^{(2)}\sqsubseteq {\cn{B}}_j^{(2)}
\end{align}
\end{minipage}
\begin{minipage}{0.45\textwidth}
\begin{align}
\exists \cn{d}_{j}^{(2)}. \bar{\cn{B}}_j^{(2)}\sqsubseteq \bar{\cn{B}}_i^{(2)} \\
\exists \cn{d}_{j}^{(2)}. {\cn{B}}_j^{(2)}\sqsubseteq {\cn{B}}_j^{(2)} \\
\exists \cn{s}_j^{(2)}. \cn{\lambda}_{j+1}^{(2)}\sqsubseteq \cn{\bar{B}}_{j+1}^{(2)}\\
\exists \cn{d}_j^{(2)}. \cn{\lambda}_{j+1}^{(2)}\sqsubseteq \cn{{B}}_{j+1}^{(2)}\label{bitwofinal}
\end{align}
\end{minipage}\\

Each leaf of this second tree corresponds to a variable $v_{i,j}$, with $i,j<2^n$. With the next axioms, we introduce a new goal concept $\cn{Ass}$ (short for assignment). In a minimal model, $\cn{Ass}$ is satisfied at each of the leaves of the tree if all nodes `picked' a truth assignment using the role $\cn{v}$.
\begin{align}
    	\cn{T(v_1) \qquad F(v_2)}\label{tf}\\
    \cn{L}_{2n}^{(2)}\sqsubseteq \exists \cn{v.}\label{pickass}\\
    \cn{\exists v. T\sqsubseteq T' \sqcap Ass}\label{pickedT}\\
    \cn{\exists v. F\sqsubseteq F' \sqcap Ass}\label{pickedF}\\
\end{align}

We keep aside the predicate $\cn{Ass}$ and discuss later how we ensure its satisfaction at all the leaves of the tree rooted in $\cn{r_2}$.

\emph{Third group of axioms: the final tree.}
This is a binary tree of depth $6n+2$. At the leafs of this tree we store: (a)  two bit vectors of length $n$, corresponding to a pair of vertices (a possible edge in the graph), (b) two bit vectors of length $2n$, corresponding to a pair of propositional atoms, and (c) two bits to indicate if the first and the second propositional atoms in the labeling are negated or not.
The tree is produced using the same axiom schema used for the others. Observe that: 
\begin{itemize}
    \item at depth $2n$, we produced all the possible binary coordinates encoding pairs of vertexes. We can make this explicit using the concepts $\cn{V}_i$, $\Bar{\cn{V}}_i$, $\cn{U}_i$ and $\Bar{\cn{U}}_i$.
    \item when we continue, at depth $4n$ we produced also the coordinates encoding pairs of variables that can be possibly assigned to possible edges. We can use for this other special concepts $\cn{X}_j$, $\Bar{\cn{X}}_i$, $\cn{Y}_j$ and $\Bar{\cn{Y}}_j$.
    \item after two  more steps, we use the concepts $\cn{P}_1, \Bar{\cn{P}}_1$ and $\cn{P}_2, \Bar{\cn{P}}_2$. 
\end{itemize}
In particular, the concepts $\cn{P}_i$ stand for `variable $i$ occurs positively', for $i\in \{1,2\}$. While the concepts $\cn{\bar{P}}_i$ stand for `variable $i$ occurs negatively', for $i\in \{1,2\}$. In a nutshell, they correspond to the polarities $\sigma^i$ in the tuple $(u,v, x,y, \sigma^1,\sigma^2)$.

First we produce the tree, similarly to the previous two cases.

\begin{align}
                &\cn{Root}_3(r_3)\\
                &\cn{Root}_3\sqsubseteq \cn{L_0}^{(3)}\\
                &\cn{L}_i^{(3)}\sqsubseteq \cn{\exists r}_i^{(3)}. \cn{L}_{i+1}^{(3)}\sqcap \cn{\exists l_i. L}_{i+1}^{(3)}\text{ for all $0\leq i< 6n+2$}
            \end{align}
We now generate a goal concept that is satisfiable iff the tree has $2^{6n+3}$ leaves.
\begin{align*}
            \cn{L}_{n}^{(3)}&\sqsubseteq \lambda_{0}^{(3)}\\
            \lambda_j^{(3)}&\sqsubseteq \exists \cn{pick}\\
            \lambda_j^{(3)}\sqcap \exists \cn{pick}. \cn{Left}&\sqsubseteq \exists \cn{s}_j^{(3)}. \lambda_{j+1,l}^{(3)}\\
            \lambda_j^{(3)}\sqcap \exists \cn{pick}. \cn{Right}&\sqsubseteq \exists \cn{d}_j^{(3)}. \lambda_{j+1,r}^{(3)}\\
            \lambda_{j+1,s}^{(3)}\sqcap \lambda_{j+1,d}^{(1)}&\sqsubseteq \lambda_{j+1}^{(3)} \text{ for all }0\leq j\leq n\\
            \lambda_{n,s}^{(3)}\sqcap\lambda_{n,d}^{(3)}&\sqsubseteq \cn{Tree}_3
\end{align*}
We can also produce the bit encodings of each vertex using special concepts $\cn{B}_i^{(3)}$ and $\cn{\bar{B}}_i^{(3)}$, with $0\leq i< 6n+2$. In particular, each element in $L_{6n+2}^{(3)}$ corresponds to a unique vertex of the graphs and all vertexes are represented exactly once.

\begin{minipage}{0.5\textwidth}
\begin{align}
\exists \cn{s}_{6n+2}^{(3)}.\cn{Tree}_3\sqsubseteq \bar{\cn{B}}_0^{(3)}\label{bitthree}\\
\exists \cn{d}_{6n+2}^{(3)}.\cn{Tree}_3\sqsubseteq \cn{B}_0^{(3)}\\
\exists \cn{s}_{j}^{(3)}. \bar{\cn{B}}_j^{(3)}\sqsubseteq \bar{\cn{B}}_j^{(3)} \\
\exists \cn{s}_{j}^{(3)}. {\cn{B}}_j^{(3)}\sqsubseteq {\cn{B}}_j^{(3)} 
\end{align}
\end{minipage}
\begin{minipage}{0.5\textwidth}
\begin{align}
\exists \cn{d}_{j}^{(3)}. \bar{\cn{B}}_j^{(3)}\sqsubseteq \bar{\cn{B}}_i^{(3)} \\
\exists \cn{d}_{j}^{(3)}. {\cn{B}}_j^{(3)}\sqsubseteq {\cn{B}}_j^{(3)} \\
\exists \cn{s}_j^{(3)}. \cn{\lambda}_{j+1}^{(3)}\sqsubseteq \cn{\bar{B}}_{j+1}^{(3)}\\
\exists \cn{d}_j^{(3)}. \cn{\lambda}_{j+1}^{(3)}\sqsubseteq \cn{{B}}_{j+1}^{(3)}\label{bitthreefinal}
\end{align}
\end{minipage}

For the sake of making the next axioms more understandable, we split the bit encoding given by $\cn{B}_i^{(3)}$ and $\cn{\bar{B}}_i^{(3)}$ to highlight which part encodes the pair of vertexes, which part encodes the variables and which part expresses the polarities. 
\begin{align*}
\cn{B}_i^{(3)}\sqsubseteq \cn{U}_i \qquad \cn{\bar{B}}_i^{(3)} &\sqsubseteq \cn{\bar{U}_i}\;0\leq i <n\\
\cn{B}_i^{(3)}\sqsubseteq \cn{V}_j \qquad \cn{\bar{B}}_i^{(3)} &\sqsubseteq \cn{\bar{V}_j}\;n\leq i <2n, j=i-n\\
\cn{B}_i^{(3)}\sqsubseteq \cn{X}_j \qquad \cn{\bar{B}}_i^{(3)} &\sqsubseteq \cn{\bar{X}_j}\;2n\leq i <4n, j=i-2n\\
\cn{B}_i^{(3)}\sqsubseteq \cn{Y}_j \qquad \cn{\bar{B}}_i^{(3)} &\sqsubseteq \cn{\bar{Y}_j}\;4n\leq i <6n, j=i-4n\\
\cn{B}_{6n}^{(3)}\sqsubseteq \cn{P_1}\qquad \cn{\bar{B}}_{6n}^{(3)} &\sqsubseteq \cn{{\bar{P}}_1}\\
\cn{B}_{6n+1}^{(3)}\sqsubseteq \cn{P_2}\qquad \cn{\bar{B}}_{6n+1}^{(3)} &\sqsubseteq \cn{{\bar{P}}_2}
\end{align*}
Our first step is to assign each pair $(u,v)$ represented at one leaf of the third tree, with two leaves in the first tree.
To do so, we generate yet another goal concept that we desire to satisfy at all the leaves of the third tree. The satisfaction of such concept, together with ensuring the connection with the first tree, also ensures that the pair $(u,v)$ encoded in at leaves of the third tree matches the two leaves in the first tree.

\begin{align}
    \leaft \sqsubseteq \cn{\exists p_1^{(1)} \sqcap \exists p_2^{(1)}}\label{threetoone1}\\
    \exists \cn{p_1^{(1)}.} \cn{L}_{n}^{(1)} \sqsubseteq \cn{C_1^{(1)}}\label{threetoone2}\\
    \exists \cn{p_2^{(1)}. }\cn{L}_{n}^{(1)}\sqsubseteq \cn{C_2}^{(1)}\label{threetoone3}\\
    \cn{C_1^{(1)} \sqcap C_2^{(1)} \sqsubseteq C^{(1)}}\label{threetoone4}
\end{align}

Intuitively speaking, in a minimal model $\cn{C^{(1)}}$ is satisfied at all the leaves of a tree rooted in $r_2$ if each node in $\leaft$ connects with the first tree via $\cn{p_1}^{(1)}$ and $\cn{p_2}^{(1)}$.

We now need to check that coordinates are fine. This can be done by importing the coordinates from first tree and checking if the match the encoded pair $(u,v)$. 

\begin{minipage}{0.45\textwidth}
\begin{align}
    \exists \cn{p_1^{(1)}.} \cn{B}_i^{(1)} \sqcap \cn{V}_i &\sqsubseteq \cn{Ok}_i^1\label{comparex1}\\
    \exists \cn{p_1^{(1)}.} \cn{\bar{B}}_i^{(1)} \sqcap \cn{\bar{V}}_i &\sqsubseteq \cn{Ok}_i^1\\
    \bigsqcap_{i=0,\cdots, n-1} \cn{Ok}_i^1&\sqsubseteq \cn{Ok}^1
\end{align}
\end{minipage}
\begin{minipage}{0.45\textwidth}
\begin{align}
    \exists \cn{p_2^{(1)}}. \cn{B}_i^{(1)} \sqcap \cn{U}_i &\sqsubseteq \cn{Ok}_i^2\\
    \exists \cn{p_2^{(1)}}. \cn{\bar{B}}_i^{(1)} \sqcap \cn{\bar{U}}_i &\sqsubseteq \cn{Ok}_i^2\\
    \bigsqcap_{i=0, \cdots, n-1} \cn{Ok}_i^2&\sqsubseteq Ok^2
\end{align} 
\end{minipage}
\begin{align}
    \cn{Ok^1\sqcap Ok^2}& \sqsubseteq \cn{Ok}_{(u,v)}\label{compared}
\end{align}
In the following we require the satisfaction of $\cn{C}^{(1)}\sqcap \cn{Ok}_{(u,v)}$ at all the leaves.
We can do exactly the same to connect the leaves with the leaves of the second tree and check that coordinates are matching.

\begin{align}
    \leaft &\sqsubseteq \cn{\exists p_1^{(2)} \sqcap \exists p_2^{(2)}}\label{threetotwo1}\\
    \exists \cn{p_1.^{(2)}} \cn{L}_{2n}^{(2)} &\sqsubseteq \cn{C_1^{(2)}}\\
    \exists \cn{p_2.^{(2)} }\cn{L}_{2n}^{(2)} &\sqsubseteq \cn{C_2}^{(2)}\\
    \cn{C_1^{(2)} \sqcap C_2^{(2)}} &\sqsubseteq \cn{C^{(2)}}\label{threetotwo4}
\end{align}
We now need to check that coordinates are fine. This can be done by importing the coordinates from first tree. 

\begin{minipage}{0.5\textwidth}
\begin{align}
    \exists \cn{p_1^{(2)}.} \cn{B}_i^{(2)} \sqcap \cn{X}_i \sqsubseteq \cn{Bit}_i^1\label{comparevalues1}\\
    \exists \cn{p_1^{(2)}.} \cn{\bar{B}}_i^{(2)} \sqcap \cn{\bar{X}}_i \sqsubseteq \cn{Bit}_i^1\\
    \bigsqcap_{i=0,\cdots, 2n-1} \cn{Bit}_i^1\sqsubseteq \cn{Var}^1
\end{align}
\end{minipage}
\begin{minipage}{0.5\textwidth}
\begin{align}
    \exists \cn{p_2^{(1)}}. \cn{B}_i^{(2)} \sqcap \cn{Y}_i \sqsubseteq \cn{Bit}_i^2\\
    \exists \cn{p_2^{(1)}}. \cn{\bar{B}}_i^{(2)} \sqcap \cn{\bar{Y}}_i \sqsubseteq \cn{Bit}_i^2\\
    \bigsqcap_{i=0, \cdots, 2n-1} \cn{Bit}_i^2\sqsubseteq \cn{Var}^2
\end{align}
\end{minipage}

\begin{align}
    \cn{Var^1\sqcap Var^2 \sqsubseteq Ok}_{(x,y)}\label{comparedvar}
\end{align} 
We will ensure that $\cn{C}^{(2)}\sqcap \cn{Ok}_{(x,y)}$ is true at each leaf.
Once such connections among trees are established, we can import the color assignment and the truth value assignment.

\begin{align}
    \exists \cn{p_1^{(1)}. C \sqsubseteq C_1}\text{ for all }\cn{C\in \{R,G,B\}}\label{importcolor1}\\
    \exists \cn{p_2^{(1)}. C \sqsubseteq C_2}\text{ for all }\cn{C\in \{R,G,B\}}\label{importcolor2}
\end{align}
Likewise we copy the truth values from the second tree. 
\begin{align}
    \exists \cn{p_1^{(2)}. T' \sqsubseteq True_1}\label{importval1}\\
    \exists \cn{p_1^{(2)}. F' \sqsubseteq False_1}\label{importval2}\\
    \exists \cn{p_2^{(2)}. T' \sqsubseteq True_2}\label{importval3}\\
    \exists \cn{p_2^{(2)}. F'\sqsubseteq False_2}\label{importval4}
\end{align}

To summarize, if all the goal concepts that we have produced are satisfied at the corresponding leaves, each leaf in the third tree has two colors assigned, and two truth values. Such assignment is well defined. If a vertex $u'$ occurs more times in the tuples $(u,v,x,y, \sigma_1,\sigma_2)$, its color is always the same. Indeed, all the occurrences of $u'$ will force the corresponding leaf in the third tree to point to the same node in the first tree (because coordinates have to match!). 

We can now mark leaves with a special concept $\cn{IsEdge}$ to recognize that the combination $(u,v,x,y, \sigma^1, \sigma^2)$ encoded in the leaf is an edge in $G$. 

For each circuit $c$ and each gate $G$ in $c$, we have a pair of concept names $\cn{Val}_{G,c}^0$ and $\cn{Val}_{G,c}^{1}$ intuitively corresponding to the output of the gate $G$. We then add the following axioms:
\begin{small}
\begin{align}
\cn{V}_i &\sqsubseteq \cn{Val}_{G,c}^1&\qquad \text{ if }G\text{ is the $i$-th input gate}\label{firstcircuit}\\
&\cn{\bar{V}}_i \sqsubseteq \cn{Val}_{G,c}^0&\qquad \text{ if }G\text{ is the $i$-th input gate}\\
&\cn{U}_i \sqsubseteq \cn{Val}_{G,c}^1&\qquad \text{ if }G\text{ is the $i+2n$-th input gate}\\
&\cn{\bar{U}}_i \sqsubseteq \cn{Val}_{G,c}^1&\qquad \text{ if }G\text{ is the $i+2n$-th input gate}\\
&\cn{Val}_{G_1,c}^1\sqcap \cn{Val}_{G_2,c}^1\sqsubseteq \cn{Val}_{G,c}^1&\qquad \text{ if $G$ is an AND-gate with inputs $G_1$ and $G_2$}\\
&\cn{Val}_{G_1,c}^0\sqcup \cn{Val}_{G_2,c}^0\sqsubseteq \cn{Val}_{G,c}^0&\qquad \text{ if $G$ is an AND-gate with inputs $G_1$ and $G_2$}\\
&\cn{Val}_{G_1,c}^1\sqcup \cn{Val}_{G_2,c}^1\sqsubseteq \cn{Val}_{G,c}^1&\qquad \text{ if $G$ is an OR-gate with inputs $G_1$ and $G_2$}\\
&\cn{Val}_{G_1,c}^0\sqcap \cn{Val}_{G_2,c}^0\sqsubseteq \cn{Val}_{G,c}^0&\qquad \text{ if $G$ is an OR-gate with inputs $G_1$ and $G_2$}\\
&\cn{Val}_{G_1,c}^1\sqsubseteq \cn{Val}_{G,c}^0&\qquad \text{ if $G$ is a NOT-gate with input $G_1$}\\
&\cn{Val}_{G_1,c}^0\sqsubseteq \cn{Val}_{G,c}^1&\qquad \text{ if $G$ is a NOT-gate with input $G_1$}
\end{align}
\end{small}
For $c=c_E$, assuming that $G^o$ is the output gate, we add the following:
\begin{align}
\cn{Val}_{G^o,{c_E}}^1 \sqsubseteq \cn{Edge}
\end{align}
For the circuits $c_\sigma^{(1)}$ and $c_\sigma^{(2)}$, given the output gate $G^o$, we add the following:
\begin{align}
\cn{Edge} \sqcap \cn{Val}_{G^o, c_\sigma^{(1)}}^1&\sqsubseteq \cn{Pos_1}\\
\cn{Edge} \sqcap \cn{Val}_{G^o, c_\sigma^{(1)}}^0&\sqsubseteq \cn{{Neg}_1}\\
\cn{Edge} \sqcap \cn{Val}_{G^o, c_\sigma^{(2)}}^1&\sqsubseteq \cn{{Pos}_2}\\
\cn{Edge} \sqcap \cn{Val}_{G^o, c_\sigma^{(2)}}^0&\sqsubseteq \cn{{Neg}_2}
\end{align}
For each circuit $c_{j}^{(i)}$, with $i=1,2$ given the output gate $G^o$, we add the following:
\begin{align}
\cn{Edge} \sqcap \cn{Val}_{G^o, c_{j}^{(i)}}^1\sqsubseteq \cn{X}_j^c\\
\cn{Edge} \sqcap \cn{Val}_{G^o, c_{j}^{(i)}}^0\sqsubseteq \cn{\bar{X}}_j^c
\end{align}
For each circuit $c_{j}^{(i)}$, with $i=3,4$ given the output gate $G^o$, we add the following:
\begin{align}
\cn{Edge} \sqcap \cn{Val}_{G^o, c_{j}^{(i)}}^1\sqsubseteq \cn{Y}_j^c\\
\cn{Edge} \sqcap \cn{Val}_{G^o, c_{j}^{(i)}}^0\sqsubseteq \cn{\bar{Y}}_j^c
\end{align}
We now compare each $(x,y,u,v,\sigma^1, \sigma^2)$ encoded in the leaves with the output predicates given by the circuits.
\begin{align}
\cn{X_j \sqcap X_j}^c &\sqsubseteq \cn{Comp}_{j,x}\\
\cn{\bar{X}_j \sqcap \bar{X}_j}^c \sqsubseteq\cn{Comp}_{j,x}\\
\bigsqcap_{j=0}^{2n-1}\cn{Comp}_{j,x} &\sqsubseteq \cn{Comp}_x\\
\cn{Y_j \sqcap Y_j}^c &\sqsubseteq \cn{Comp}_{j,y}\\
\cn{\bar{Y}_j \sqcap \bar{Y}_j}^c &\sqsubseteq\cn{Comp}_{j,y}\\
\bigsqcap_{j=0}^{2n-1}\cn{Comp}_{j,y} &\sqsubseteq \cn{Comp}_y\\
\cn{Comp}_x\sqcap \cn{Comp}_y &\sqsubseteq \cn{Lab}
\end{align}
We do the same for the polarities of the variable.
\begin{align}
\cn{Pos}_1 \sqcap \cn{P}_1 &\sqsubseteq \cn{Pol}_1\\
\cn{Neg}_1 \sqcap \cn{\bar{P}_1} &\sqsubseteq \cn{Pol}_1\\
\cn{Pos}_2 \sqcap \cn{P}_2 &\sqsubseteq \cn{Pol}_2\\
\cn{Neg}_2 \sqcap \cn{\bar{P}_2} &\sqsubseteq \cn{Pol}_2\\
\cn{Pol}_1 \sqcap \cn{Pol}_2 &\sqsubseteq \cn{Pol}
\end{align}
We now mark nodes that correspond to edges in the graph: 
\begin{align}
\cn{Edge} \sqcap \cn{Lab} \sqcap \cn{Pol}\sqsubseteq \cn{IsEdge}\label{finalcircuit}
\end{align}

We now mark with a concept $\cn{TrueEdge}$ the edges that are true under the assignment given by the second tree.
\begin{align}
    \cn{IsEdge}\sqcap \cn{Neg}_1 \sqcap \cn{False}_1\sqsubseteq \cn{TrueEdge_1}\label{truedge1}\\
    \cn{IsEdge}\sqcap \cn{Pos}_1\sqcap \cn{True}_1\sqsubseteq \cn{TrueEdge_1}\\
    \cn{IsEdge}\sqcap \cn{Neg}_2 \sqcap \cn{False}_2\sqsubseteq \cn{TrueEdge_2}\\
    \cn{IsEdge}\sqcap \cn{Pos}_2 \sqcap \cn{True}_2 \sqsubseteq \cn{TrueEdge_2}\label{truedge2}\\
    \cn{TrueEdge_1} \sqcap \cn{TrueEdge_2}\sqsubseteq \cn{TrueEdge}\label{truedgefinal}
\end{align}
\begin{align}
    \cn{TrueEdge}\sqcap \cn{R_1\sqcap R_2\sqsubseteq Flood}\label{detectmistake}\\
    \cn{TrueEdge}\sqcap \cn{G_1\sqcap G_2\sqsubseteq Flood}\\
    \cn{TrueEdge}\sqcap \cn{B_1\sqcap B_2\sqsubseteq Flood}\label{detectmistake1}\\
    \cn{Flood}\sqsubseteq \cn{\bigsqcap_{i=1,2} R_i\sqcap G_i\sqcap B_i}
\end{align}
We propagate $\cn{Flood}$ to the root.
\begin{align}
\exists \cn{r}_i^{(3)}. \cn{Flood} \sqcup \exists \cn{l}_i^{(3)}. \cn{Flood}\sqsubseteq \cn{Flood}\label{spreadflood}
\end{align}
We now connect all the leaves to $\cn{Root}_3$ as follows:
\begin{align}
\leaft\sqsubseteq \exists \cn{root.}\label{toroot}\\
\exists \cn{root.} \cn{Root_3} \sqsubseteq \cn{Rooted}\label{Rooted}
\end{align}
By requiring that $\cn{Rooted}$ is true at all the leaves, we can use the structure to flood all the leaves.
\begin{align}
\exists \cn{root.}\cn{Flood}\sqsubseteq \cn{Flood}\label{speadflood2}
\end{align}

With the following inclusions we produce a further subgoal $\cn{Success}$ that we want to satisfy at each element in $\cn{L}_n^{(C)}$ (the leaves of the three), for all $C\in \{R,G,B\}$. 
\begin{align}
    \cn{L}_{n}^{(C)}\sqsubseteq \exists \cn{seeTree}\\
    \exists \cn{seeTree}. \leaft\sqsubseteq \cn{Success}\label{success}
\end{align}
We will require that $\cn{Success}$ is true at all the leaves of the color trees. In a minimal model, the latter requirement is satisfied iff each element in $\cn{L}_n^{(C)}$ has a $\cn{seeTree}$-connection to some element in $\cn{L}_{6n+2}^{(3)}$. 
We use such connections to detect if the predicate $\cn{Flood}$ is satisfied. In such situation, we mark all the nodes in $\cn{L}_n^{(C)}$ with $\cn{Chosen}$. Intuitively speaking, saturating the color trees, we have that for each $\cn{L}_n^{(1)}$, all the colors are chosen.
\begin{align}
    \exists \cn{seeTree. Flood}\sqsubseteq \cn{Chosen}\label{choiceflood}
\end{align}
We collected a large family of concepts that ideally we want to be true at all the leaves of one of the constructed trees. To do so, we first propagate them to the roots. We achieve the latter via the following axioms
\begin{align}
\cn{Success} \sqcap \cn{C}&\sqsubseteq \cn{Goal}_C\;\; C\in \{R,G,B\}\label{goalc}\\
\cn{Col} &\sqsubseteq \cn{Goal_1} \label{goalone}\\
\cn{Ass}&\sqsubseteq \cn{Goal_2}\label{goaltwo}\\
\cn{C^{(1)}\sqcap}\cn{Ok}_{(u,v)}\sqcap \cn{C^{(2)}}\sqcap \cn{Ok}_{(x,y)}\sqcap \cn{Rooted} &\sqsubseteq \cn{Goal_3}\label{goalthree}\\
\exists \cn{r}_i^{(j)}. \cn{Goal}_j \sqcap \exists\cn{l}_i^{(j)}. \cn{Goal_j}&\sqsubseteq \cn{Goal}_j\;\;j=1,2,3\label{propgoal}\\
\exists \cn{r}_i^{(C)}. \cn{Goal}_C \sqcap \exists \cn{l}_i^{(C)}. \cn{Goal}_C&\sqsubseteq \cn{Goal}_C\;\;C\in \{R,G,B\}\label{propgoalc}
\end{align}
We now ask that all such goal true at different domain elements generate a `done with this task' concept.
\begin{align}
	\cn{Root}_C \sqcap \cn{Tree}_C \sqcap \cn{Goal}_C\sqsubseteq \exists \cn{done}_C. \cn{D}_C\;\;C\in \{R,G,B\}\\
    \cn{Root_1\sqcap Tree_1 \sqcap {Goal}_1} \sqsubseteq \exists \cn{done_1. D_1}\label{fromroo1}\\
    \cn{Root_2\sqcap Tree_2\sqcap Goal_2} \sqsubseteq \exists \cn{done_2. D_2}\label{fromroot2}\\
    \cn{Root_3\sqcap Tree_3 \sqcap Goal_3 \sqcap Flood}\sqsubseteq \exists \cn{done_3. D_2}\label{fromroot3}\\
    (\bigsqcap_{C\in \{R,G,B\}} \cn{D}_C)\sqcap\cn{ D_1\sqcap D_2\sqcap D_3\sqsubseteq Final\_Goal}\label{final}
\end{align}

\begin{lemma}\label{claim:cert3col}
    $G$ is a yes-instance of \coccol iff $\cn{Final\_Goal}$ is satisfiable in a minimal model of $\Kmc$.
\end{lemma}
\begin{proof}[Proof of \ref{claim:cert3col}]

\emph{(Only If).} Assume that $G$ is a yes-instance of \coccol. Let $t$ be a truth assignment, such that the induced subgraph is not 3-colorable. We construct a minimal model $\I$ of $\K$ such that $\cn{Final\_Goal}^{\I}\not =\emptyset$. 
First, we define the domain of $\I$ as follows:
\begin{align*}
\Delta^{\I}&=\{e\}\cup \{o, o', r_i, r_C, r, g, b, v_1, v_2\vert C\in\{R,G,B\},\,i\in \{1,2,3\}\}\cup\\
&\bigcup_{1\leq l\leq n} \{n_{l,k_l}^{(\star)}\vert\, \star\in \{1,R,G,B\}, \,  0\leq k_l\leq 2^l-1\}\cup \\
&\bigcup_{1\leq l\leq 2n} \{n_{l,k_l}^{(2)}\vert\, 0\leq k_l\leq 2^l-1\}\cup\\
&\bigcup_{1\leq l\leq 6n+3} \{n_{l,k_l}^{(3)}\vert\, 0\leq k_l\leq 2^l-1\}.
\end{align*}

We now define the interpretation function as follows.
\begin{align*}
&c^{\I}=c\qquad\text{for all}\quad c\in \{o,o',r_i,r_c,r,g,b,v_1,v_2\vert C\in \{R,G,B\},\,i\in \{1,2,3\}\}\\
&\cn{Root_i}^{\Imc}=\{r_i\}\qquad \{i\in \{1,2,3,R,G,B\}\}\\
&\cn{Left}^{\Imc}=\{o\}\quad\cn{Right}^{\Imc}=\{o'\}\quad  \cn{R_A}^{\Imc}=\{r\} \quad \cn{G_A}^{\Imc}=\{g\}\quad \cn{B_A}^{\Imc}=\{b\}\\
&\cn{T}^{\Imc}=\{v_1\}\qquad \cn{F}^{\Imc}=\{v_2\}
\end{align*}
We construct the interpretation function in a way such that $\Imc$ contains 5 threes rooted in each $r_i$, each of them with $\mathrm{depth}_i$ such that: for $i\in \{1,R,G,B\}$, $\mathit{depth}_i=n\}$; for $i=2$, $\mathit{depth}_i=2n$; and for $i=3$, $\mathit{depth}_i=6n+2$.
\begin{align*}
&(\cn{l}_0^{(i)})^{\Imc}=\{(r_i, n_{1,1}^{(i)})\}\qquad (\cn{r}_0^{(i)})^{\Imc}=\{(r,n_{1,2}^{(i)})\}\\
       & (\cn{l}_j^{(i)})^{\Imc}=\bigcup_{0\leq k\leq 2^{j}-1}\{(n_{j,k}^{(i)}, n_{j+1,k'}^{(i)})\vert k'=2k\} \text{ for all }0\leq j\leq \mathit{depth}_i-1\\
       & (\cn{r}_j^{(i)})^{\Imc}=\bigcup_{0\leq k\leq 2^{j}-1}\{(n_{j,k}^{(i)}, n_{j+1,k'}^{(i)})\vert k'=2k+1\}\text{ for all }0\leq j\leq  \mathit{depth}_i-1\\
       & (\cn{L}_j^{(i)})^{\Imc}=\{n_{j,k}^{(i)}\vert 1\leq k\leq 2^{j}-1\}\text{ for all }0\leq j \leq \mathit{depth}_i\\ 
       &({\cn{\lambda}_j^{(i)}})^{\Imc}=(\cn{L}_{2n-j}^{(i)})^{\Imc}\text{ for all }0\leq j\leq \mathit{depth}_i\\
       		& \cn{pick}^\Imc= \{(n_{j,k}^{(i)},o)\vert\,k\text{is even}\}\cup \{(n_{j,k},o')\vert\,k\text{is odd}\}\\
       &{(\cn{s}_j^{(i)})}^{\Imc}= ((r_{\mathit{depth}_i-(j+1)}^{(i)})^-)^{\Imc}\text{ for all }0\leq j\leq \mathit{depth}_i\\
        &\cn{Tree}_i^{\Imc}=\cn{Root}_i^{\Imc}
\end{align*}

We define the extension of the concepts $\cn{B}_j^{(i)}$ and $\cn{\bar{B}}_j^{(i)}$, with $0\leq j<\mathit{depth}_i-1$, for each $i\in \{1,2,3,R,G,B\}$ as follows:
\begin{align*}
(\cn{B}_j^{(i)})^{\Imc}=\{n_{l,k_l}^{(i)}\vert bit_j(k_l)=1,\,0\leq l\leq \mathit{depth}_i,\,0\leq k_l\leq 2^l-1\}\\
(\cn{\bar{B}}_j^{(i)})^{\Imc}=\{n_{l,k_l}^{(i)}\vert bit_j(k_l)=1,\,0\leq l\leq \mathit{depth}_i,\,0\leq k_l\leq 2^l-1\}\\
\end{align*}

We omit here the definition of the extension function of concepts $\cn{U}_i,\, \cn{\bar{U}}_i,\, \cn{V}_i,\, \cn{\bar{V}}_i$, $\cn{P}_i,\,\cn{\bar{P}}_i$, for all $i$, since they can be defined simply by selecting the corresponding subsets of the extensions of the concepts $\cn{B}_i$ and $\cn{\bar{B}}_i$. Observe that the former concepts serve more as auxiliary predicates in the axioms description and can completely be replace by the concepts $\cn{B}_i$ and $\cn{\bar{B}}_i$ (adjusting the indexes).
\begin{align*}
&\cn{pick\_col}^{\Imc}=\{(n_{n,k}^{(C)}, C)\vert 0\leq k\leq 2^n-1, C\in \{R,G,B\}\}\\
&\cn{C}^{(\Imc)}=\{n_{n,k}^{(C)}\vert 0\leq k\leq 2^n-1\}\\
&\cn{col}^{(\Imc)}=\{(n_{n,k}^{(1)}, n_{n,k}^{(C)}\vert0\leq k\leq 2^n-1\}\\
&\cn{Chosen}^{\Imc}=\{n_{l,k}^{(i)}\vert l=\mathit{depth}_i, 0\leq k\leq 2^l-1, i\in \{R,G,B\}\}\\
&{\cn{C'}}^{\Imc}=C^{\Imc}\cup \{n_{n,k}^{(1)}\vert 0\leq k\leq 2^n-1\}\text{ for all }C\in \{R,G,B\}\\
&\Sigma^{\Imc}= \{n_{n,k}^{(1)}\vert 0\leq k\leq 2^n-1\}\text{ for all }\Sigma\in \{\cn{GoodCol}_C, \cn{GoodCol}, \cn{ColCoord}_i^{(C)}, \cn{CoordMatch}^{(C)}, \cn{CoordMatch}, \cn{GoodChoice}, \cn{Col}\}\\
&\cn{seeTree}^{\Imc}=\{(n_{n,k}^{(C)}, n_{6n+2,0}^{\Imc}\vert 0\leq k\leq 2^n-1,\, C\in \{R,G,B\}\}\\
&\cn{Success}^{\Imc}=\{n_{n,k}^{(C)}\vert 0\leq k\leq 2^n-1,\, C\in \{R,G,B\}\}
\end{align*}
Observe that $\cn{Choosen}$ is placed at every leaf of the trees rooted in $r_C$, for all colors $C\in \{R,G,B\}$.

Analogously as above, we can now assign truth values to the elements $n_{2n, k}^{(2)}$ by `copying' the truth value of the corresponding variable. Let $\pi\colon L_{2n}^{(3)}\rightarrow Var$ be the bijection such that for each $n_{2n,k}^{(3)}$, the concepts $\cn{B}_i^{(2)}$ and $\cn{\bar{B}}_i^{(2)}$ encode in binary the pair $(i,j)$ such that $\pi(n_{2n,k}^{(3)})=v_{i,j}$. 
\begin{align*}
&\cn{v}^{\Imc}=\bigcup_{0\leq k\leq 2^{2n}-1}\{(n_{2n,k}^{(3)}, v_1)\vert t(\pi(n_{2n,k}^{(3)}))=1\} \cup \{\{(n_{2n,k}^{(3)}, v_2)\vert t(\pi(n_{2n,k}^{(3)}))=0\}\\
&\cn{Ass}^{\Imc}=\cn{L}_{2n}^{(2)}\\
&\cn{T'}^{\Imc}=\{n_{2n,k}\vert t(\pi(n_{2n,k}^{(3)}))=1\}\\
&\cn{F'}^{\Imc}=\{n_{2n,k}\vert t(\pi(n_{2n,k}^{(3)}))=0\}
\end{align*}

Observe that each $n_{6n+2,k}^{(3)}$, with $0\leq k\leq 2^{(6n+2)}-1$ with its bits encoded via the concepts $\cn{B}_i^{(3)}$ and $\cn{\bar{B}}_i^{(3)}$ represents a tuple $(u_k,v_k,x_k,y_k, \sigma^1_k, \sigma^2_k)$ with $0\leq u_k,v_k\leq 2^n-1$, $0\leq x_k,y_k<2^{2n}-1$ and $\sigma^i_k\in \{0,1\}$. 

\begin{align*}
&(\cn{p}_1^{(1)})^{\Imc}=\{(n_{6n+2,k}^{(3)}, n_{n,k'}^{(1)})\vert u_k=k'\}\\
&(\cn{p}_2^{(1)})^{\Imc}=\{(n_{6n+2,k}^{(3)}, n_{n,k'}^{(1)})\vert v_k=k'\}\\  
&(\cn{p}_1^{(2)})^{\Imc}=\{(n_{6n+2,k}^{(3)}, n_{n,k'}^{(2)})\vert \pi(n_{n,k'})^{(2)})=x_k\}\\
&(\cn{p}_2^{(2)})^{\Imc}=\{(n_{6n+2,k}^{(3)}, n_{n,k'}^{(2)})\vert \pi(n_{n,k'})^{(2)})=y_k\}\\
&({C}^{(i)})^{\Imc}=(\cn{C}^{(i)})=\cn{L}_{6n+2}^{\Imc}\text{ with }C\in \{\cn{C_1, C_2}\}\text{ and }i=1,2\\
&({\cn{Ok}}_i^{j})^{\Imc}=(\cn{Ok}^{j})=(\cn{Ok}_{(u,v)})^{\Imc}=(\cn{L}_{6n+2}^{(3)})^{\Imc}\text{ with }0\leq i\leq {n}-1\text{ and }j=1,2\\
&({\cn{Bit}}_i^{j})^{\Imc}=(\cn{Var}^{j})=(\cn{Ok}_{(x,y)})^{\Imc}=(\cn{L}_{6n+2}^{(3)})^{\Imc}\text{ with }0\leq i\leq {2n}-1\text{ and }j=1,2\\
&\cn{True}_1^{\Imc}=\{n_{6n+2,k}^{(3)}\vert t(x_k)=1\}\\
&\cn{False}_1^{\Imc}=\{n_{6n+2,k}^{(3)}\vert t(x_k)=0\}\\
&\cn{True}_2^{\Imc}=\{n_{6n+2,k}^{(3)}\vert t(y_k)=1\}\\
&\cn{False}_2^{\Imc}=\{n_{6n+2,k}^{(3)}\vert t(y_k)=0\}\\
&\cn{R}_i^{\Imc}=\cn{G}_i^{\Imc}=\cn{B}_i^{\Imc}=(L_{6n+2}^{(3)})^{\Imc}
\end{align*}
We now define the extensions of the predicates occurring in \eqref{firstcircuit}-\eqref{finalcircuit}. It is easy to observe that the extension of such concepts can be uniquely determined by looking at the `computation' of each circuit $c$ at each input $(u,v)$. The leaves of the tree with root $r_3$ correspond to all possible combinations of pairs of vertices, i.e., the inputs. The extension of involved predicates can be minimally defined by `copying' the behavior of each circuit.
For each circuit $c$:
\begin{align*}
&    {\cn{Val}_{G,c}^0}^{\Imc}=\{n_{6n+2,l}^{(3)}\vert \text{ on the input $(u,v)$ encoded in $n_{6n+2,k}^{(3)}$, the gate $G$ in $c$ outputs $0$}\}\\
&    {\cn{Val}_{G,c}^1}^{\Imc}=\{n_{6n+2,l}^{(3)}\vert \text{ on the input $(u,v)$ encoded in $n_{6n+2,k}^{(3)}$, the gate $G$ in $c$ outputs $1$}\}\\
 &   \cn{Edge}^{\Imc}=\{n_{6n+2,k}^{(3)}\in \cn{Val}^1_{G^o, c_E}\vert 0\leq k<(2^{6n+2}-1)\}\\
&\cn{Pos_1}^{\Imc}=\cn{Edge}^{\Imc}\cap {\cn{Val}_{G_0, c_{\sigma}^{(1)}}^1}^{\Imc}\\
&\cn{Neg_1}^{\Imc}=\cn{Edge}^{\Imc}\cap {\cn{Val}_{G_0, c_{\sigma}^{(1)}}^0}^{\Imc}\\
&\cn{Pos_2}^{\Imc}=\cn{Edge}^{\Imc}\cap {\cn{Val}_{G_0, c_{\sigma}^{(2)}}^1}^{\Imc}\\
&{\cn{Neg_2}}^{\Imc}=\cn{Edge}^{\Imc}\cap {\cn{Val}_{G_0, c_{\sigma}^{(2)}}^0}^{\Imc}\\
&{(\cn{{X}}_j^c)}^{\Imc}=\cn{Edge}^{\Imc}\cap {\cn{Val}_{G_0, c_{\sigma}^{(2)}}^1}^{\Imc}\\
&{(\cn{\bar{X}}_j^c)}^{\Imc}=\cn{Edge}^{\Imc}\cap {\cn{Val}_{G_0, c_{\sigma}^{(2)}}^0}^{\Imc}\\
&{(\cn{Y}_j^c)}^{\Imc}=\cn{Edge}^{\Imc}\cap {\cn{Val}_{G_0, c_{\sigma}^{(2)}}^1}^{\Imc}\\
&{(\cn{\bar{Y}}_j^c)}^{\Imc}=\cn{Edge}^{\Imc}\cap {\cn{Val}_{G_0, c_{\sigma}^{(2)}}^0}^{\Imc}\\
&\cn{Comp}_{j,x}^{\Imc}= (\cn{X}_j^{\Imc} \cap (\cn{X}_j^c)^{\Imc})\cup (\cn{\bar{X}}_j^{\Imc} \cap (\cn{\bar{X}}_j^c)^{\Imc})\\
&\cn{Comp}_{j,y}^{\Imc}= (\cn{X}_j^{\Imc} \cap (\cn{X}_j^c)^{\Imc}) \cup (\cn{\bar{Y}}_j^{\Imc} \cap (\cn{\bar{Y}}_j^c)^{\Imc})\\
&\cn{Comp}_x^{\Imc}= \bigcap_{j=0}^{2n-1} \cn{Comp}_{j,x}^{\Imc}\qquad \cn{Comp}_y^{\Imc}= \bigcap_{j=0}^{2n-1} \cn{Comp}_{j,y}^{\Imc}\\
&\cn{Lab}^{\Imc}=\cn{Comp}_x^{\Imc} \cap \cn{Comp}_x^{\Imc}
\end{align*}
Analogously we can define $\cn{Pol}_1$ and $\cn{Pol}_2$. Let $i=1,2$:
\begin{align*}
 &\cn{Pol_i}^{\Imc}=(\cn{Pos}_i^{\Imc} \cap \cn{P}_i^{\Imc}) \cup (\cn{Neg}_i^{\Imc}   \cap \cn{\bar{P}}_i^{\Imc})\\
 &\cn{Pol}^{\Imc}=\cn{Pol}_1^{\Imc} \cap \cn{Pol}_2^{\Imc}.
\end{align*}
We can now define the extension of $\cn{isEdge}$:
\begin{align*}  \cn{IsEdge}^{\Imc}=\cn{Edge}^{\Imc}\cap \cn{Lab}^{\Imc}\cap \cn{Pol}^{\Imc}
\end{align*}
Observe that the predicates related to the circuits are mimicking the behavior of the circuits.
\begin{align*}
&\cn{TrueEdge_1}^{\Imc}=\cn{IsEdge}^{\Imc} \cap ((\cn{Neg}_1\cap \cn{False}_1)\cup (\cn{Pos}_1\cap \cn{True}_1))\\
&\cn{TrueEdge_2}^{\Imc}=\cn{IsEdge}^{\Imc} \cap ((\cn{Neg}_2\cap \cn{False}_2)\cup (\cn{Pos}_2\cap \cn{True}_2))\\
&\cn{TrueEdge}^{\Imc}=\cn{TrueEdge}_1^{\Imc}\cap \cn{TrueEdge}_2^{\Imc}
\end{align*}
The elements $n_{6n+2,k}^{(3)}$ belonging to $\cn{TrueEdge}$ are exactly those leaves encoding a tuple $(u_k,v_k,x_k,y_k,\sigma^1_k,\sigma^2_k)$ representing labeled edges in $G$ that evaluates to true under the truth assignment $t$.

We are left with connecting all the leaves of the third tree with its root $r_3$.
\begin{align*}
&\cn{root}^{\Imc}=\{(n_{6n+2,k}^{(3)}, r_3)\vert 0\leq k<2^{6n+2}-1\}\\
&\cn{Rooted}^{\Imc}= (\cn{L}_{6n+2}^{(3)})^{\Imc}
\end{align*}
We define the extension of the `flooding concept'.
\begin{align*}
&\cn{Flood}^{\Imc}=\{n_{l,k_l}^{(3)}\vert 0\leq l\leq 6n+2,\,0\leq k\leq 2^l-1\}\\
\end{align*}
We can now define the extension of goal concepts reaching the final goal. Let $\mathit{depth}_\star=n$, with $\star\in\{1,R,G,B\}$, $\mathit{depth}_2=2n$ and $\mathit{depth}_3=6n+2$.
\begin{align*}
&\cn{Goal}_i^{\Imc}=\{n_{l,k}^{(i)}\vert 0\leq l\leq \mathit{depth}_i, 0\leq k<2^l-1\}\text{ for all }i\in \{1,2,3,R,G,B\}\\
&\cn{done}_i^{\Imc}=\{(r_i, e)\}\text{ for all }i\in \{1,2,3,R,G,B\}\\
&\cn{Final\_Goal}^{\Imc}=\cn{D}_i^{\Imc}=\{e\}\text{ for all }i\in \{1,2,3,R,G,B\}
\end{align*}

The constructed interpretation $\I$ is a model of $\K$. The latter can be checked rule by rule. Furthermore, the elements $n_{6n+3,k}^{(3)}$ with their set of satisfied concept names encode all the edges of $G$ that are true under the truth assignment $t$. It easy to observe that:
\begin{itemize}
\item each leaf in the first tree is connected exactly with one element in $\cn{L}_n^{(C)}$, for each $C\in \{R,G,B\}$ that furthermore matches its coordinates;
\item all elements in $\cn{L}_n^{(C)}$ satisfy a unique concept $\cn{C}$, with $\cn{C}\in\{\cn{R,G,B}\}$;
\item each leaf in the first tree sees at least one node labeled with $\cn{Chosen}$ in each of the \emph{color trees};
\item each leaf in the second tree has a unique truth assignment encoded by the concept $T'$ and $F'$, defined by $t$;
\item since the membership in $T'$ and $F'$ is uniquely determined by $t$, each leaf in the second tree is either in $T'$ or in $F'$;
\item each leaf in the color tree is connected to exactly one node in $\cn{L}_{6n+2}^{(3)}$;
\item each leaf in the third tree is connected to the root $r_3$ via the role $\cn{root}$.
\end{itemize}
Recalling that each element in $L_{6n+2}^{(3)}$ encodes a tuple $(u,v,x,y, \sigma^1, \sigma^2)$, it is easy to observe that:
\begin{itemize}
\item each leaf in the third tree has a unique $\cn{p}_1^{(1)}$-successor and a unique $\cn{p}_2^{(2)}$-successor in the leaves of the first tree that are matching his $(u,v)$ coordinates
\item each leaf in the third tree has a unique $\cn{p}_2^{(1)}$-successor and a unique $\cn{p}_2^{(2)}$-successor in the leaves of the second tree that are matching his $(x,y)$ coordinates.
\end{itemize}
For the circuit related predicates, since their extension in uniquely determined by the outputs of the circuits, it is easy to verify that they cannot be further minimized. 

From the previous observations, we have that the extensions of all roles are minimal. Furthermore, from the previous observations, it is easi check that all concepts $\cn{Root}_i$, $\cn{Tree}_i$, $\cn{Goal}_i$, with $i\in \{1,2,3, R,G,B\}$, $\cn{C}, \cn{Success}$, $\cn{Ass}$, $\cn{C}^{(j)}$, $\cn{Ok}_{(x,y)}$, $\cn{Ok}_{(u,v)}$, with $j\in \{1,2\}$, and $\cn{Rooted}$ are minimal.
We argue that also the concept $\cn{Rooted}$ and $\cn{Flood}$ also have a minimal extension, immediately deriving that $\I$ is a minimal model of $\K$ where $\cn{Final\_Goal}$ is satisfied. 

Observe that to minimize $\cn{Chosen}$, one first has to minimized $\cn{Flood}$ everywhere in the third tree (due to \eqref{Rooted}-\eqref{spreadflood}-\eqref{speadflood2}-\eqref{choiceflood}). 
Assume there exists $\Jmc\subseteq \Imc$, then $\cn{Flood}^{\Jmc}=\emptyset$. W.l.o.g. we can assume that $\cn{Chosen}^{\Jmc}$ is minimal, i.e. there exists $\J'$ such that $\J'\subseteq \I$ and $\cn{Chosen}^{\Jmc'}\subset \cn{Chosen}^{\Jmc}$. 
From $\Jmc$ is easy to derive a color assignment $\chi$ for $G$. Let $\chi(v)=C$ iff given element $n_{n,k}^{(1)}$ such that $k=v$, there exists $d\in (\cn{L}_{n}^{(C)})^{\Jmc}$ such that $(n_{n,k}^{(1)}, d)\in \cn{col}^{\Jmc}$ and $d\in \cn{Chosen}^{\Jmc}$. Since $\cn{Flood}^{\Jmc}=\emptyset$, it is never the case that, in the sub-graph induced by $t$, there are two nodes $(u,v)$ such that $\chi(u)=\chi(v)$. Furthermore, $\chi$ is well-defined since each $n_{n,k}^{(1)}$ has a unique $d\in \cn{L}_{n}^{(C)}$ for each $C\in \{R,G,B\}$ and $\cn{Chosen}^{\Jmc}$ is minimal. We derived a contradiction. 

\emph{(If).}
Before proceeding in the details of the proof, let us remark that we follow the same proof strategy of Theorem \ref{nexp:el:strong:ac}: we show that the satisfaction of goal concepts at the root of the trees transfers to the satisfaction of the goal concepts at the roots. From the latter, we show how we derive our desiderata. In particular, we use often the minimality of $\I$ to derive that for each inclusion $C\sqsubseteq D$, if $d\in D^{\I}$, then $d\in C^{\I}$. 

Assume that $\cn{Final\_Goal}$ is satisfiable in a minimal model $\Imc$ of $\Kmc$, i.e. there exists $e\in \Delta^{\Imc}$ such that $e\in \cn{Final\_Goal}^{\Imc}$. Since $\I$ is minimal, from \eqref{final}, $e\in (\bigsqcap_{C\in \{R,G,B\}} \cn{D}_C\sqcap \cn{D_1\sqcap D_2\sqcap D_3})^{\Imc}$. Indeed, if it is not the case, a smaller model $\Jmc$ can be obtained from $\Imc$ by minimizing the concept $\cn{Final\_Goal}$ at $e$.
Hence $e\in \cn{D}_i^{\Imc}$, for all $i=1,2,3$ and $e\in \cn{D}_C$, for all $C\in \{R,G,B\}$.
Observe that since each concept $\cn{Root}_i$ and $\cn{Root}_C$, with $i=1,2,3$ and $C\in \{R,G,B\}$ does not occur on the right-hand side of any of the inclusions in $\Kmc$, from the minimality of $\Imc$, it follows that $\cn{Root}_i^{\Imc}=\{r_i^{\Imc}\}$, with $i=1,2,3$, and $\cn{Root}_C^{\Imc}=\{r_C\}$, with $C\in \{R,G,B\}$.

\begin{claimrep}\label{rootsconcepts}
Since $\I$ is minimal and $e\in \cn{D}_i^{\Imc}$, the following hold:
\begin{itemize}
	\item  $(r_1^{\Imc},e^{\I})\in \cn{done_1}^{\Imc}$ and $r_1\in (\cn{Root_1\sqcap Tree_1\sqcap Goal_1})^{\Imc}$,
	\item $(r_2^{\Imc},e^{\I})\in \cn{done_2}^{\Imc}$ and $r_2\in (\cn{Root_2\sqcap Tree_2\sqcap Goal_2})^{\Imc}$,
	\item $(r_3^{\Imc},e^{\I})\in \cn{done_3}^{\Imc}$ and $r_3\in (\cn{Root_3\sqcap Tree_3\sqcap Goal_3}\sqcap \cn{Flood})^{\Imc}$,
	\item  $(r_C^{\Imc},e^{\I})\in \cn{done_C}^{\Imc}$ and $r_1\in (\cn{Root_C\sqcap Tree_C\sqcap Goal_C})^{\Imc}$, for all $C\in \{R,G,B\}$.
\end{itemize}
\end{claimrep}
\begin{proof}
We show the claim for the first item in the list. The other two cases can be treated analogously. 
From the minimality of $\I$, since $e\in \cn{D}_1^{\Imc}$, it follows that there exists $d\in \Delta^{\Imc}$ such that $(d,e)\in \cn{done_1}^{\Imc}$ and $d\in (\cn{Root_1\sqcap Tree_1\sqcap Goal_1})^{\Imc}$. Similarly as above, if this is not the case, a smaller model $\Jmc$ can be obtained by minimizing the concept $\cn{D}_1$ in $e$.
We show that $d=r_1^{\Imc}$. We have already observed that $\cn{Root_1}^{\Imc}=\{r_1^{\Imc}\}$. Since $(\cn{Root_1\sqcap Tree_1\sqcap Goal_1})^{\Imc}\subseteq \cn{Root_1}^{\Imc}=\{r_1^{\Imc}\}$, it follows that $d=r_1^{\Imc}$. Thus the thesis. 
\end{proof}

We can now turn our attention to the elements $r_i^{\Imc}$, with $i=1,2,3$. From Claim \ref{rootsconcepts}, it follows that $r_i^{\Imc}\in \cn{Tree}_i^{\Imc}$, for all $i=1,2,3$. We show that each $r_i$ is at the root of a tree with exponentially many leaves. 

\begin{claimrep}\label{exponential:leaves}
Since $\I$ is minimal and each ${\cn{r}_i}^{\Imc}\in \cn{Tree}_i^{\Imc}$ the following hold:
\begin{itemize}
\item $\vert (\cn{L}_{n}^{(1)})^{\Imc}\vert =2^n$,
\item $\vert (\cn{L}_{n}^{(C)})^{\Imc}\vert =2^n$, with $C\in \{R,G,B\}$,
\item $\vert (\cn{L}_{2n}^{(2)})^{\Imc}\vert = 2^{2n}$,
\item $\vert (\leaft)^{\Imc}\vert = 2^{6n+3}$.
\end{itemize}
\end{claimrep} 

\begin{proof}
\newcommand{\rone}{r_1^{\Imc}}
We show the result for the first item, the other two cases can be proved analogously. Since $r_1^{\Imc}\in \cn{Root}_1^{\Imc}$, from the minimality of $\Imc$, we have that $\rone \in (\lambda_{n,s}^{(1)}\sqcap \lambda_{n,d}^{(1)})^{\Imc}$. Thus, from the fact that $\Imc$ is minimal there exists $e_s, e_d\in \Delta^{\Imc}$ such that:
\begin{itemize}
\item $(e_s, \rone)\in (\cn{s_n}^{(1)})^{\Imc}$ and $e_s\in (\lambda_{n-1}^{(1)}\sqcap \exists \cn{pick. Left})^{\Imc}$, and
\item $(e_d, \rone)\in (\cn{d_n}^{(1)})^{\Imc}$ and $e_d\in (\lambda_{n-1}^{(1)}\sqcap \exists \cn{pick. Right})^{\Imc}$.
\end{itemize}
It is easy to see that if one of the above conditions fails, a smaller model can be obtained by minimizing $\lambda_{n,s}^{(1)}$ or $\lambda_{n,d}^{(1)})^{\Imc}$ at $e$. We show that $e_s\not=e_d$. Since all predicates are minimized, each node in $\lambda_{n-1}^{(1)}$ can minimally justify only one $\cn{pick}$. If we assume that $e:=e_s=e_d$, then $e\in \exists \cn{pick}. \cn{Left}^{\Imc}$ and $e\in \exists \cn{pick}. \cn{Right}^{\Imc}$. Since $\cn{Left}$ and $\cn{Right}$ only occur on the left-hand side of axioms, we have that $\cn{Left}^{\Imc}=\{a^{\Imc}\}$ and $\cn{Right}^{\Imc}=\{b^{\Imc}\}$. Thus, $(e,a^{\I})\in \cn{pick}^{\Imc}$ and $(e,b^{\Imc})\in \cn{pick}^{\Imc}$. From out previous observation, a smaller model can be obtained by minimizing $\cn{pick}$ at $(e,a^{\Imc})$ or $(e,b^{\Imc})$. 

The above discussion can be applied at each level of the tree, showing each $e\in (\lambda_{i}^{(1)})^{\Imc}$ has exactly one $\cn{s}_{i-1}^{(1)}$ predecessor and exactly one $\cn{d}_{i-1}^{(1)}$ predecessor. Since each node at level $i-1$ can produce either a $\cn{s}_{i-1}^{(1)}$ successor or a $\cn{d}_{i-1}^{(1)}$ (since $\cn{pick}$ is minimized), we have that $\vert (\lambda_{i-1}^{(1)})^{\Imc}\vert = 2\cdot \vert (\lambda_{i}^{(1)})^{\Imc}\vert$, for all $0\leq i\leq n$. Since $\vert (\lambda_{n}^{(1)})^{\Imc}\vert =1$, by (reverse) induction on $i$, we have that $\vert (\lambda_{0}^{(1)})^{\Imc}\vert=2^n$. Since $\I$ is minimal, we have that $ (\lambda_{0}^{(1)})^{\Imc}=(\cn{L}_n^{(1)})^{\Imc}$. 
Thus, $\vert (\cn{L}_n^{(1)})^{\Imc}\vert \geq 2^n$.
To derive the thesis, it is sufficient to observe that each instance of $\cn{L}_n^{(1)}$ can only be generate by the tree that can be constructed starting from $\cn{Root_1}$ using the axioms \eqref{root1}-\eqref{tree1gen}. Since $\cn{Root_1}^{\Imc}=\{r_1^{\Imc}\}$, one can easily observe that $\vert (\cn{L}_n^{(1)})^{\Imc}\vert = 2^n$, i.e. $r_1^{\Imc}$ is the root of a \emph{full} binary tree of depth $n$, with $2^n$ leaves.
\end{proof}
To ease the following part of the proof, we define the sets of leaves of each tree as follows:
\begin{itemize}
\item the leaves of the \emph{vertex tree}, $\leafone= (\cn{L}_{n}^{(1)})^{\Imc}$,
\item the leaves of the \emph{variable tree}, $\leaftwo=(\cn{L}_{2n}^{(2)})^{\Imc}$,
\item the leaves of the \emph{circuit tree}, $\leafthree=(\cn{L}_{6n+2}^{(3)})^{\Imc}$,
\item the leaves of the \emph{red tree}, $\leafred=(\cn{L}_{n}^{(R)})^{\Imc}$,
\item the leaves of the \emph{green tree}, $\leafgreen=(\cn{L}_{n}^{(G)})^{\Imc}$,
\item the leaves of the \emph{blue tree}, $\leafblue=(\cn{L}_{n}^{(B)})^{\Imc}$.
\end{itemize}

From the minimality of $\Imc$ each $d\in \leafi$ (resp. each $d\in \leafc$) uniquely correspond to a natural number coded in binary using the concepts $\cn{B}_j^{(i)}$, $\cn{\bar{B}}_j^{(i)}$ and (resp. each $\cn{B}_j^{(C)}$, $\cn{\bar{B}}_j^{(C)}$).

\begin{claimrep}\label{goal-at-leaf}
For each $i\in \{1,2,3\}$, $d\in \cn{Goal}_i^{\Imc}$ for all $d\in \leafi$. Similarly, for each $C\in \{R,G,B\}$, $d\in \cn{Goal}_C^{\Imc}$ for all $d\in \leafc$. 
\end{claimrep}

\begin{proof}
We show it for $\cn{Goal}_1$, the other cases are analogous. We use the \emph{propagation axioms} \eqref{propgoal}. Since $r_1^{\Imc}\in (\cn{Root}_1\sqcap \cn{Tree}_1\sqcap \cn{Goal_1})^{\Imc}$, as argued in the proof of claim \ref{exponential:leaves}, there exist two distinct $e_s$ and $e_d$ such that 
\begin{itemize}
\item $(e_s, r_1^{\Imc})\in (s_n^{(1)})^{\Imc}$ and $e_s\in (\lambda_{n-1}^{(1)})^{\Imc}$, and 
\item $(e_d, r_1^{\Imc})\in (s_n^{(1)})^{\Imc}$ and $e_d\in (\lambda_{n-1}^{(1)})^{\Imc}$.
\end{itemize}
Furthermore, in Claim \ref{exponential:leaves}, we proved that $\vert (\lambda_{n-1}^{(1)})^{\Imc}\vert = 2$. Hence, since $\Imc$ is minimal, we have that $e_s,e_d\in \cn{Goal}_1^{\Imc}$. In fact, if this is not the case, the left-hand side of \eqref{propgoal} is not satisfied. Thus, a smaller model $\Jmc$ for $\Kmc$ can be obtained by minimizing $\cn{Goal}_1$ at $r_1$, deriving a contradiction. By iterating the latter discussion to each level of the tree, we derive that each $d\in \leafone$ must satisfy the concept $\cn{Goal}_1$, i.e. $d\in \cn{Goal}_1^{\Imc}$. 
\end{proof}
With Claim \ref{goal-at-leaf}, we proved that all the leaves satisfy a certain goal concept. Recall that in the construction of the KB, we encoded with these goals a family of conditions that the leaves of the trees must satisfy. We discuss the effect of each of them on each of the trees.

\emph{Color Trees}
With the next claims, we show relevant properties satisfied by the three color trees. Since each $d\in \leafc$ is such that $d\in \cn{Goal}_C^{\Imc}$, from the minimality of $\I$, we have that $d\in \cn{Success}^{\Imc}$ and $d\in \cn{C}^{\Imc}$. Otherwise, a smaller model for $\K$ can be constructed from $\I$ my minimizing $\cn{Goal}_C$ at one element $d\in \leafc$, deriving a contradiction. 

We can furthermore prove the following claim which states that the color trees do not share any of their leaves.
\begin{claimrep}\label{disjoint}
For each $C_1,C_2\in \{R,G,B\}$ with $C_1\not =C_2$,  $\mathrm{Leaf}_{C_1}\cap \mathrm{Leaf}_{C_2}=\emptyset$. 
\end{claimrep} 
\begin{proof}
W.l.o.g. assume $C_1=R$ and $C_2=G$. By a way of contradiction assume that there exists $d\in \leafred\cap \leafgreen$. From Claim \ref{goal-at-leaf}, we have that $d\in \cn{Goal}_R$ and $d\in \cn{Goal}_G$. Since $\Imc$ is minimal, we have that $d\in \cn{R}^{\Imc}$ and $d\in \cn{G}^{\Imc}$ (see \eqref{goalc}). From the minimility of $\I$ the left-hand side of axiom \eqref{picked} must be true, i.e. we have that $d\in (\exists \cn{pick\_col}. \cn{R_A})^{\Imc}$ and $d\in (\exists \cn{pick\_col}. \cn{G_A})^{\Imc}$.
Since $\cn{R_A}^{\Imc}=\{r^{\Imc}\}$ and $\cn{G_A}^{\Imc}=\{g^{\Imc}\}$, we have that $(d,g^{\Imc}), (d,r^{\Imc})\in \cn{pick\_col}^{\Imc}$. It easy to see that the interpretation $\Jmc$ obtained from $\Imc$ by minimizing $\cn{pick\_col}$ at $(d,r^{\Imc})$ is a model of $\K$. Since we only modified the extension of $\cn{pick}$, it is sufficient to observe that \eqref{pickcol} is satified. Since $\Jmc\subset \Imc$, we derived a contradiction.
\end{proof}

As a direct consequence of the proof of Claim \ref{disjoint}, we have that given $\cn{C_1}\in \{R,G,B\}$, for each $d\in \mathrm{Leaf}_{C_1}^{\Imc}$, $d\in (\cn{C_1}')^{\Imc}$ and $d\not \in \cn{C_2}^{\Imc}$, for each $\cn{C_2\in \{R,G,B\}}$ with $\cn{C_2\not =C_1}$.

We can prove the following claim.
\begin{claim}\label{successclaim}
For each $C\in \{R,G,B\}$, given any $d\in \leafc$, there exists $d'\in \leafthree$ such that $(d,d')\in \cn{seeTree}^{\Imc}$. 
\end{claim}
\begin{proof}
Given $d\in \leafc$, we observed that $d\in \cn{Success}^{\Imc}$. From the minimality of $\I$, the left-hand side of \eqref{success} must be true, i.e. we have that $d\in (\exists \cn{seeTree}. \leafthree)^{\Imc}$. Indeed, if for some $d\in \leafc$, we have that $d\not \in (\exists \cn{seeTree}. \leafthree)^{\Imc}$, the predicate $\cn{Success}$ can be minimized at $d$, preserving the satisfaction of axioms. Thus the thesis. 
\end{proof}

\emph{First tree \& First Goal}. We show that satisfaction of $\cn{Goal}_1$ at each element of $\leafone$ implies that each such a leaf picked a color form the color trees. From the minimality of $\Imc$ and axiom \eqref{goalone} we can deduce that each $d\in \cn{Goal}_1$ is such that $d\in \cn{Col}$. Indeed, if one assumes it is not the case, a smaller model can be obtained from $\Imc$ by minimizing $\cn{Goal_1}$ at $d$. 

Similarly, using the minimality of $\Imc$ and \eqref{coljust}, we can deduce that each $d\in \leafone$ is such that $d\in (\cn{GoodCol}\sqcap\cn{GoodChoice}\cn{CoordMatch})^{\Imc}$.
Therefore, $d\in \cn{GoodCol}^{\Imc}$, $d\in \cn{GoodChoice}^{\Imc}$ and $d\in \cn{CoordMatch}^{\Imc}$.

\begin{claim}\label{perfect-coloring}
For each $d\in \leafone$ we show the following.
\begin{enumerate}[(i)]
	\item There exists a unique $d_C\in \leafc$ for each $C\in \{R,G,B\}$ such that $(d,d_C)\in \cn{col}^{\Imc}$ and for each pair of colors $C_1, C_2$, $d_{C_1}\not = d_{C_2}$. Furthermore, for some $C\in \{R,G,B\}$, $d_C\in \cn{Chosen}^{\Imc}$.
	\item There is no $d'\not =d$ such that $(d,d_C)\in \cn{col}^{\Imc}$ and $(d',d_C)\in \cn{col}^{\Imc}$.
\end{enumerate} 
\end{claim}
\begin{proof}
Assume $d\in \leafone$.

\emph{(i)}. Since $d\in \cn{GoodCol}^{\Imc}$, from the minimality of $\Imc$ and \eqref{colorconnection}, we derive that $d\in \cn{GoodCol}_C^{\Imc}$, for each $C\in \{R,G,B\}$. 
Since each predicate must be justified, from the minimality of $\I$ it follows that for each $\cn{C}\in \{R,G,B\}$, $d\in \exists \cn{col}. (\cn{C}'\sqcap \cn{C})^{\Imc}$. 
Observe that, from \eqref{performchoice}, for each $d\in \leafone$, there exists $d_C$ such that $(d,d_C)\in \cn{col}^{\Imc}$ and $d_C\in ({\cn{C}'})^{\Imc}$, for all $C\in\{R,G,B\}$.
Since $d\in \exists \cn{col}. (\cn{C}'\sqcap \cn{C})^{\Imc}$, then from the minimality of $\I$ and Claim \ref{disjoint} we have that $d_C\in \leafc$, for all $C\in \{R,G,B\}$. 
Given two different colors, $C_1$ and $C_2$, the fact that $d_{C_1}\not =d_{C_2}$ immediately follows from the above observations and Claim \ref{disjoint}. The uniqueness trivially follows from the minimality of $\I$.

We now show that for some $C$, $d_C\in \cn{Chosen}^{\Imc}$. We have already observed that $d\in \cn{GoodChoice}^{\Imc}$. Hence, by \eqref{choose}, we have that for some $C\in \{R,G,B\}$, $d\in \exists \cn{col}. (\cn{Chosen} \sqcap\cn{C})^{\Imc}$ (the latter can be argued using the minimality of $\Imc$, as done previously). 
Observe that from the minimality of $\I$, all the occurrences of $\cn{C}^{\Imc}$ are at the elements of $\leafc$, hence there exists $d'\in \leafc$ such that $(d,d')\in \cn{col}^{\Imc}$ and $d'\in\cn{Chosen}^{\Imc}$, for some $C\in \{R,G,B\}$. Since $d_C$ is the unique element of $\leafc$ such that $(d,d_c)\in \cn{col}^{\Imc}$ and $d_C\in \cn{C}^{\Imc}$, we have that $d'=d_C$.

\emph{(ii)}. Assume by a way of contradiction, that there exists $d'\in \leafone$ such that $(d',d_C), \in \cn{col}^{\Imc}$ for some color $C$.
Using the minimality of $\Imc$ and the fact that $d\in \cn{CoordMatch}^{\Imc}$, using the axioms \eqref{checkcoordn}-\eqref{checkcoordp}-\eqref{coordmatch}, it is easy to check that the bit encoded in $d$ must match the bit encoded in $d'$, deriving a contradiction.

Indeed, for all $d\in \cn{CoordMatch}^{\Imc}$, the minimality of $\I$ combined with the axioms implies that $d\in (\cn{ColCoord}_i^{(C)})^{\Imc}$ which implies that the bit of coordinates encoded in $d$ matches with the bit of coordinates encoded in $d_C$. Hence, the bits of $d,d'$ must coincide with those of $d_i$, a contradiction.
\end{proof}

\emph{Second Tree \& Second Goal}
The satisfaction of $\cn{Goal}_2$ at each element $\leaftwo$ implies that each node (corresponding to a binary encoding of a variable $v_{i,j}$, in particular of $i,j$) corresponds to a truth value.

First, let us observe that Claim \ref{exponential:leaves} and \eqref{bittwo}-\eqref{bitwofinal} ensure that each node in $\leaftwo$ has a unique bit encoding.

Given $d\in \leaftwo$, since $d\in \cn{Goal}_2^{\Imc}$ (see Claim \ref{goal-at-leaf}), from the minimality of $\I$ and \eqref{goaltwo}, $d\in \cn{Ass}^{\I}$.

Furthermore, observe that since $V$ and $F$ occur as assertions and never on the right-hand side of the axioms. Since $\Imc$ is minimal, we have that $\cn{T}^{\Imc}=\{v_1^{\Imc}\}$ and $\cn{F}^{\Imc}=\{v_2^{\Imc}\}$. 
\begin{claim}
For each $d\in \leaftwo$, either $(d,v_1^{\Imc})$ and $d\in (\cn{T}')^{\Imc}$ or $(d,v_2^{\Imc})$ and $d\in (\cn{F}')^{\Imc}$.
\end{claim}
\begin{proof}
Since $d\in \cn{Ass}^{\Imc}$, from \eqref{pickedT}-\eqref{pickedF} and the minimality of $\I$, we have that $d\in (\exists \cn{v.T})^{\Imc}$ or $d\in (\exists \cn{v.F})^{\Imc}$. From the previous observation, $(d,v_1^{\Imc})\in \cn{v}^{\Imc}$ or $(d,v_2^{\Imc})\in \cn{v}^{\Imc}$. If both occur, we violate the minimality of $\Imc$: a smaller model can be obtained by minimizing $\cn{v}$ at $(d,v_1)$.
We obtained the thesis.
\end{proof}

\emph{Third Tree \& Third Goal}
One of our ultimate goals is to show that $d\in \cn{Flood}^{\Imc}$, for all $d\in \leafthree$. 
From \eqref{bitthree}-\eqref{bitthreefinal} and Claim \ref{exponential:leaves}, each node in $\leafthree$ uniquely encodes in binary a tuple $(u,v,x,y,\sigma_1, \sigma_2)$, using $6n+2$ bits. 
Using the minimality of $\Imc$ and \eqref{goalthree}, we derive that $d\in (\cn{C}^{(1)}\sqcap \cn{Ok}_{(u,v)}\sqcap \cn{C}^{(2)}\sqcap \cn{Ok}_{(x,y)}\sqcap \cn{Rooted})^{\Imc}$.

\begin{claim}\label{connectionfirst}
For each $d\in \leafthree$ there exists a unique $d_1\in \leafone$ such that $(d,d_1)\in \cn{p_1}^{(1)}$ and a unique $d_2\in \leafone$ such that $(d,d_2)\in \cn{p_2}^{(1)}$. Furthermore, if $u$ and $v$ are the natural numbers encoded in binary using the concepts $\cn{U}_i, \bar{\cn{U}}_i, \cn{V}_i, \bar{\cn{V}}_i$ and $u'$ and $v'$ are the natural numebers encoded in binary at $d_1$ and $d_2$, then $u=u'$ and $v=v'$.
\end{claim}
\begin{proof}
Recall that each $d\in \leafthree$ is such that $d\in (\cn{C^{(1)}}\sqcap \cn{Ok}_{u,v})^{\Imc}$, i.e. $d\in (\cn{C^{(1)}})^{\Imc}$ and $d\in (\cn{Ok}_{u,v})^{\Imc}$. From $d\in (\cn{C^{(1)}})^{\Imc}$, using the minimality of $\I$ it is easy to derive that there exists $d_1\in \leafone$ such that $(d,d_1)\in \cn{p}_1^{(1)}$. Indeed, from \eqref{threetoone1}-\eqref{threetoone4}, for each $d\in \leafthree$ there exists $d_1\in \Delta^{\Imc}$ such that $(d,d_1)\in (\cn{p_1}^{(1)})^{\Imc}$. Since $\Imc$ is minimal, then $d_1\in \leafone$. The uniqueness of $d_1$ comes again from the minimality of $\Imc$: each node can only justify one occurrence of $\cn{p_1}^{(1)}$.
Similarly, we can prove that there exists a unique $d_2\in \leafone$ such that $(d,d_2)\in \cn{p_2}^{(1)}$.

Since $d\in \cn{Ok}_{(u,v)}$ using a similar argument as above and the previous claims, \eqref{comparex1}-\eqref{compared} together with the minimality of $\Imc$ imply that in order to minimally justify each occurrence of $\cn{Ok}_{(u,v)}$ then first two components of the tuple $(u,v,x,y,\sigma_1,\sigma_2)$ encoded in binary at $d$ must coincide with the numbers encoded in binary at $d_1$ and $d_2$ via \eqref{bitone}-\eqref{bitonefinal}. 
\end{proof}
Analogously to the previous claim, we can show that each node in $\leafthree$ is also connected to exactly two leaves of the second tree, i.e. to the set $\leaftwo$, such that the numbers encoded in binary correspond to $x$ and $y$. 

\begin{claim}\label{connectionsecond}
For each $d\in \leafthree$ there exists a unique $d_3\in \leaftwo$ such that $(d,d_3)\in \cn{p_1}^{(2)}$ and a unique $d_4\in \leafone$ such that $(d,d_4)\in \cn{p_2}^{(2)}$. Furthermore, if $u$ and $v$ are the natural numbers encoded in binary using the concepts $\cn{X}_i, \bar{\cn{X}}_i, \cn{Y}_i, \bar{\cn{Y}}_i$ and $x'$ and $x'$ are the natural numebers encoded in binary at $d_1$ and $d_2$, then $x=x'$ and $y=y'$.
\end{claim}
\begin{proof}
The proof is analogous to the proof of Claim \ref{connectionfirst}, starting from the observation that each $d\in \leafthree$ is such that $d\in (\cn{C}^{(2)}\sqcap \cn{Ok}_{(x,y)})^{\Imc}$ and using the axioms \eqref{threetotwo1}-\eqref{comparedvar}.
\end{proof}

From Claim \ref{connectionfirst} and Claim \ref{connectionfirst}, given the axioms \eqref{importcolor1}-\eqref{importval4}, each $d\in \leafthree$ is labeled with a unique tuple of concepts $(\cn{C_1, C_2,}Val_1, Val_2)$ with $\cn{C}\in \{\cn{R,G, B}\}$ and $Val_1, Val_2\in \{\cn{True_1, False_1, True_2, False_2}\}$.

Axioms \eqref{firstcircuit}-\eqref{finalcircuit}, together with the minimality of $\I$, ensure that every $d\in \leafthree$ such that $d\in \cn{isEdge}^{\Imc}$ is such that $(u,v,x,y,\sigma_1,\sigma_2)$ describes an edge on the graph. 

With axioms \eqref{truedge1}-\eqref{truedgefinal}, we mark with $\cn{TrueEdge}$ all the nodes that evaluate to true under the truth assignment copied from the leaves of the second tree.

The following claim is easy to prove
\begin{claim}\label{rooted}
For all $d\in \leafthree$, we have that $(d,r_3^{\Imc})\in \cn{root}^{\Imc}$.
\end{claim}
\begin{proof}
The proof uses essentially the same argument as the other claims, using the axioms \ref{toroot}-\ref{Rooted}.
\end{proof}

Since $r_3^{\Imc}\in \cn{Flood}^{\Imc}$, from Claim \ref{Rooted}, we have that each $d\in \leafthree$ is such that $d\in \cn{Flood}^{\Imc}$. Therefore, from Claim \ref{successclaim}, for each $C\in \{R,G,B\}$ and each $d\in \leafc$, from \eqref{choiceflood} we have that $d\in \cn{Chosen}^{\Imc}$. 

We show that $G$ is a \emph{yes-instance} of \coccol. We use the variable assignment given by the leaves of the second tree, i.e. the elements of $\leaftwo$. Given a variable $v_{i,j}$, we define the mapping $\pi$ that assigns each $(i,j)$ to the unique element $d\in \leaftwo$ such that
\begin{itemize}
\item $d\in \cn{B}_k^{\Imc}$ iff $bit_k(i)=1$, for all $0\leq k<n$,
\item $d\in \cn{\bar{B}}_k^{\Imc}$ iff $bit_k(i)=0$,  for all $0\leq k<n$,
\item $d\in {\cn{B}_k}^{\Imc}$ iff $bit_{2n-k}(j)=1$, for all $n\leq k<2n$,
\item $d\in \cn{\bar{B}}_k^{\Imc}$ iff $bit_{2n-k}(j)=0$, for all $n\leq k<2n$.
\end{itemize} 
Similarly, we can define a mapping from the set of vertices of $G$ to the leaves of the first tree. We denote such a mapping with $\nu$.

We define a truth assignment $t$ over the variables $v_{i,j}$ by stating $t(v_{i,j})=1$ if $\pi(v_{i,j})\in (\cn{T}')^{\Imc}$, $t(v_{i,j})=0$ if $\pi(v_{i,j})\in (\cn{F}')^{\Imc}$.

Assume by a way of contradiction that there exists a color assignment $\chi$ of edges such that $t(G)$ is 3-colorable. We show that such a color assignment can be used to find a model $\Jmc$ of $\Kmc$ such that $\Jmc\subset \Imc$. 

We define $\Jmc$ as the interpretation such that $\Delta^{\Jmc}=\Delta^{\Imc}$ and the interpretation function is defined as follows:
\begin{itemize}
\item $\cn{Flood}^{\Jmc}=\emptyset$,
\item $\cn{Chosen}^{\Jmc}=\{d_C\in \leafc\vert \chi(\nu^{-1}(d))$, where $d$ and $d_C$ are as in the statement of Claim \ref{perfect-coloring}.
\item $\cn{C'}^{\Jmc}=\{d\in \leafone\vert d_C\in \cn{Choice}^{\Jmc}\}$, for all $\cn{C}\in \{\cn{R,G,B}\}$, where $d$ and $d_C$ are as in the statement of Claim \ref{perfect-coloring}.
\item $\cn{C_1}^{\Jmc}=\{d\in \leafthree\vert d_1\in \cn{C'}^{\Jmc}\}$, where $d$ and $d_1$ are as in Claim \ref{connectionfirst},
\item $\cn{C_2}^{\Jmc}=\{d\in \leafthree\vert d_2\in \cn{C'}^{\Jmc}\}$ where $d$ and $d_1$ are as in Claim \ref{connectionfirst},
\item for all $p\in sig(\K)\setminus \{\cn{Flood,Chosen, C', C_1, C_2\vert C\in\{R,G,B\}}\}$, $p^{\Jmc}=p^{\Imc}$.
\end{itemize}

It is trivial to observe that $\Jmc\subset \Imc$. To show that $\Jmc$ is a model, it is sufficient to observe that all axioms involving the minimized concepts are satisfied. In particular, from Claim \ref{perfect-coloring} for each color $C$ and each $d\in \leafone$ there exists a unique $d_C\in leafc$ such that $(d,d_C)\in \cn{col}^{\Imc}$ and no other $d'\in \leafone$ is such that $(d',d)\in \cn{col}^{\Imc}$, then the new placement of $\cn{Chosen}$ does not clash with the color assignment $\chi$.

We derived a contradiction with the minimality of $\Imc$.
\end{proof}
\end{proof}